\def\year{2021}\relax
\documentclass[letterpaper]{article} 
\usepackage{aaai21}  
\usepackage{times}  
\usepackage{helvet} 
\usepackage{courier}  
\usepackage[hyphens]{url}  
\usepackage{graphicx} 
\usepackage{natbib}  
\usepackage{caption} 
\usepackage{ifthen}
\newboolean{make_appendix}  
\setboolean{make_appendix}{true} 
\usepackage{enumitem}
\usepackage{footmisc}
\usepackage{amsmath, amsfonts, bm} 
\usepackage{amsthm}
\usepackage{thmtools, thm-restate}

\newtheorem{theorem}{Theorem}
\newtheorem{lemma}[theorem]{Lemma}
\newtheorem{rmk}[theorem]{Remark}
\newtheorem{defn}[theorem]{Definition}
\newtheorem{prop}[theorem]{Proposition}
\newtheorem{cor}[theorem]{Corollary}

\newcommand\independent{\protect\mathpalette{\protect\independenT}{\perp}}
\def\independenT#1#2{\mathrel{\rlap{$#1#2$}\mkern2mu{#1#2}}}
\newcommand{\erf}{\text{erf}}
\newcommand\numberthis{\addtocounter{equation}{1}\tag{\theequation}}

\title{Avoiding Kernel Fixed Points: \\Computing with ELU and GELU Infinite Networks}
\author {
    Russell Tsuchida\textsuperscript{$\dagger$}\textsuperscript{$\ddagger$}, Tim Pearce\textsuperscript{$*$}, Chris van der Heide\textsuperscript{$\ddagger$}, Fred Roosta\textsuperscript{$\ddagger$}\textsuperscript{$**$}, Marcus Gallagher\textsuperscript{$\ddagger$}\\
}
\affiliations {
     \textsuperscript{$\dagger$}CSIRO,
     \textsuperscript{$\ddagger$}The University of Queensland,
     \textsuperscript{$*$}University of Cambridge,
     \textsuperscript{$**$}International Computer Science Institute
}

\begin{document}

\maketitle

\begin{abstract}
Analysing and computing with Gaussian processes arising from infinitely wide neural networks has recently seen a resurgence in popularity. 
Despite this, many explicit covariance functions of networks with activation functions used in modern networks remain unknown. 
Furthermore, while the kernels of deep networks can be computed iteratively, theoretical understanding of deep kernels is lacking, particularly with respect to fixed-point dynamics. 
Firstly, we derive the covariance functions of multi-layer perceptrons (MLPs) with exponential linear units (ELU) and Gaussian error linear units (GELU) and evaluate the performance of the limiting Gaussian processes on some benchmarks. 
Secondly, and more generally, we analyse the fixed-point dynamics of iterated kernels corresponding to a broad range of activation functions. 
We find that unlike some previously studied neural network kernels, these new kernels exhibit non-trivial fixed-point dynamics which are mirrored in finite-width neural networks. The fixed point behaviour present in some networks explains a mechanism for implicit regularisation in overparameterised deep models.
Our results relate to both the static iid parameter conjugate kernel and the dynamic neural tangent kernel constructions\footnote{Software at \url{github.com/RussellTsuchida/ELU_GELU_kernels}\label{footnote:software}}.
\end{abstract}

\section{Background --- Infinitely wide neural networks as Gaussian processes}
\label{sec:background}
Infinitely wide neural networks (NNs) and Gaussian processes (GPs) share an interesting connection~\citep{neal1995bayesian,jacot2018neural} which has only partially been explored. We begin by reviewing this connection. Readers familiar with this connection may skip to \S~\ref{sec:motivation}. Consider a one-hidden layer network with independent parameters. Suppose each $i$th row of weights $\mathbf{W}_i$ together with the corresponding bias $B_i$ in the hidden layer has distribution $(\mathbf{W}_i^\top, B_i)^\top = \widetilde{\mathbf{W}_i} \sim \mathcal{N} \big(\mu, \Sigma)$, with $\Sigma \succ 0$ being a diagonal matrix having a unique ``square root" $\Sigma^{(1/2)}$. Further, suppose the output layer parameter vector $\mathbf{V}=\frac{1}{\sqrt{n}} \mathbf{U}$ satisfies $\mathbf{U} \sim \mathcal{N}(\mathbf{0}, \sigma_w^2I)$, where $n$ is the number of neurons in the hidden layer and the output bias satisfies $V_b \sim \mathcal{N}(0, \sigma^2_b)$. The output evaluated at input $\mathbf{x}_1$ is $f(\mathbf{x}_1) = \frac{1}{\sqrt{n}}\sum_{i=1} U_i \psi(\widetilde{\mathbf{W}}_i^\top \widetilde{\mathbf{x}}_1) + V_b$, 
where $\psi$ is an activation function and $\widetilde{\mathbf{x}}_1 = (\mathbf{x}_1^\top, 1)^\top$.
The covariance between any two outputs is
\begin{align*} 
&\phantom{{}={}}k^{(1)}(\mathbf{x}_1, \mathbf{x}_2) \\
&= \mathbb{E}\Big[ \sum_{i=1}^n V_i \psi(\widetilde{\mathbf{W}}_i^\top \widetilde{\mathbf{x}}_1) \sum_{j=1}^n V_j \psi(\widetilde{\mathbf{W}}_j^\top \widetilde{\mathbf{x}}_2) \Big] {+} \sigma_b^2 \\
&= \sigma_w^2\mathbb{E}\Big[ \psi(\widetilde{\mathbf{W}}_1^\top \widetilde{\mathbf{x}}_1)\psi(\widetilde{\mathbf{W}}_1^\top \widetilde{\mathbf{x}}_2) \Big] + \sigma_b^2.
\end{align*}
The expectation over $d+1$ random variables reduces to an expectation over $2$ random variables, $\widetilde{\mathbf{W}}_1^\top \widetilde{\mathbf{x}}_1$ and $\widetilde{\mathbf{W}}_1^\top \widetilde{\mathbf{x}}_2$. The joint distribution of these two random variables is a bivariate Gaussian. The mean of each component is zero, and the variance is $\Vert \Sigma^{(1/2)} \widetilde{\mathbf{x}}_i \Vert^2$, where $\Vert \cdot \Vert$ denotes the Euclidean norm. The covariance is $\Vert \Sigma^{(1/2)} \widetilde{\mathbf{x}}_1 \Vert \Vert \Sigma^{(1/2)} \widetilde{\mathbf{x}}_2 \Vert \cos\theta$, where $\theta$ is the angle between $\Sigma^{(1/2)} \widetilde{\mathbf{x}}_1$ and $\Sigma^{(1/2)} \widetilde{\mathbf{x}}_2$. Therefore, the expectation in terms of $\mathbf{Z}  \sim \mathcal{N}(\mathbf{0},S)$ is
\begin{align*}
k^{(1)}(\mathbf{x}_1, \mathbf{x}_2) = \sigma_w^2 \mathbb{E}\Big[ &\psi\big(s_1 Z_1 + \widetilde{\mu}_1\big) \psi\big(s_2 Z_2 + \widetilde{\mu}_2 \big)\Big] + \sigma_b^2, \numberthis \label{eq:kernel}
\end{align*}
where $S$ has diagonals $1$ and off-diagonals $\cos\theta$, $s_i = \Vert \Sigma^{(1/2)} \widetilde{\mathbf{x}}_i \Vert$ and $\widetilde{\mu}_i = \bm{\mu}^\top \widetilde{\mathbf{x}}_i$. 
\begin{defn}
We call~\eqref{eq:kernel} the kernel. We call $\cos\theta^{(1)} = \frac{k^{(1)}(\mathbf{x}_1, \mathbf{x}_2)}{\sqrt{k^{(1)}(\mathbf{x}_1, \mathbf{x}_1)k^{(1)}(\mathbf{x}_2, \mathbf{x}_2)}}$ the normalised kernel.
\end{defn}
The above NN converges to a GP as $n \to \infty$ under mild conditions on the input and activation function $\psi$~\cite{neal1995bayesian}. Since $f(\mathbf{x}_1)$ is a sum of independent random variables scaling as $n^{-1/2}$, it converges to a Gaussian random variable with zero mean as $n \to \infty$. More generally, any fixed collection of $N$ evaluations of $f$, $\{f(\mathbf{x}_i)\}_{i=1}^N$ converges to an $N$-dimensional $\mathbf{0}$-mean Gaussian as $n \to \infty$. 

Analytical and closed-form covariance functions~\eqref{eq:kernel} are available for specific choices of $\psi$~\cite{le2007continuous, tsuchida2018invariance, tsuchida2018exchangeability, pearce2019expressive, tsuchida2019richer}, although some of these require $\bm{\mu}=0$. Most notably, the kernel is known for historically relevant activation functions $\psi(z) = \erf(z)$, RBF networks~\cite{williams1997computing} and for the more modern ReLU activation, $\psi(z)=\max(0,z)$ \cite{NIPS2009_3628}. More recently~\citet{meronen2020stationary} solved the \emph{inverse} problem, finding $\psi$ that recovers the Mat\'ern class of covariance functions. Once the form of~\eqref{eq:kernel} is known, the kernel of deep networks can be evaluated in the case where $\Sigma=\text{diag}(\sigma_w^2,...,\sigma_w^2, \sigma_b^2)$ and ${\bm{\mu}=\bm{0}}$~\cite{matthews2018gaussian, lee2017deep, NIPS2019_9186, yang2019scaling}. The case where $\bm{\mu}\neq 0$ can also be handled~\cite{tsuchida2019richer}, but we focus on $\bm{\mu}=0$ in this work. The kernel in layer $l+1$ can be found iteratively as a function of the kernel in layer $l$,
\begin{align*}
k^{(l+1)}(\mathbf{x}_1, \mathbf{x}_2) &= \sigma_w^2 \mathbb{E}\Big[\psi \big( s_1^{(l)} Z_1^{(l)} \big) \psi \big(s_2^{(l)}Z_2^{(l)} \big) \Big] + \sigma_b^2, \\
\begin{bmatrix} Z_1^{(l)} \\ Z^{(l)}_2 \end{bmatrix} &\sim \mathcal{N} \Bigg( \mathbf{0},\begin{bmatrix}
   1      & \cos\theta^{(l)} \\
   \cos\theta^{(l)}       & 1
\end{bmatrix}  \Bigg), \numberthis \label{eq:kernel_iter}
\end{align*}
where $\cos\theta^{(l)}$ is the normalised kernel in layer $l$, and $s_i^{(l)} = \sqrt{k^{(l)}(\mathbf{x}_i, \mathbf{x}_i)}$.
\begin{defn}
We call $k^{(l)}$ in~\eqref{eq:kernel_iter} the kernel in layer $l$. We call $\cos \theta^{(l)} = \frac{k^{(l)}(\mathbf{x}_1, \mathbf{x}_2)}{\sqrt{k^{(l)}(\mathbf{x}_1, \mathbf{x}_1)k^{(l)}(\mathbf{x}_2, \mathbf{x}_2)}}$ the normalised kernel in layer $l$.
\end{defn}
A generalisation of iid weight priors to partially exchangeable weight priors is also available~\cite{tsuchida2019richer}, resulting in a GP with an additional layer of inference over the hyperparameters $\bm{\mu}$ and $\Sigma$. Convergence to GPs also occurs for other NN architectures such as convolutional architectures~\cite{garriga2018deep, novak2018bayesian} and general compositions of recurrent, graph convolution, pooling, skip connection, attention and normalisation layers~\cite{NIPS2019_9186, yang2019scaling}\footnote{As detailed in these works, knowledge of~\eqref{eq:kernel} is often sufficient to describe these kernel, so our new kernels apply to more complicated networks.}. When an MLP is trained under a continuous-time analogue of gradient descent, the limiting output is still a GP~\cite{jacot2018neural}. The dynamics and associated covariance of the GP depends on the neural tangent kernel (NTK) $T^{(l)}$, which in addition to the iterations~\eqref{eq:kernel_iter}, is given by
\begin{align*}
T^{(1)}(\mathbf{x}_1, \mathbf{x}_2) &= k^{(1)}(\mathbf{x}_1, \mathbf{x}_2) \\ 
\dot{k}^{(l+1)} (\mathbf{x}_1, \mathbf{x}_2) &=  \sigma_w^2 \mathbb{E}\Big[ \psi'(s_1 Z^{(l)}_1) \psi'(s_2 Z^{(l)}_2) \Big], \\
T^{(l+1)}(\mathbf{x}_1, \mathbf{x}_2) &= T^{(l)}(\mathbf{x}_1, \mathbf{x}_2) \dot{k}^{(l+1)} (\mathbf{x}_1, \mathbf{x}_2) \\
&\phantom{{}={}} + k^{(l+1)}(\mathbf{x}_1, \mathbf{x}_2). \numberthis \label{eq:ntk}
\end{align*}

\section{Contributions and motivation}
\label{sec:motivation}

\begin{figure}[t]
    \centering
    \includegraphics[scale=0.25]{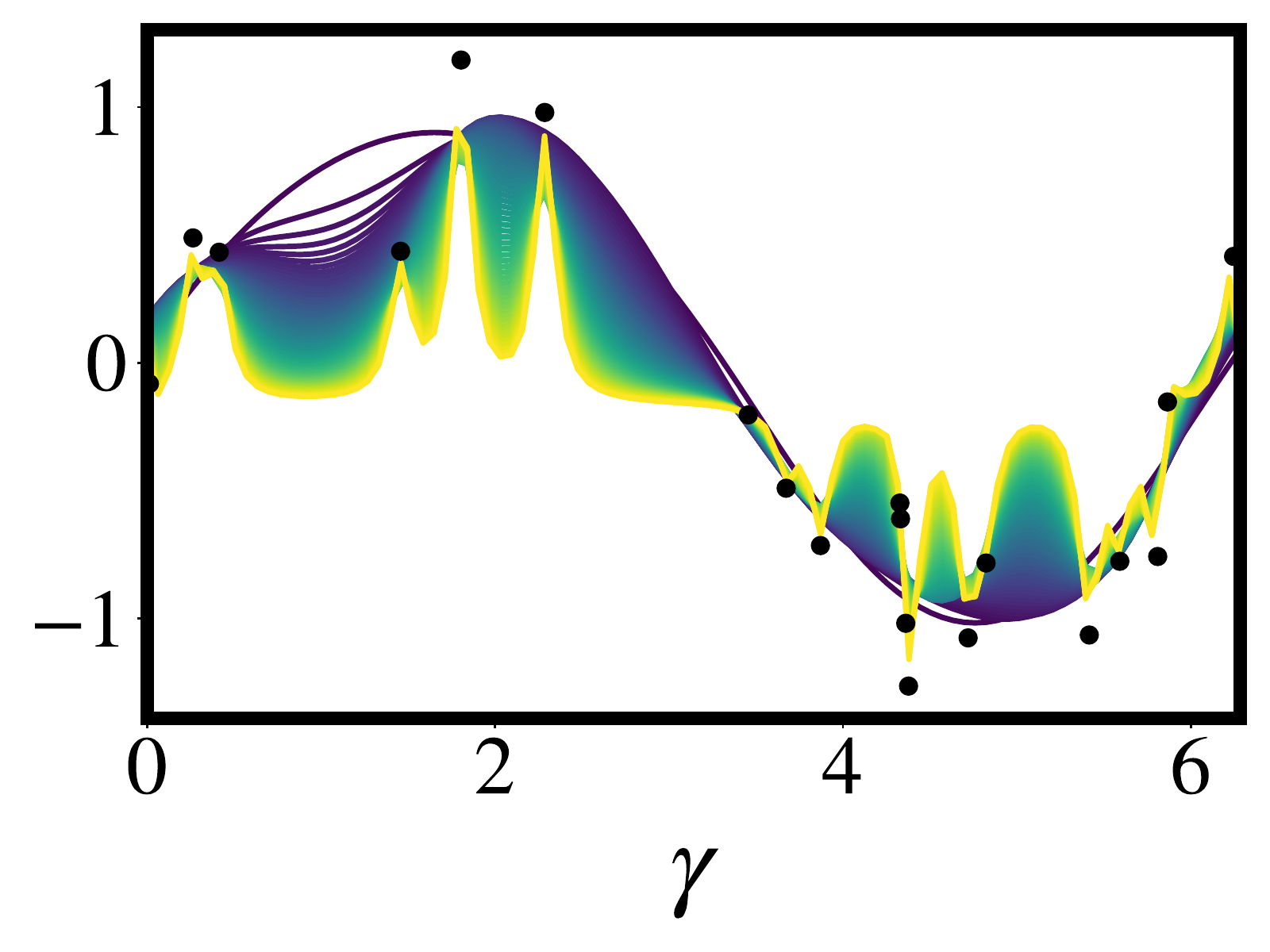}
    \includegraphics[scale=0.25]{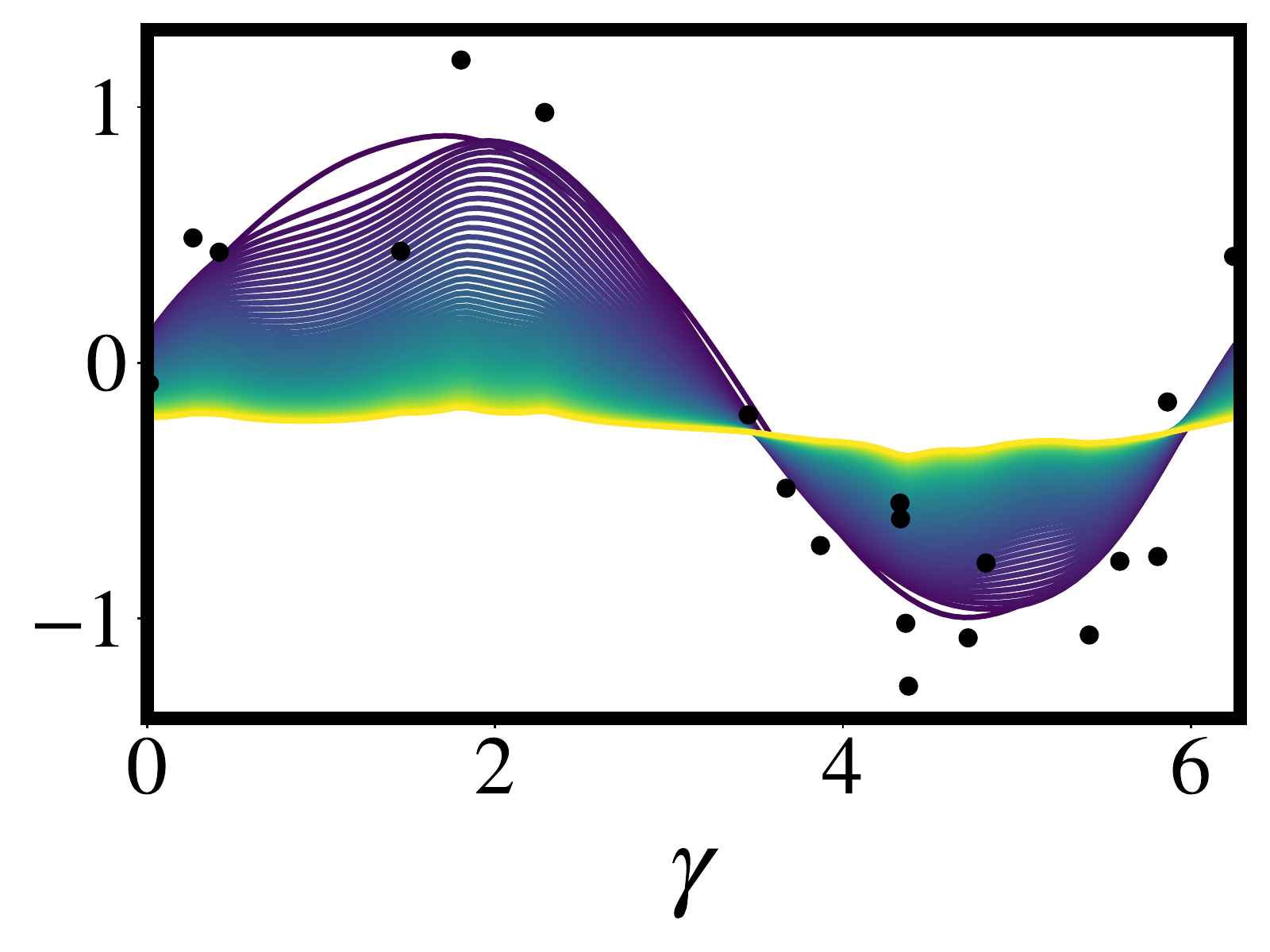}
    \caption{Illustration of simplicity bias due to kernel fixed points. Training data $\mathbf{x} \in \mathbb{R}^2$ is uniformly sampled on the unit disc at heading $\gamma$. Curves show the posterior mean of a GP regression model on $y=\sin(\gamma) + \epsilon$ with known additive noise variance. $\sigma_w$ is chosen according to Figure 3. Colours move from purple to yellow as depth increases from $1$ to $64$. (Left) GELU without unique kernel fixed point leading to overfitting (Right) ReLU with unique kernel fixed point leading to underfitting. More examples in Appendix~\ref{app:simplicity_bias}.}
    \label{fig:simplicity}
\end{figure}

This paper contains two main contributions. We:
\begin{enumerate}[leftmargin=*]
\item Derive kernels for GELU and ELU activations (defined below) and verify our results numerically. We implement GPs with different NN kernels on some benchmarks.
\item Study the fixed point dynamics of the kernel when $\psi$ is bounded by the absolute value of a polynomial. We find sufficient conditions for the existence of a unique kernel fixed point. We show theoretically and empirically that unlike the kernel corresponding to ReLU $\psi$, the new kernels are able to avoid unique fixed points. These conditions apply to both the iid prior~\citep{neal1995bayesian} and dynamic NTK~\citep{jacot2018neural} cases. This fixed point behaviour can be used to explain a simplicity bias in deep NN. More surprisingly, we find theoretically that the NTK dynamic which approximates gradient descent preserves this simplicity bias.
\end{enumerate}

\subsection{Motivation for studying fixed points}
Viewing NNs through the arguably idealised lens of GPs has some surprisingly non-intuitive practical implications. One important open problem is in explaining the empirical observation that some overparameterised NNs do not overfit, even when trained without explicit regularisation~\cite{zhang2017understanding}.~\citet{tsuchida2019richer} show empirically that samples from the limiting prior of deep MLPs with zero-mean parameters and LReLU activations are approximately constant on the unit hypersphere.~\citet{valle2019deep} argue that deep NNs with ReLU activations exhibit a ``simplicity bias", in that randomly initialised NNs implementing Boolean functions are likely to be simple.~\citet{yang2019fine} explain this simplicity bias through a spectral decomposition of the limiting kenel, showing that most of its mass is concentrated around the constant eigenfunction, even when accounting for training using gradient descent under the NTK~\cite{jacot2018neural}.~\citet{yang2019fine} are clear to separate the case of ReLU activations, which do result in kernels having peaked spectral mass and exhibiting a simplicity bias, from ERF activations which do not.

\begin{figure*}[t]
\centering
\includegraphics[scale=0.25]{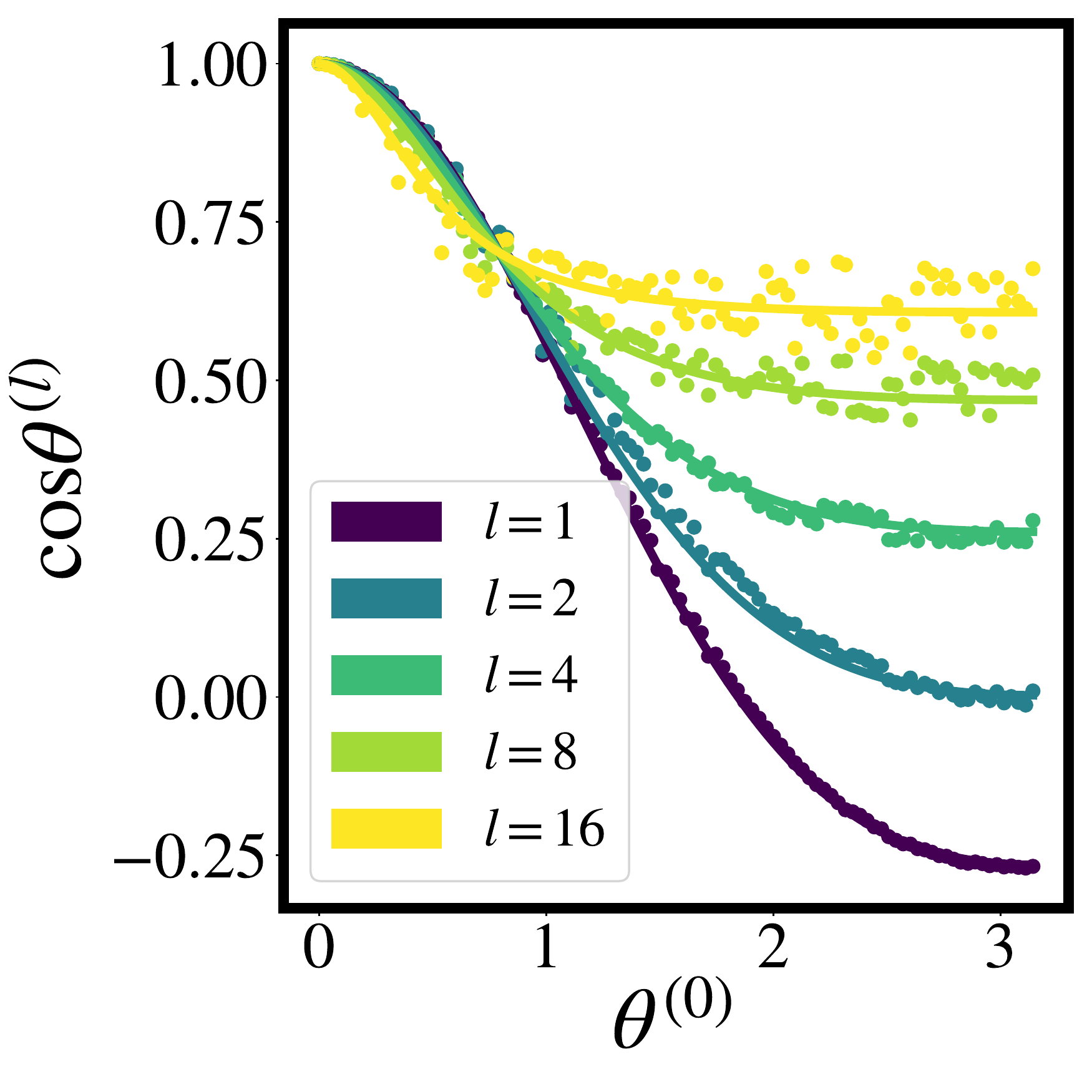}
\includegraphics[scale=0.25]{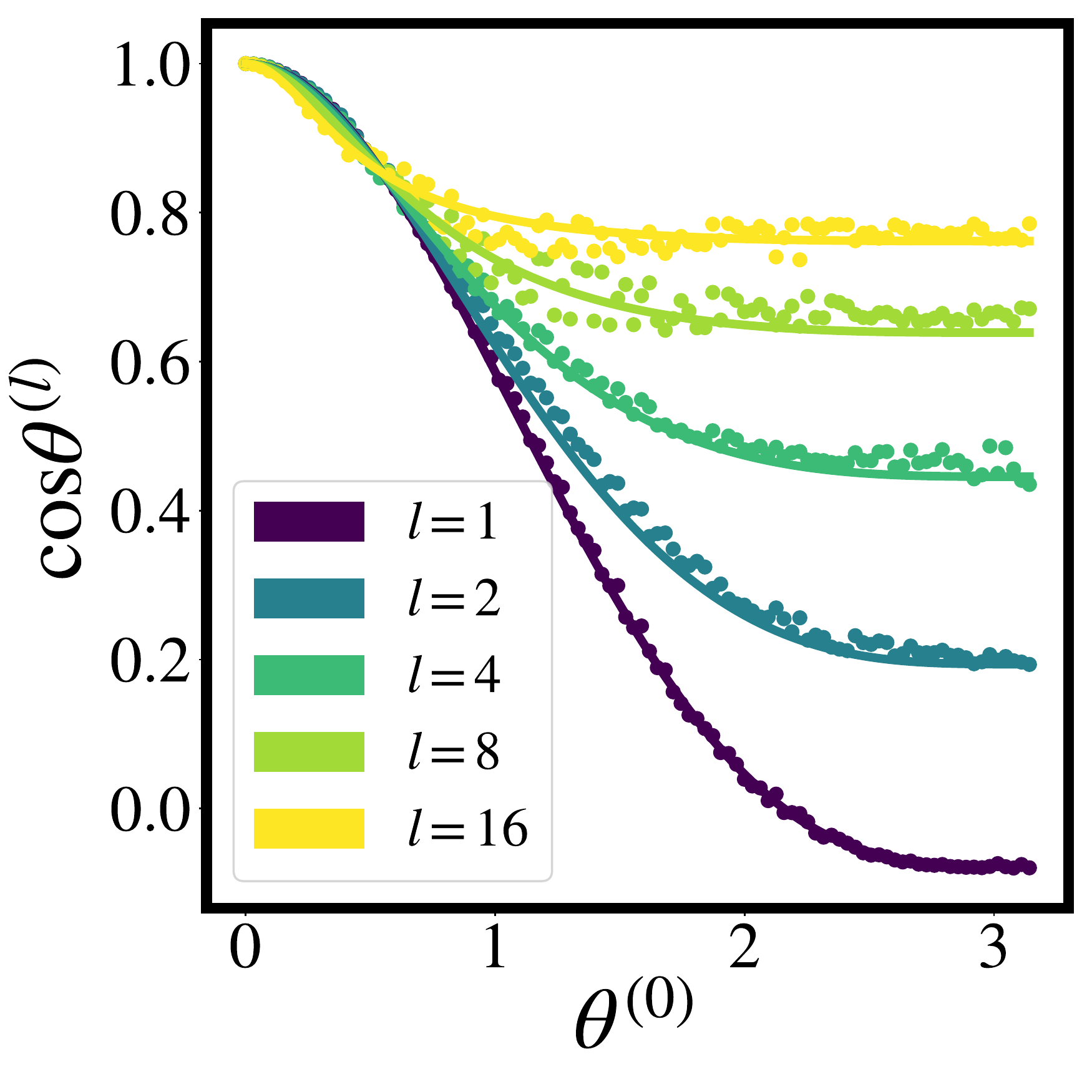}
\includegraphics[scale=0.25]{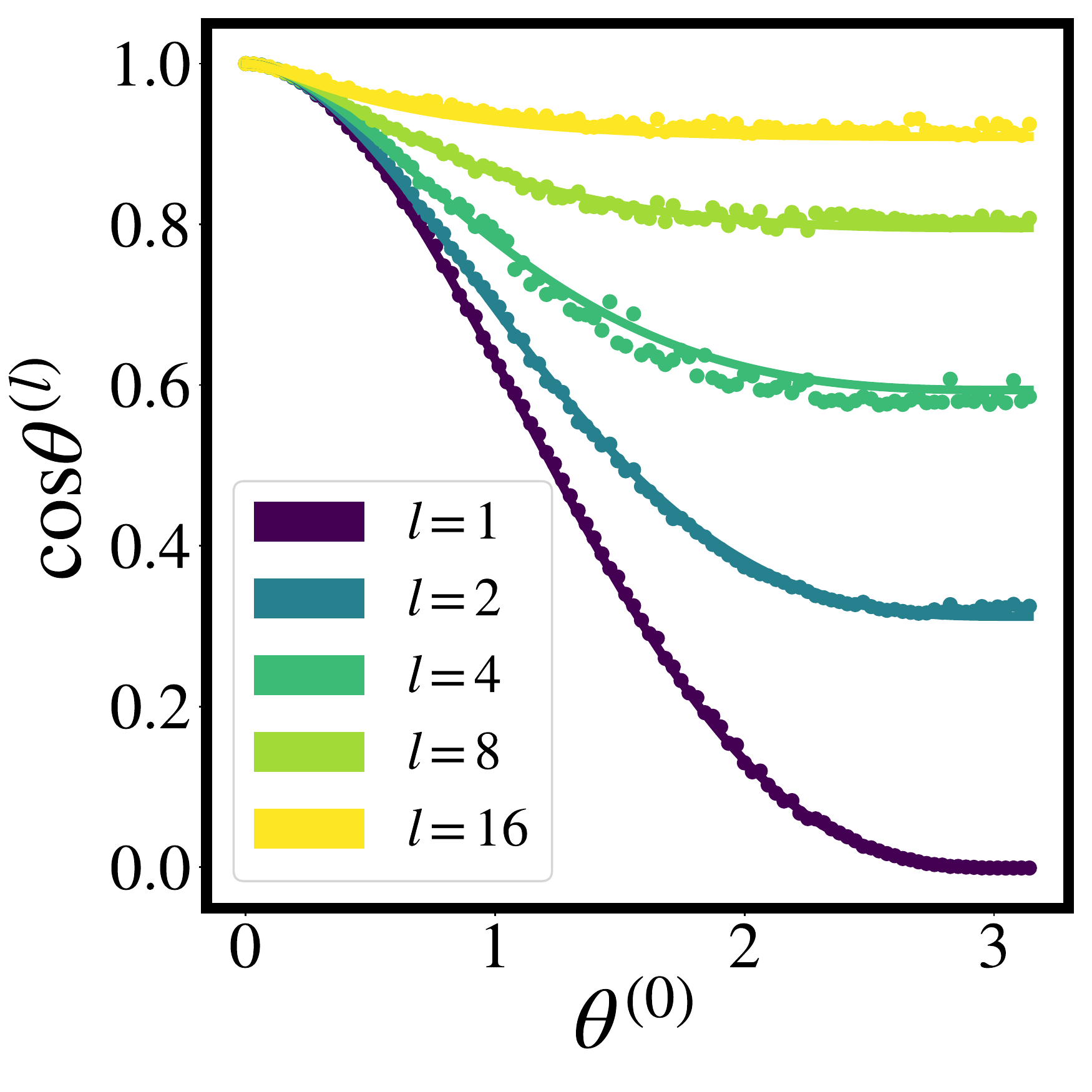}
\small
\begin{align*}
&k(\mathbf{x}_1, \mathbf{x}_2) = \\
&\sigma_b^2 + \sigma_w^2 \Bigg(\frac{s_1s_2}{4}\cos\theta+ \frac{s_1^2 s_2^2}{2\pi } \Bigg[ \frac{\frac{1}{2}(\cos(2\theta) + 3) + s_1^2 + s_2^2 +s_1^2s_2^2 \sin^2\theta}{(1+s_1^2)(1+s_2^2)\sqrt{1+s_1^2+s_2^2+s_1^2s_2^2\sin^2\theta}} + \frac{\cos\theta}{s_1s_2} \tan^{-1}\Bigg( \frac{\cos\theta s_1s_2}{ \sqrt{1+s_1^2+s_2^2+s_1^2s_2^2\sin^2\theta}} \Bigg) \Bigg] \Bigg)
\end{align*}
\caption{The GELU kernel. Plots show the normalised kernels in layer $l$ as a function of the angle $\theta^{(0)}$ between the inputs for MLPs of increasing depth when $\Sigma = \text{diag}(\sigma_w^2, ..., \sigma_w^2, 0)$ and $\Vert \mathbf{x}_i \Vert$ is constant for all $i$. Values of $\sigma_w$ are chosen to preserve the expected square norm $\Vert \mathbf{x} \Vert^2$ (see \S~\ref{sec:norm_preserve}). Solid curve shows infinitely wide limit, and dots show samples from a network with $2$ inputs and $3000$ neurons in each layer. Each dot corresponds to an $\mathbf{x}_1$ and $\mathbf{x}_2$ generated through a random rotation of $(1, 0)^\top$ and $(\cos\theta^{(0)}, \sin\theta^{(0)})^\top$. The random rotation is found through a QR decomposition of a matrix containing entries sampled independently from $\mathcal{U}[0,1]$. (Left) $\Vert \mathbf{x} \Vert=0.5$, $\sigma_w=1.59$ (Middle) $\Vert \mathbf{x} \Vert = 1$, $\sigma_w = 1.47$. (Right) $\Vert \mathbf{x} \Vert = 5$, $\sigma_w=1.42$.} 
\label{fig:gelu_kernel}
\end{figure*}

Our motivation for studying fixed points is in a similar spirit to the work above. For LReLU networks, we observe a so called kernel fixed point. An infinitely deep LReLU network is degenerate, and therefore over-regularised, in that all functions in the prior are constant over inputs on any hypersphere. Therefore, increasingly deep kernels represent a strict and potentially undesirable prior. On the other hand, kernels corresponding to GELU and ELU activation functions do not exhibit unique fixed points, and are therefore less biased towards simple functions. Just as traditional regularisation frameworks allow practitioners to control the bias-variance trade-off, in our framework, the activation function represents a similar choice. A deep ReLU network contains a higher degree of implicit regularisation than a deep GELU network. An illustration of this effect is shown in Figure~\ref{fig:simplicity}. Even more surprisingly, our analysis extends to the NTK.

\subsection{Recently introduced activation functions}
The increased volume of gradient-based deep learning research has seen the introduction of new popular activation functions. Notably these include the exponential linear unit (ELU)~\cite{clevert2015fast}, the Gaussian error linear unit (GELU)~\cite{hendrycks2016gaussian} and the Swish~\cite{ramachandran2017searching, elfwing2018sigmoid}. The GELU and ELU are
\begin{align*}
\psi(z) = z\Phi(z) \text{ and }\psi(z) &= \Theta(z)z + \Theta(-z)(e^z-1),
\end{align*}
respectively, where $\Phi$ denotes the CDF of the standard Gaussian and $\Theta$ denotes the Heaviside step function.

Many state-of-the-art models use GELU~\cite{radford2018improving, devlin2019bert} or swish activations~\cite{chua2018deep}. However, even when critically evaluating empirical evidence, it is difficult to determine whether certain activation functions are a better choice for a given problem, let alone separate the activation function expressivity from the ability of optimisers to find good solutions. Analysing activation functions through the lens of GPs allows one to visualise the function space in isolation of the ability of the optimiser to find good solutions, and reveals interesting implicit regularisation structure in the infinitely wide setting.

\subsection{Model selection}
The choice of activation function in an NN can be framed as a  model selection problem, and as such shares similarities with choosing other model hyperparameters. Take the case of choosing the L2 ridge-regularisation parameter as an example. One could apply n-fold cross-validation, with an implicit understanding that this parameter penalises model complexity as measured through the norm in an RKHS. Alternatively, a Bayesian might interpret the L2 weighting as the precision of a Gaussian prior, and optimise the marginal likelihood with respect to this weighting. A more pure Bayesian might put a prior over the regularisation (or precision) parameter and marginalise over these models. Common to all these approaches is (a) an intuition based on theory of what the parameter does and (b) a practical methodology for dealing with the parameter. In this paper we provide (a), and leave it to the practitioner to decide on (b) based on their philosophy and/or apparatus. Our work shows that GELU and ELU activations can avoid a regularisation mechanism that grows with depth that is \emph{always} implicit in ReLU activations.

We stress that the new kernels and the fixed point mechanisms are neither ``good" nor ``bad" in isolation. The new models should be evaluated according to model selection criteria in their given application, and the fixed point mechanisms describe a regularisation implicit in some kernels.

\begin{figure*}[!ht]
\centering
\includegraphics[scale=0.25]{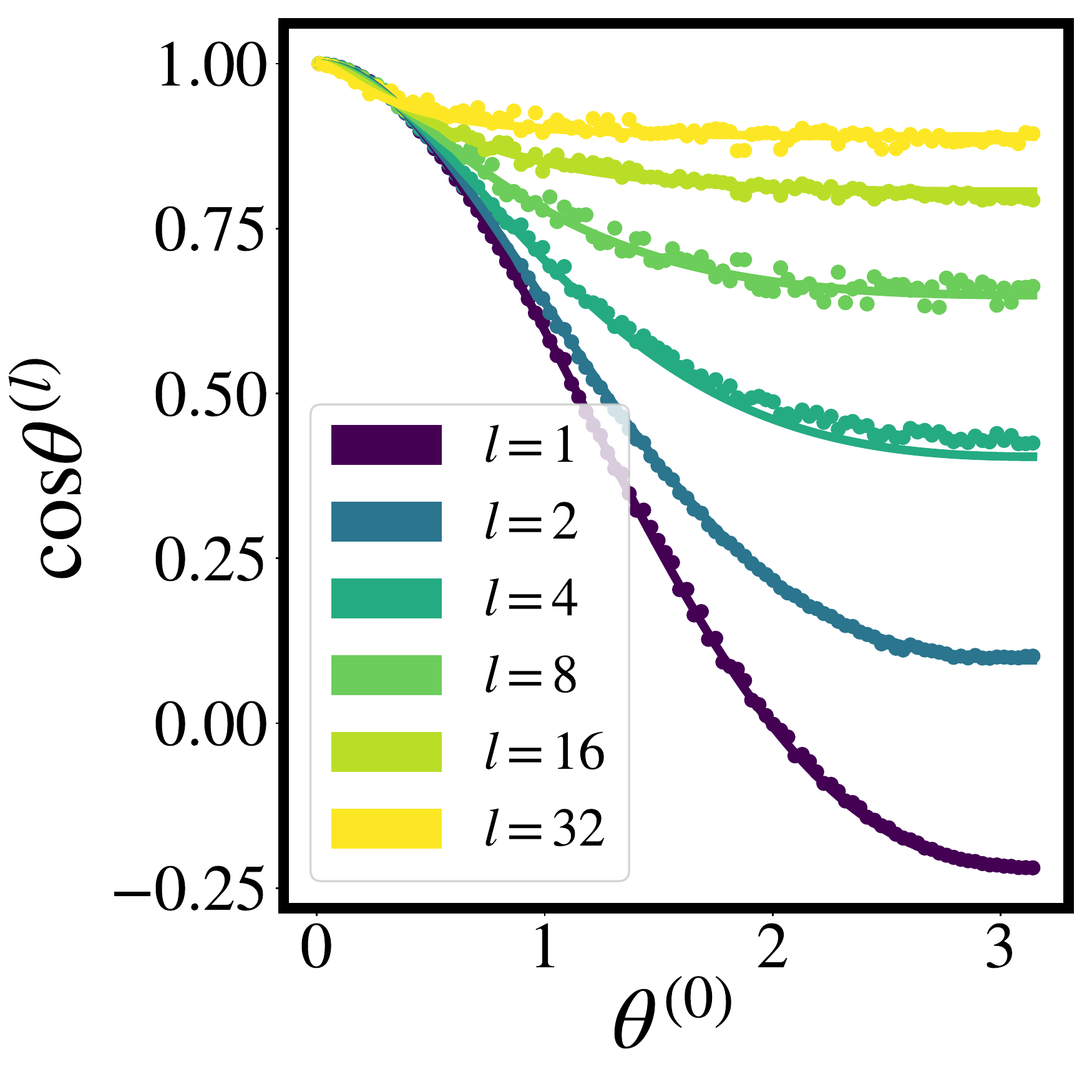}
\includegraphics[scale=0.25]{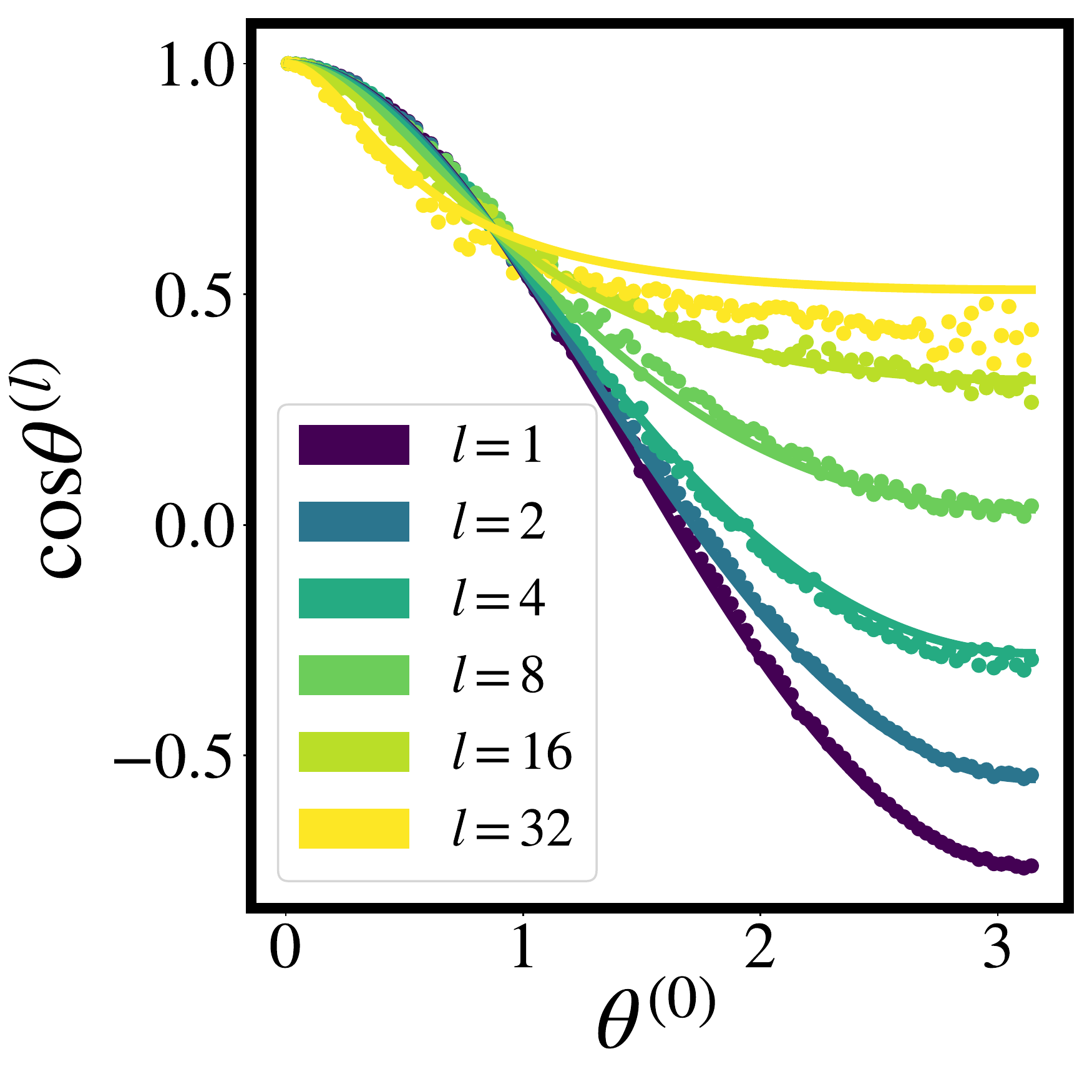}
\includegraphics[scale=0.25]{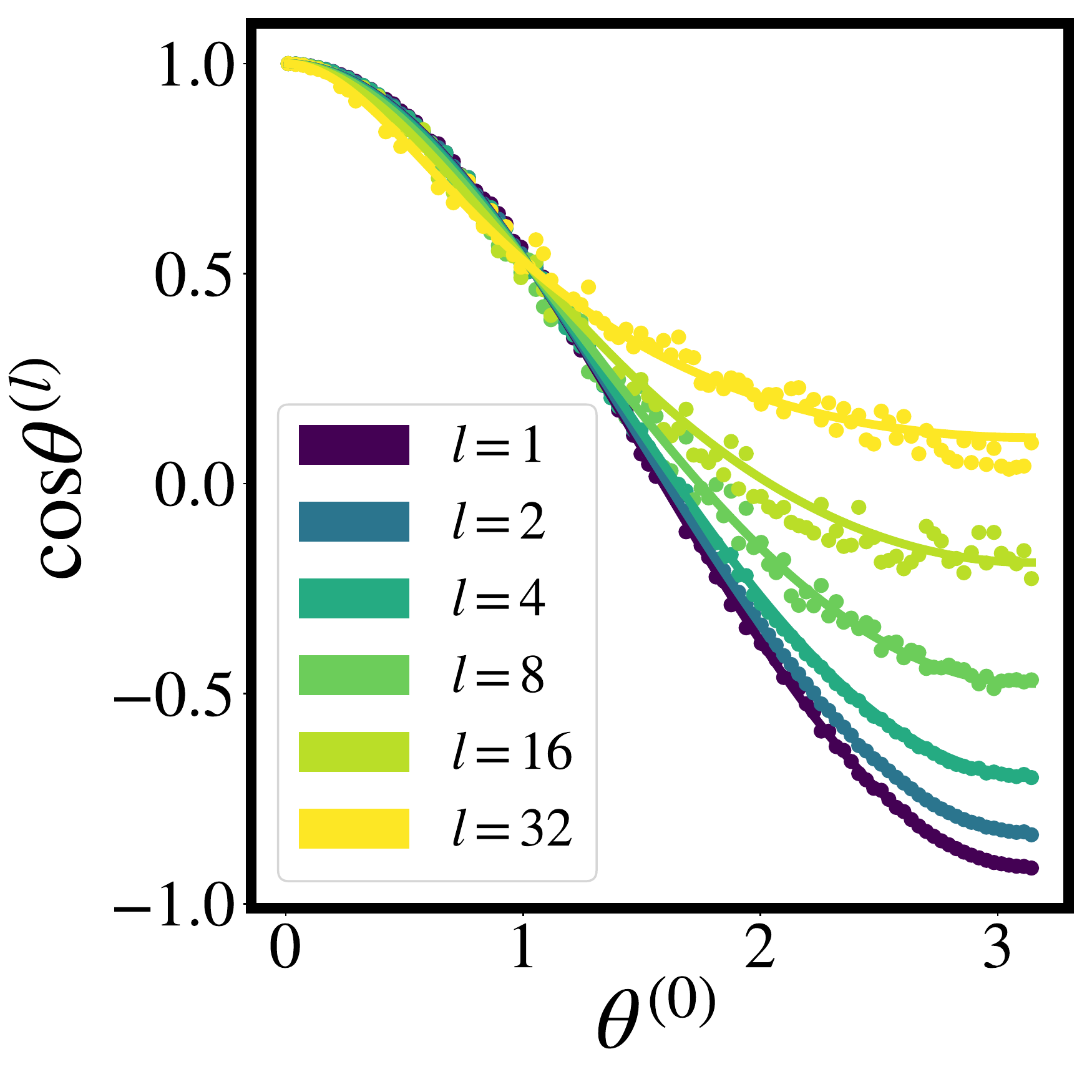}
\caption{As in Figure~\ref{fig:gelu_kernel}, but with ELU $\psi$. (Left-Right) $(\Vert \mathbf{x} \Vert, \sigma_w) = (5, 1.40)$, $(1, 1.26)$, $(0.5, 1.17)$.}
\label{fig:elu_kernel}
\end{figure*}

\section{New kernels}
\label{sec:new_kernel}
\begin{prop}
When $\psi$ is the GELU and $\bm{\mu}=0$, the kernel~\eqref{eq:kernel} is given by the equation in Figure~\ref{fig:gelu_kernel}.
\end{prop}
The derivation is in the same spirit as~\citet{williams1997computing}; we introduce dummy parameters $\beta_1$ and $\beta_2$ in the argument of $\Phi$, differentiate with respect to $\beta_1$ and $\beta_2$ to obtain a PDE, then solve the PDE and evaluate the solution at $\beta_1=\beta_2=1$. Complete working is given in Appendix~\ref{app:gelu_kernel}. It is plausible that our method of derivation extends to the case $\bm{\mu}\neq 0$, although the calculations and resulting expression become more complicated. Even with $\bm{\mu}=0$, this kernel has some interesting properties that we discuss in \S~\ref{sec:fixed_point}. Interestingly, unlike the ELU kernel with $\bm{\mu}=0$, the GELU kernel does not contain any hard-to-compute special functions, only some (inverse) trigonometric functions.

Our expression for the ELU kernel is lengthy and we do not assemble it in the main text, but provide a visualisation in the form of Figure~\ref{fig:elu_kernel}.
\begin{prop}
When $\psi$ is the ELU, the kernel~\eqref{eq:kernel} has an analytical expression implemented in software\footref{footnote:software} in terms of the univariate and bivariate normal CDFs.
\end{prop}
Complete working is given in Appendix~\ref{app:elu_kernel}. Unfortunately, the ELU kernel involves exponentiating arguments involving $s_1$ and $s_2$. This can lead to numerical instability in GP regression when many data points are involved. Despite this, having an analytical expression still allows us to gain insights into finite width networks, as we shall see in \S~\ref{sec:fixed_point}. 

The scaled exponential linear unit (SELU)~\citep{klambauer2017self} is a slightly modified version of the ELU which introduces two scaling parameters: one applied over the entire domain and one over the negative domain. Our analysis for the ELU also handles the SELU (see Appendix~\ref{app:elu_kernel}). \citet{klambauer2017self} motivates the SELU by showing that it is able avoid the exploding and vanishing gradient problem by carefully selecting the value of the scaling parameters. An entirely distinct problem is to analyse the \emph{correlations} between signals in the network. Fixed points in signal norms are desirable to maintain the scale of signals as they propagate through the network. Fixed points in signal correlations can be undesirable as they force unrelated inputs to have similar feature representations in deep layers of the network. While~\citet{klambauer2017self} is concerned with obtaining fixed points of signal norms (i.e. our \S~\ref{sec:norm_preserve}), it does not relate to fixed points of signal correlations (i.e. our \S~\ref{sec:fp_general}).

\section{Fixed point analysis}
\label{sec:fixed_point}
In this section, we first analyse the conditions under which the expected squared norm of the signals in each layer are preserved as the signal propagates through the network for a different choices of the activation function (\S~\ref{sec:norm_preserve}). In such a situation, the expected squared norm remains at a fixed point as it passes through the network. Then, more generally, we analyse conditions under which the expected squared norm of any two signals \emph{and} the cosine angle between them approaches a constant as depth increases (i.e., when the kernel has a fixed point). We are especially interested in the case where the cosine angle between the signals converges to a unique fixed point (\S~\ref{sec:fp_general}). Finally, we relate the existence of a unique fixed point to a degenerate, underfitting property of very deep infinitely wide MLPs (\S~\ref{sec:degenerate}).

For \S~\ref{sec:fixed_point},~\ref{sec:ntk_jacob} and~\ref{sec:experiments}, we suppose all the weights have the same variance, and so do all the biases; the first $d$ diagonals of the diagonal matrix $\Sigma$ are $\sigma_w^2$, and the last diagonal is $\sigma_b^2$.
\subsection{Warm-up --- norm preservation}
\label{sec:norm_preserve}

A useful application of the kernel in finite-width iid-initialised networks is to track the expected squared norm of the signals in each layer as the depth of the network increases. This is used in initialisation to avoid exploding or vanishing signals when using gradient optimisers. 

The expected norm squared of the signal in the first hidden layer is $(k^{(1)}(\mathbf{x}_1, \mathbf{x}_1)-\sigma_b^2)/\sigma_w^2$. For the squared norm of the signal in the hidden layer to be the same as the squared norm of the input, we set $\Vert \widetilde{\mathbf{x}}_1 \Vert^2 = (k^{(1)}(\mathbf{x}_1, \mathbf{x}_1)-\sigma_b^2)/\sigma_w^2$. We may then solve this condition to find the hyperparameter values that preserve input norms. For example, using the kernel corresponding to ReLU~\cite{NIPS2009_3628}, one obtains He initialisation~\cite{he2015delving}, that $\sigma_w = \sqrt{2}$, where $\Sigma^{1/2}= \text{diag}(\sigma_w, ..., \sigma_w, 0)^\top$. 

The analogue for GELU is more involved since no \emph{single} $\sigma_w$ preserves the expected square norms of \emph{all} inputs. Setting $\sigma_b=0$, $k(\mathbf{x}, \mathbf{x})/\sigma_w^2 = \Vert \mathbf{x} \Vert^2$ and $s_1 = s_2 = \sigma_w \Vert \mathbf{x} \Vert$ in the equation in Figure~\ref{fig:gelu_kernel}, we find a root $\sigma^*(\Vert \mathbf{x} \Vert)$ of
\begin{align*}
g_{\Vert \mathbf{x} \Vert} (\sigma) &= \frac{\sigma^4 \Vert \mathbf{x} \Vert^2}{\pi (\sigma^2 \Vert \mathbf{x} \Vert^2 + 1) \sqrt{ 2\sigma^2 \Vert \mathbf{x} \Vert^2 + 1}} + \\
&\phantom{{}={}}\frac{\sigma^2}{4} \Big( 1 + \frac{2}{\pi} \sin^{-1} \frac{\sigma^2 \Vert \mathbf{x} \Vert^2}{1 + \sigma^2 \Vert \mathbf{x} \Vert^2} \Big) - 1,
\end{align*}
numerically. Figure~\ref{fig:gelu_elu_roots} shows a plot of $\sigma^*(\Vert \mathbf{x} \Vert)$ as $\Vert \mathbf{x} \Vert$ varies. The root of the limit of $g_{\Vert \mathbf{x} \Vert} (\sigma)$ as $\Vert \mathbf{x} \Vert \to \infty$ is $\sqrt{2}$, which recovers He initialisation. This implies that when data has large norms (such as images or audio files), He initialisation is suitable. The same procedure can be carried out for the ELU kernel as shown in Figure~\ref{fig:gelu_elu_roots}. This procedure may be viewed as a warm-up handling the special case of $\mathbf{x}_1=\mathbf{x}_2$ for our general fixed point analysis.

\subsection{General fixed point analysis}
\label{sec:fp_general}
\begin{figure*}[t]
\centering
\includegraphics[scale=0.23]{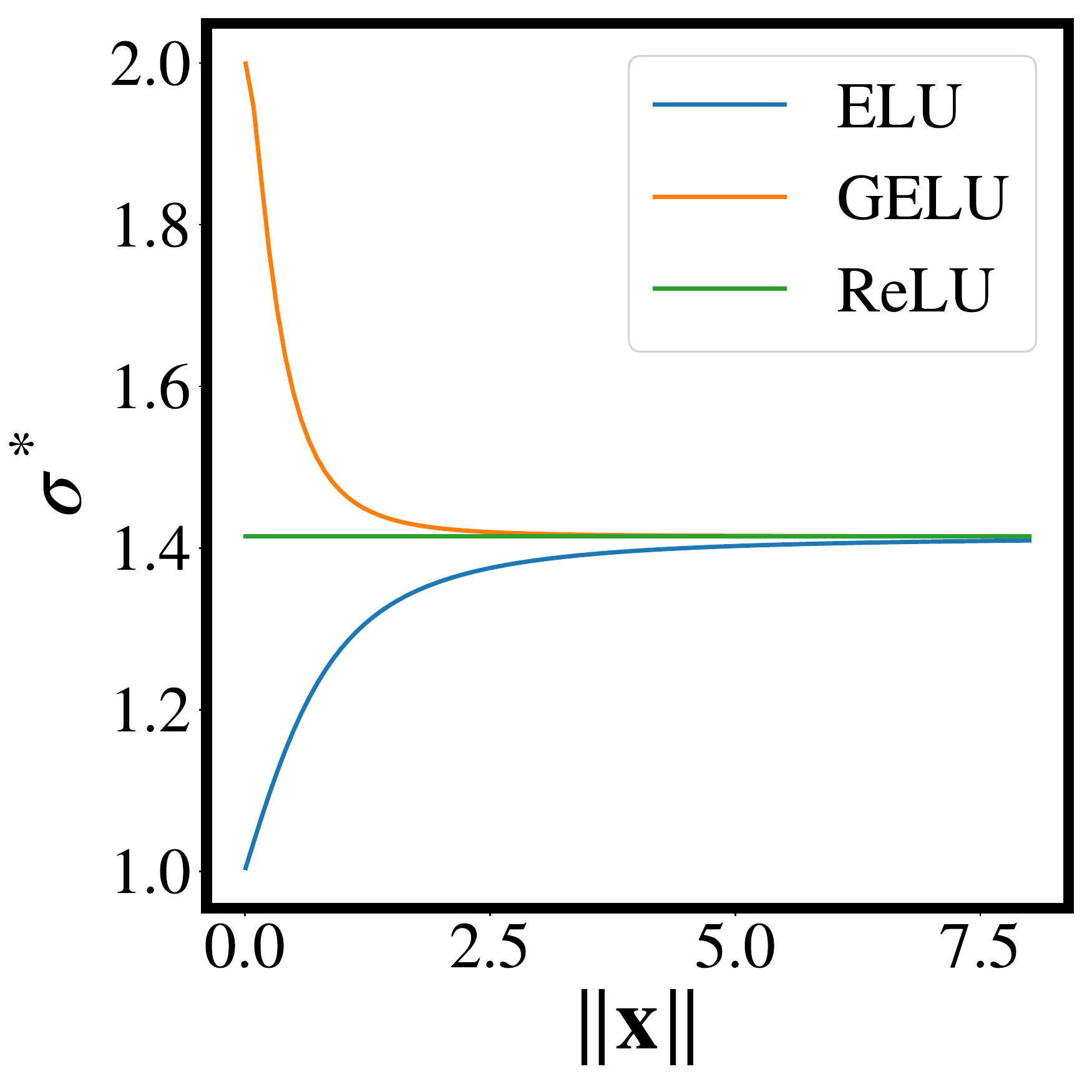}
\includegraphics[scale=0.23]{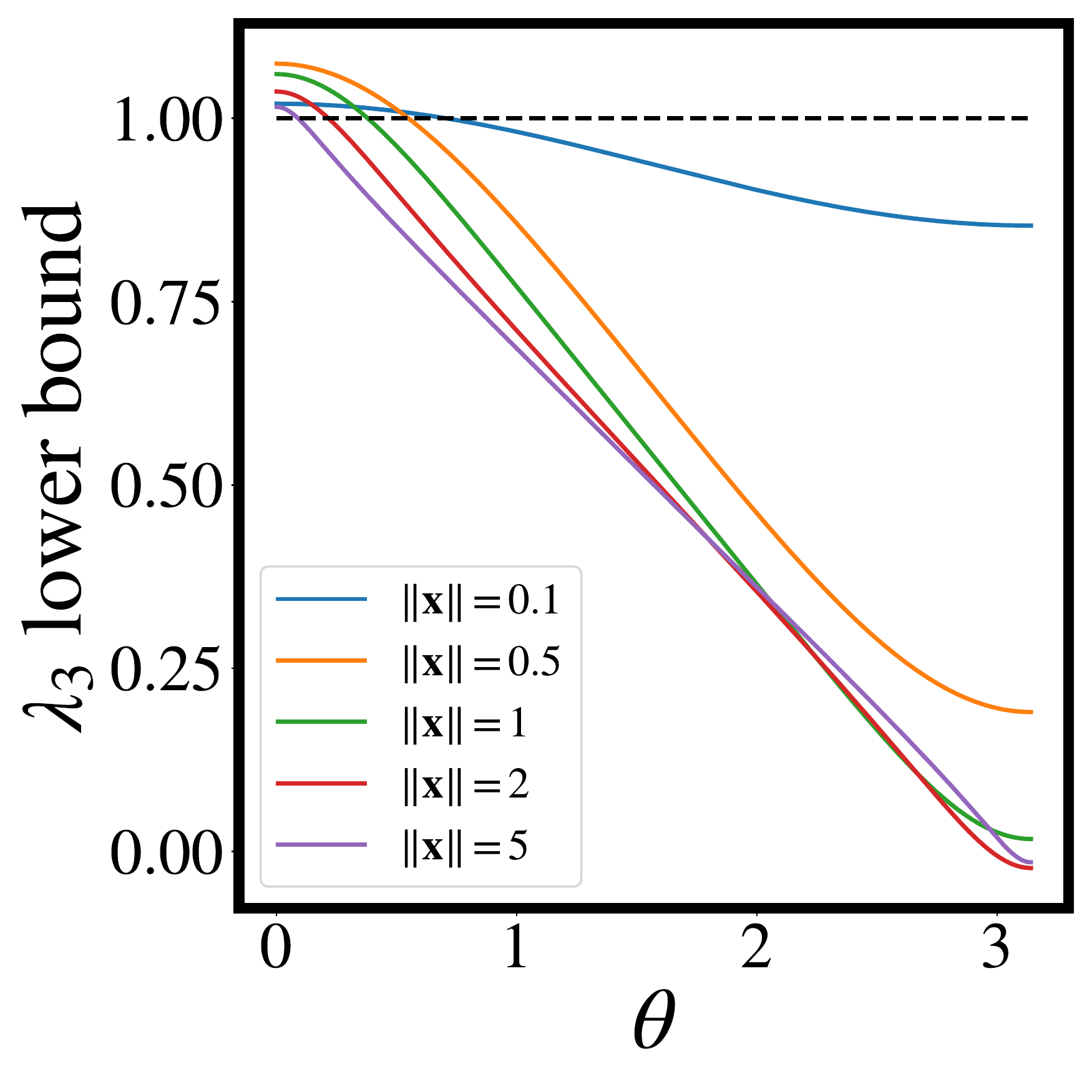}
\includegraphics[scale=0.23]{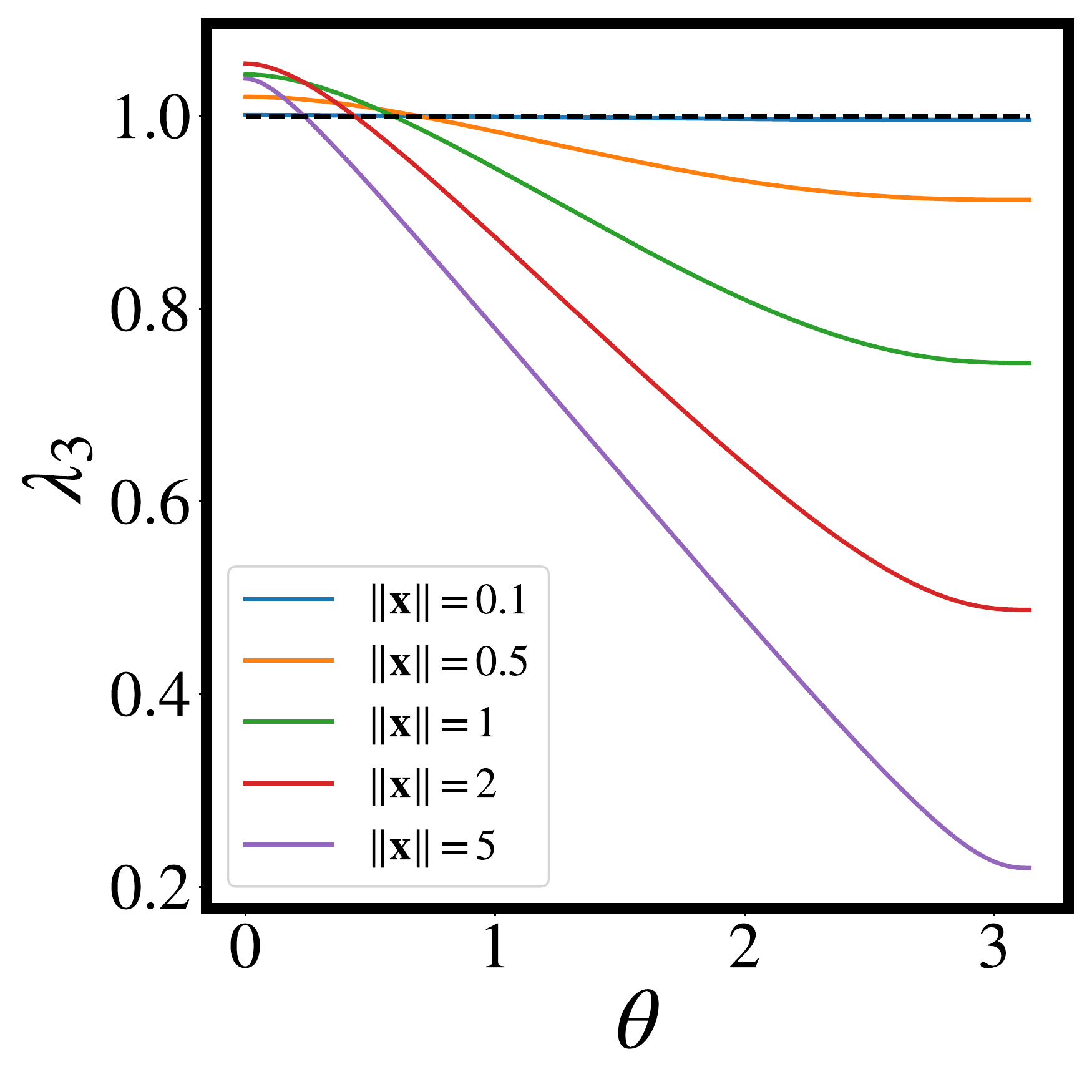}
\caption{(Left) Values $\sigma^*$ that preserve the layer-wise expected square norm of $\Vert \mathbf{x} \Vert$ in MLPs with GELU, ELU and ReLU activations. (Middle) lower bound of $\lambda_3$ for GELU (Right) $\lambda_3$ for ELU. If $\lambda_3 \leq 1$ on $\theta\in(0, \pi)$, a unique fixed point exists.}
\label{fig:gelu_elu_roots}
\end{figure*}

Let $\mathcal{S} \subseteq [0, \infty)\times [0, \infty) \times [-1, 1]$. In the infinitely wide limit, we may view each layer as updating a state $(s_1^2, s_2^2, \cos\theta) \in \mathcal{S}$ containing the expected square norms and the cosine angle between the signals in the hidden layers through a function $\mathbf{g}: \mathcal{S} \to  \mathcal{S}$. Let $(G_1, G_2)^\top \sim \mathcal{N}(\mathbf{0},I)$. We study the fixed-point dynamics of the iterated map $\mathbf{g}$ having components
\begin{align*}
g_1(s_1^2, s_2^2, \rho) &= \sigma_w^2\mathbb{E}\big[ \psi^2( s_1 G_1) \big] + \sigma_b^2, \\
g_2(s_1^2, s_2^2, \rho) &= \sigma_w^2 \mathbb{E}\big[ \psi^2(s_2 G_2) \big] + \sigma_b^2 \\
 g_3(s_1^2, s_2^2, \rho) &= \\
 &\hspace{-5em}\frac{\mathbb{E}\Big[ \sigma_w^2\psi(s_1 G_1) \psi \big(s_2(G_1 \rho + G_2 \sqrt{1-\rho^2}) \big)  + \sigma_b^2 \Big]}{\sqrt{g_1(s_1^2, s_2^2, \rho) g_2(s_1^2, s_2^2, \rho)}} \numberthis \label{eq:g_iterated},
\end{align*}
which track the expected square norms (after a linear transformation involving $\sigma_w^2$ and $\sigma_b^2$) and normalised kernel as the signals propogate through the layers\footnote{We expressed the kernel~\eqref{eq:kernel} in terms of iid Gaussians $\mathbf{G}$ instead of dependent Gaussians $\mathbf{Z}$.}. By inspection, $g_3$ (but not necessarily $\mathbf{g}$) always has an uncountable set of fixed points at $\rho=1$ along $s_1 = s_2$. Banach's fixed point theorem says that if $\mathbf{g}$ is a contraction mapping on a closed set, then $\mathbf{g}$ has a unique fixed point on that set~\cite{agarwal2001fixed}. In a slightly different setting,~\citet[Theorem 2.2.16]{hasselblatt2003first} allows some open sets.
\begin{theorem}
\label{thm:banach}
Let $D\mathbf{g}$ denote the Jacobian of $\mathbf{g}$ and $d'$ denote the metric induced by a norm with induced matrix norm $\Vert \cdot \Vert'$. If $C\subset \mathbb{R}^m$ is an open strictly convex set, $\overline{C}$ is its closure, $\mathbf{g}: \overline{C} \to \overline{C}$ differentiable on $C$ and continuous on $\overline{C}$ with $\Vert D\mathbf{g} \Vert' \leq \lambda < 1$ on $C$, then $\mathbf{g}$ has a unique fixed point $\mathbf{c}_0 \in \overline{C}$ and $d'\big(\mathbf{g}^L(\mathbf{c}), \mathbf{c}_0\big) \leq \lambda^L d'(\mathbf{c}, \mathbf{c}_0)$ for every $\mathbf{c} \in \overline{C}$.
\end{theorem}
We therefore consider the eigenvalues of the Jacobian, proving the following in Appendix~\ref{app:jac_proof}.
\begin{theorem}
\label{thm:jacobian}
Let $\mathbf{g}$ be as in~\eqref{eq:g_iterated}, and suppose the absolute value of $\psi$ is bounded by a polynomial. Let $(Z_1, Z_2)\sim \mathcal{N}(\mathbf{0}, S)$ with covariance $\rho=\cos\theta$ and unit variances. Then for $\rho \in (-1, 1)$ $0<s_1, s_2$, the (unordered) eigenvalues of the Jacobian of $\mathbf{g}$ are
\begin{align*}
\lambda_1 := \frac{\partial g_1}{\partial s_1^2 } &= \sigma_w^2 \mathbb{E}\Big[ \big(Z_1^2-1\big) \psi^2(s_1 Z_1 ) \Big]/(2s_1^2), \\
\lambda_2 := \frac{\partial g_2}{\partial s_2^2 } &= \sigma_w^2 \mathbb{E}\Big[ \big(Z_2^2-1\big) \psi^2(s_2 Z_2 ) \Big]/(2s_2^2),  \quad \text{and}  \\
\lambda_3 := \frac{\partial g_3}{\partial \rho } &= \frac{\sigma_w^2 s_1 s_2}{\sqrt{g_1 g_2}} \mathbb{E} \big[ \psi'( s_1 Z_1)  \psi'(s_2 Z_2)  \big],
\end{align*}
provided the right hand terms are finite, where $\psi'$ is the distributional derivative of $\psi$.
\end{theorem}
Our result does not include $\rho \in \{-1, 1\}$. With the additional assumption that $\psi$ is continuous almost everywhere, the expression for $\lambda_3$ is valid on the closed interval $\rho \in [-1, 1]$, as shown in Appendix~\ref{app:endpoints}. We would now like to combine Theorem~\ref{thm:banach} and Theorem~\ref{thm:jacobian} in order to comment on the existence of unique fixed points in some special cases. We consider two general cases in Corollaries~\ref{cor:abs_hom_fp} and~\ref{cor:g3_only} below.
\begin{cor}[Unique fixed point under absolute homogeneity]
\label{cor:abs_hom_fp}
Suppose $\sigma_b^2=0$ (as is common when initialising neural networks) and $\psi$ is absolutely homogeneous, that is, $\psi(|a|z) = |a| \psi(z)$ for any $a \in \mathbb{R}$. Then 
$$ \frac{\partial g_3}{\partial \rho} = \lambda_3 =  \frac{\mathbb{E} \big[ \psi'( Z_1)  \psi'( Z_2)  \big]}{\mathbb{E}[\psi^2(Z_1)]}.$$
Furthermore, if
\begin{itemize}
    \item $\max_{i=1,2,3} |\lambda_i| < 1$ then $\mathbf{g}$ admits a unique fixed point at $\big( s^2 s^2, 1 \big)$ for some $s^2$.
    \item $\lambda_3 < 1$ and $g_1(\cdot, s_2^2, \rho)$ admits a fixed point $s^2$ for any $s_2^2, \rho$, then $g_3(s^2, s^2, \cdot)$ admits a unique fixed point at $1$.
\end{itemize}

\end{cor}
\begin{proof}
Absolute homogeneity implies that
\begin{align*}
    &\phantom{{}={}}g_3(s_1^2, s_2^2, \rho) \\
    &= \frac{\mathbb{E}\Big[ \sigma_w^2\psi(s_1 G_1) \psi \big(s_2(G_1 \rho + G_2 \sqrt{1-\rho^2}) \big)  \Big]}{\sqrt{\mathbb{E}\Big[ \sigma_w^2\psi^2(s_1 G_1)  \big)  \Big]} \sqrt{\mathbb{E}\Big[ \sigma_w^2\psi^2(s_2 G_2)  \big)  \Big]}}  \\ 
    &= \frac{\mathbb{E}\Big[ \psi(G_1) \psi \big(G_1 \rho + G_2 \sqrt{1-\rho^2} \big)  \Big]}{\sqrt{\mathbb{E}\Big[ \psi^2(G_1)  \big)  \Big]} \sqrt{\mathbb{E}\Big[ \psi^2(G_2)  \big)  \Big]}}  \\ 
    &= 0 = \frac{\partial g_3}{\partial s_1^2} = \frac{\partial g_3}{\partial s_2^2}.
\end{align*}
Absolute homogeneity also implies that for all $a \in \mathbb{R}$, $\psi'(|a|z)=\psi'(z)$. Then by Theorem~\ref{thm:jacobian},
\begin{align*}
    \lambda_3 &= \frac{ \sigma_w^2s_1 s_2}{\sqrt{g_1 g_2}} \mathbb{E} \big[ \psi'( s_1 Z_1)  \psi'(s_2 Z_2)  \big] \\
    &= \frac{ \sigma_w^2s_1 s_2}{\sigma_w^2 s_1 s_2 \mathbb{E}\big[ \psi^2( G_1) \big]  } \mathbb{E} \big[ \psi'( Z_1)  \psi'(Z_2)  \big] \\
    &= \frac{\mathbb{E} \big[ \psi'( Z_1)  \psi'(Z_2)  \big] }{\mathbb{E}\big[ \psi^2( G_1) \big] }.
\end{align*}
Note that the Jacobian 
\begin{align*}
    \begin{pmatrix}
    \frac{\partial g_1}{\partial s_1^2} & \frac{\partial g_1}{\partial s_2^2} & \frac{\partial g_1}{\partial \rho} \\
    \frac{\partial g_2}{\partial s_1^2} & \frac{\partial g_2}{\partial s_2^2} & \frac{\partial g_2}{\partial \rho} \\
    \frac{\partial g_3}{\partial s_1^2} & \frac{\partial g_3}{\partial s_2^2} & \frac{\partial g_3}{\partial \rho}
    \end{pmatrix} = 
    \begin{pmatrix}
    \lambda_1 & 0 & 0 \\
    0 & \lambda_2 & 0 \\
    0 & 0 & \lambda_3
    \end{pmatrix}
\end{align*}
is diagonal, and therefore the induced matrix norm (corresponding to a Euclidean vector norm) of the Jacobian, the largest singular value of the Jacobian, is simply the largest absolute value of the diagonal elements. Thus by Theorem~\ref{thm:banach}, if $\max_{i=1,2,3} |\lambda_i| < 1$ then $\mathbf{g}$ has a unique fixed point. 

Alternatively, suppose $\lambda_3 < 1$ and $g_1(\cdot, s_2^2, \rho)$ admits a fixed point at $s^2$ for any $s_2^2, \rho$. Then $g_3\big( s^2, s^2, \cdot \big)$ admits a unique fixed point by applying Theorem~\ref{thm:banach} to the $1$D system $g_3$ with induced matrix norm $|\lambda_3|$. 
\end{proof}

Intuitively but informally, the absolute homogeneity condition in Corollary~\ref{cor:abs_hom_fp} leads to independent updates, so that $\mathbf{g}$ may be thought of as three functions $g_1, g_2, g_3$ whose inputs and outputs do not interact between iterations. When absolute homogeneity is removed, $g_1$ and $g_2$'s inputs and outputs are not affected by $g_3$, but $g_3$'s inputs are affected by the outputs of $g_1$ and $g_2$. This makes it more difficult to analyse exactly the same situation as in Corollary~\ref{cor:abs_hom_fp}. However, if we fix the output of $g_1$ and $g_2$ at some fixed point (which are guaranteed to exist in the cases handled in \S~\ref{sec:norm_preserve}), then we can neglect interactions between the outputs of $g_1, g_2$ and the inputs of $g_3$ and analyse the iterates of $g_3$ as a univariate function. We proceed with this strategy in Corollary~\ref{cor:g3_only}.

\begin{cor}[Unique fixed point of normalised kernel]
\label{cor:g3_only}
If a fixed point $s^2$ of the system only involving $g_1(\cdot, s_2^2, \rho): [0, \infty) \to [0, \infty)$ exists when $\sigma_w=\sigma^*$, we have
\begin{align*}
\frac{\partial g_3}{\partial \rho} = \lambda_3 = (\sigma^*)^2  \mathbb{E} \big[ \psi'( s Z_1)  \psi'( s Z_2)  \big]
\end{align*}
at the fixed point of $g_1(\cdot, s_2^2, \rho)$ and $g_2(s_1^2, \cdot, \rho)$. Furthermore, if $|\lambda_3| < 1$, then $g_3(s^2, s^2, \cdot): [-1, 1] \to [-1, 1]$ admits a unique fixed point at $\rho=1$.
\end{cor}
\begin{proof}
By Theorem~\ref{thm:jacobian}, we have
\begin{align*}
    \lambda_3 &= \frac{ \sigma_w^2s_1 s_2}{\sqrt{g_1 g_2}} \mathbb{E} \big[ \psi'( s_1 Z_1)  \psi'(s_2 Z_2)  \big] \\
    &=  (\sigma^*)^2  \mathbb{E} \big[ \psi'( s Z_1)  \psi'(s Z_2)  \big].
\end{align*}
$g_3(s^2, s^2, \cdot)$ admits a unique fixed point by taking $C=(-1,1)$ and $d$ as the the Euclidean metric in Theorem~\ref{thm:banach}.
\end{proof}

\subsection{Degenerate priors and posteriors}
\label{sec:degenerate}
Having established conditions under which unique fixed points exist, we now examine what a unique fixed point implies for the limiting prior and posterior. The prior and posteriors are degenerate in the sense that they are almost surely constant over subsets of the input space. We first consider the limiting prior as the depth goes to infinity.
\begin{restatable}{prop}{propDegeneratePrior}
\label{propDegeneratePrior}
Let $\{f^{(L)}(\mathbf{x})\}_{\mathbf{x} \in \mathcal{X}}$ be a Gaussian process with mean zero and covariance function $k^{(L)}$. Suppose that $\lim\limits_{L \to \infty} k^{(L)}(\mathbf{x}_1, \mathbf{x}_2)=\lim\limits_{L \to \infty} k^{(L)}(\mathbf{x}_1, \mathbf{x}_1)=\lim\limits_{L \to \infty} k^{(L)}(\mathbf{x}_2, \mathbf{x}_2)<\infty$ for every $\mathbf{x}_1$, $\mathbf{x}_2 \in \mathcal{X}_*$. Then $$\lim\limits_{L \to \infty} f^{(L)}(\mathbf{x}_1) - f^{(L)}(\mathbf{x}_2)=0$$ almost surely. That is, all draws from the limiting prior are almost surely a constant function.
\end{restatable}
The proof is given in Appendix~\ref{app:degenerate}. Suppose we take \emph{a priori} the Gaussian process in Proposition~\ref{propDegeneratePrior} \emph{before the limit is taken}, update our belief after observing some data to obtain the posterior \emph{and then} take the limit. An interesting question is whether the limit commutes with the Bayesian update.

\begin{restatable}{prop}{propDegeneratePosterior}
\label{propDegeneratePosterior}
Let $\{f^{(L)}(\mathbf{x})\}_{\mathbf{x} \in \mathcal{X}}$ be a Gaussian process prior with mean zero and covariance function $k^{(L)}$. Fix some $\mathcal{X}_* \subseteq \mathcal{X}$ such that for every $\mathbf{x}_1$, $\mathbf{x}_2 \in \mathcal{X}_* \subseteq \mathcal{X}$ and $\mathbf{x}_3, \mathbf{x}_4 \in \mathcal{X}$,
\begin{itemize}
\item$\lim\limits_{L \to \infty} k^{(L)}(\mathbf{x}_1, \mathbf{x}_2)=\lim\limits_{L \to \infty} k^{(L)}(\mathbf{x}_1, \mathbf{x}_1)=\lim\limits_{L \to \infty} k^{(L)}(\mathbf{x}_2, \mathbf{x}_2)=C<\infty$, 
\item $\lim\limits_{L \to \infty} k^{(L)}(\mathbf{x}_1, \mathbf{x}_3)=\lim\limits_{L \to \infty} k^{(L)}(\mathbf{x}_2, \mathbf{x}_3)=D(\mathbf{x}_3)$, and
\item $\lim\limits_{L \to \infty} k^{(L)}(\mathbf{x}_3, \mathbf{x}_4) < \infty$ exists
\end{itemize}
where $C \in \mathbb{R}$ and $D:\mathcal{X} \to \mathbb{R}$ may depend on $\mathcal{X}_*$. Fix some dataset $\mathbf{X}, \mathbf{Y}$, where each row $\mathbf{X}_i$ of $\mathbf{X}$ is in $\mathcal{X}$.

Then under Bayesian Gaussian process regression with Gaussian likelihood and strictly positive noise variance $\sigma_n^2>0$, all draws from the limiting posterior predictive distribution given observations $\mathbf{X}, \mathbf{Y}$ over $\mathcal{X}_*$ as $L \to \infty$ are almost surely a constant function.
\end{restatable}
The proof is given in Appendix~\ref{app:degenerate}. We are now ready to relate the existence of unique fixed points to degenerate priors and posteriors for some specific examples.

\subsubsection{Example, LReLU}
\label{sec:relu_fixed_point}
Taking $\mathcal{X}_*$ to be any (subset of a) hypersphere in Propositions~\ref{propDegeneratePrior} and~\ref{propDegeneratePosterior}, we have the following result, the proof of which is given in Appendix~\ref{app:degenerate}.

\begin{restatable}{cor}{corDegenerateLReLU}
\label{corDegenerateLReLU}
Let $\mathcal{X}_*$ be any hypersphere. Define a Gaussian process prior with a covariance function corresponding to an infinitely wide MLP with LReLU activations, $\bm{\mu}=\mathbf{0}$, $\sigma_w^2=2$ and depth $L$. Then as $L \to \infty$, draws from the prior and posterior predictive distributions are almost surely constant over $\mathcal{X}_*$.
\end{restatable}

\subsubsection{Examples, GELU and ELU}
\label{sec:gelu_elu_eigs}
In contrast with LReLU activations, such a degeneracy guarantee does not exist for GELU and ELU activations. For both the GELU and ELU, we consider the dynamics on a ball of constant $ \Vert \mathbf{x} \Vert$, where $\sigma_w$ is chosen such that $g_1 = \Vert \mathbf{x} \Vert$. In Figure~\ref{fig:gelu_elu_roots}, for different values of $\Vert \mathbf{x} \Vert$ in the context of Corollary~\ref{cor:g3_only}, we evaluate a lower bound for $\lambda_3$ in the case of GELU and $\lambda_3$ exactly in the case of ELU. Full working is given in Appendices~\ref{app:gelu_eigen} and~\ref{app:elu_eigen}. We observe that each exceeds $1$ \emph{at some point on the} (but not over the whole) interval, and is therefore not a contraction mapping and hence not guaranteed to have a unique fixed point. This is consistent with Figures~\ref{fig:gelu_kernel} and~\ref{fig:elu_kernel}, where fixed points are shown by intersecting curves.

\section{Extension of theoretical results to NTK}
\label{sec:ntk_jacob}
We may also study kernel fixed points of infinitely wide neural networks trained under gradient flow. This amounts to studying the fixed point properties of the neural tangent kernel (NTK). If such a unique fixed point exists, this implies that the functions obtained by applying gradient flow to an infinitely wide, infinitely deep MLP are also degenerate. In this section, we briefly sketch how such a result may be obtained. We leave the presentation of formal results and empirical evaluations for future work.

Informally, the value of the eigenvalue $\lambda_3$ can still predict the fixed point behaviour of the NTK. Formally, the result is slightly more involved, see Appendix~\ref{app:extension_ntk}. Consider a state space $\mathcal{S}\subset \mathbb{R}^4$ containing states $(s_1^2, s_2^2, k, T)$ consisting of squared norms for both inputs, the kernel, and the NTK. We update the states through $\mathbf{h}:\mathcal{S} \to \mathcal{S}$:
\allowdisplaybreaks
\begin{align*}
h_i(s_1^2, s_2^2, k, T) &= \sigma_w^2\mathbb{E}\big[ \psi^2( s_i Z_i) \big] + \sigma_b^2, \quad i=i,2\\
h_3(s_1^2, s_2^2, k, T) &= \sigma_w^2\mathbb{E}\big[ \psi( s_1 Z_1) \psi( s_2 Z_2) \big] + \sigma_b^2 , \\
h_4(s_1^2, s_2^2, k, T) &= T\sigma_w^2 \mathbb{E}\big[ \psi'( s_1 Z_1) \psi'( s_2 Z_2) \big] + \\
&\phantom{{}={}} \sigma_w^2\mathbb{E}\big[ \psi( s_1 Z_1) \psi( s_2 Z_2) \big] + \sigma_b^2, 
\end{align*} 
where $\text{Cov}(Z_1, Z_2) = k/(s_1s_2)$, $\mathbb{E}[Z_1]=\mathbb{E}[Z_2]=0$. As in Corollary~\ref{cor:g3_only}, if a fixed point of the system only involving $s_1$ and $h_1$ exists at a value $\sigma_w=\sigma^*$ and $s_1=s_2=s$, then the system reduces to a $2$ dimensional update involving only $h_3$ and $h_4$ along $s_1=s_2=s$. The Jacobian is diagonal,
\begin{align*}
J = \begin{pmatrix}
\frac{\partial h_3}{\partial k}  & \frac{\partial h_3}{\partial T} \\
\frac{\partial h_4}{\partial k}  & \frac{\partial h_4}{\partial T}
\end{pmatrix} &= \begin{pmatrix}
\frac{\partial h_3}{\partial k}  & 0 \\
\frac{\partial h_4}{\partial k}  & \frac{\partial h_4}{\partial T}
\end{pmatrix},
\end{align*}
so the eigenvalues are  $\frac{\partial h_3}{\partial k}$ and $\frac{\partial h_4}{\partial T}$. By Theorem~\ref{thm:jacobian}, 
\begin{align*}
&\frac{\partial h_3}{\partial k}  = \frac{\partial h_3 }{\partial \rho }\Big( \frac{\partial k}{\partial \rho} \Big)^{-1} = \frac{\partial g_3 }{\partial \rho }\sqrt{h_1 h_2} (s_1 s_2)^{-1}  \\
&= (\sigma^*)^2 \mathbb{E}\big[ \psi'( s_1 Z_1) \psi'( s_2 Z_2) \big] = \frac{\partial h_4}{\partial T}=\lambda_3.
\end{align*}

\section{Gaussian process experiments}
\label{sec:experiments}

We perform two sets of experiments. In the first, we investigate the performance of GPs with various neural network kernels. In the second, we observe the degree of implicit regularisation that is obtained using GPs of finite depth.

\subsection{Benchmarking} 
We provide a software implementation of our new kernels. To demonstrate usage of our covariance functions, first we compare the performance of GP regression models using ReLU, LReLU, ERF and GELU kernels on a popular Bayesian deep learning benchmark~\cite{HernandezLobato2015}. The purpose of these experiments is not to showcase the superiority of one prior over another, but rather to provide a sample implementation. This implementation has already been ported over to another framework in concurrent work by others~\citep{neuraltangents2020}. The ELU kernel was not included in our experiments (see \S~\ref{sec:new_kernel}). We perform separate experiments on shallow models having $1$ hidden layer and deep models having up to $32$ hidden layers. All data was standardised to have mean $0$ and variance $1$. 

\begin{figure}[t]
\centering
	 \includegraphics[scale=0.49]{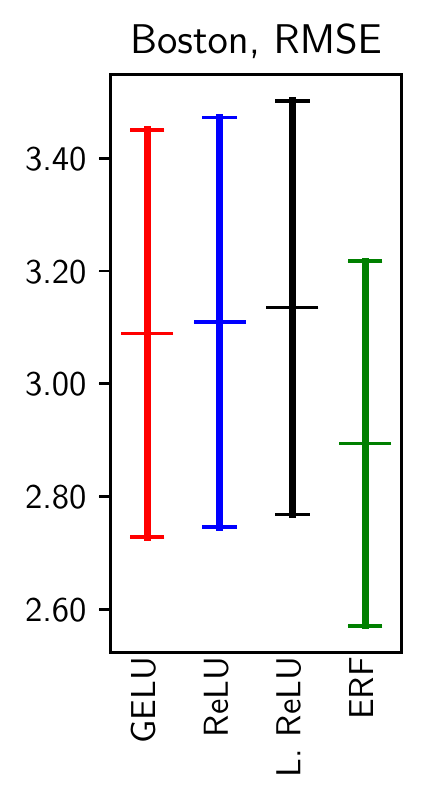}
	 \includegraphics[scale=0.49]{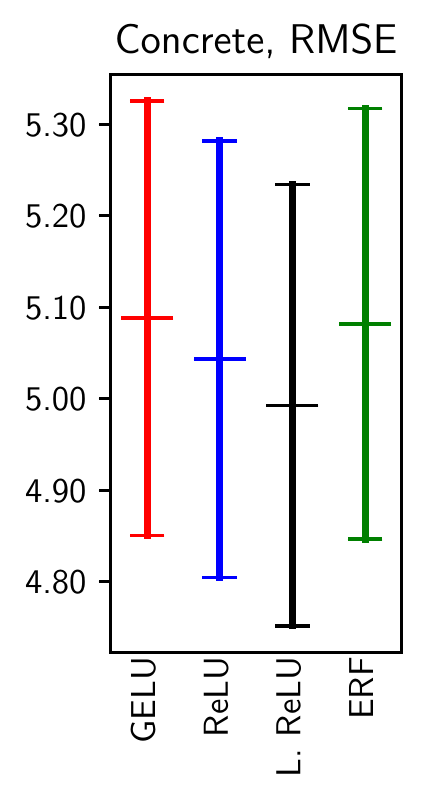}
	 \includegraphics[scale=0.49]{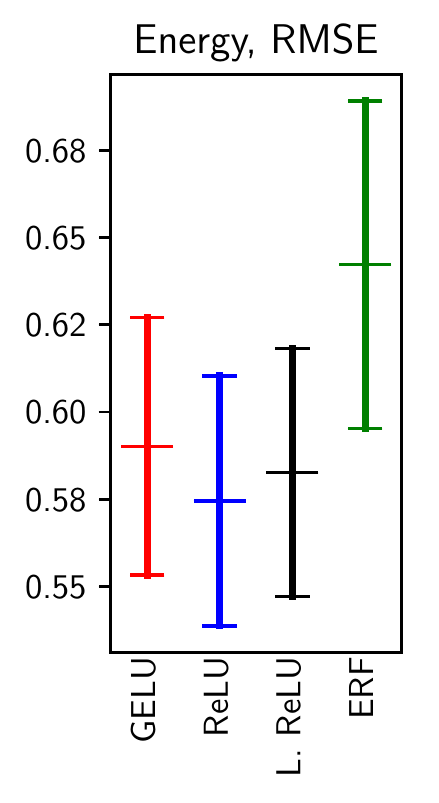} \\
	 \includegraphics[scale=0.49]{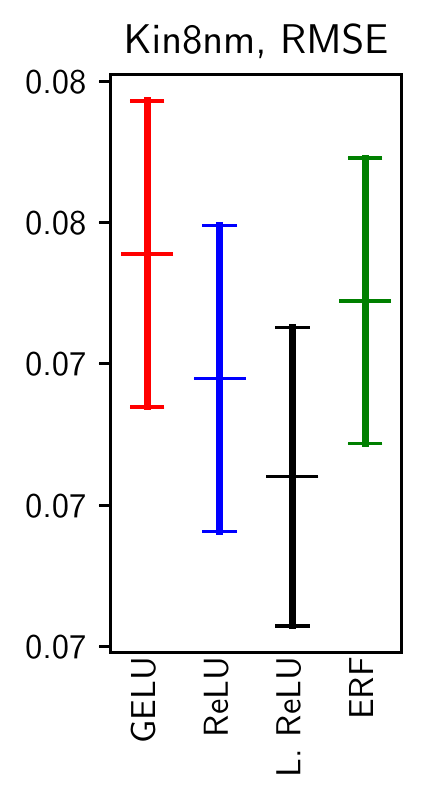}\hspace{-0.06in}
	 \includegraphics[scale=0.49]{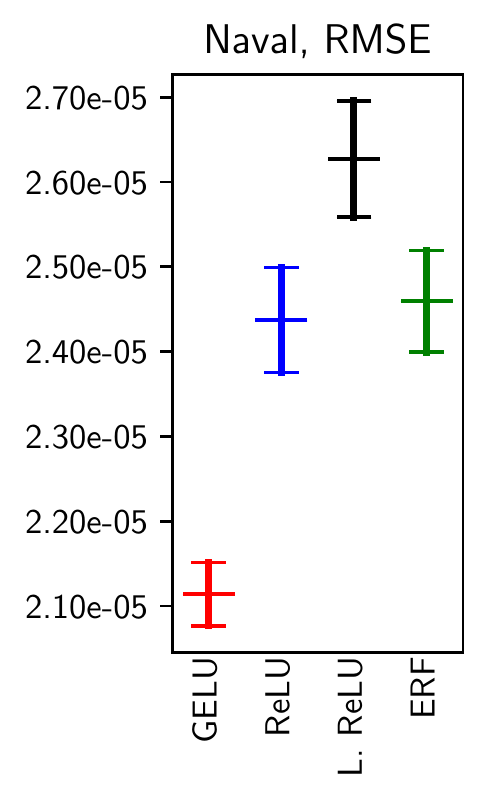}\hspace{-0.06in}
	 \includegraphics[scale=0.49]{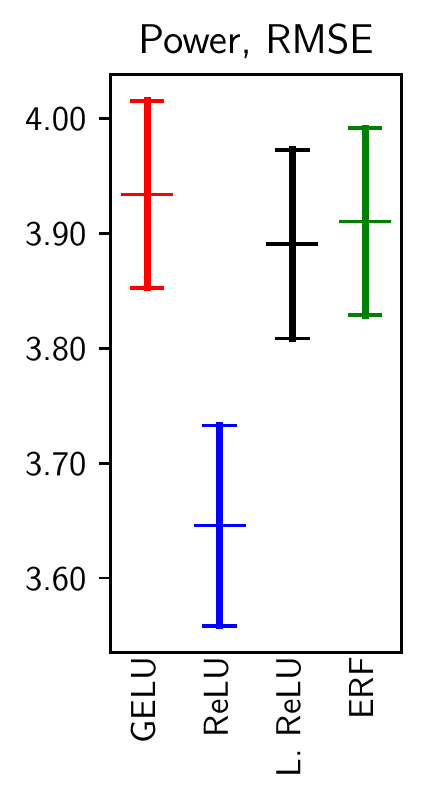} \\
	 \includegraphics[scale=0.49]{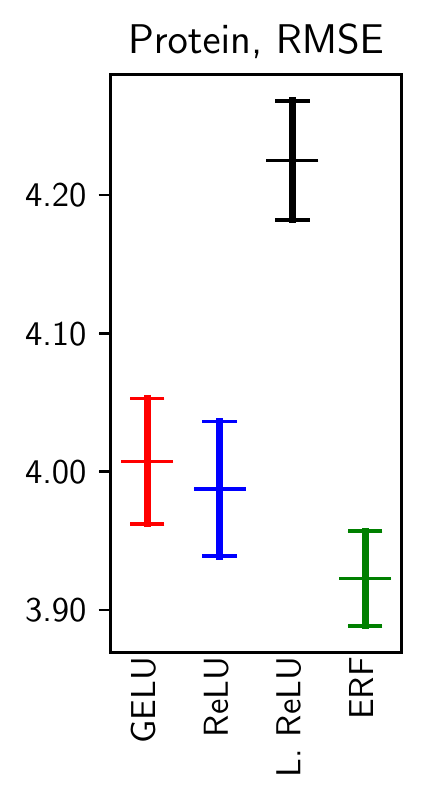}
	 \includegraphics[scale=0.49]{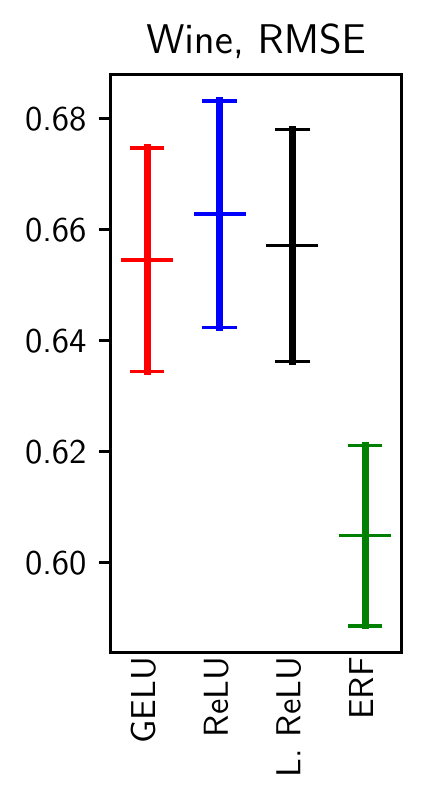}
	 \includegraphics[scale=0.49]{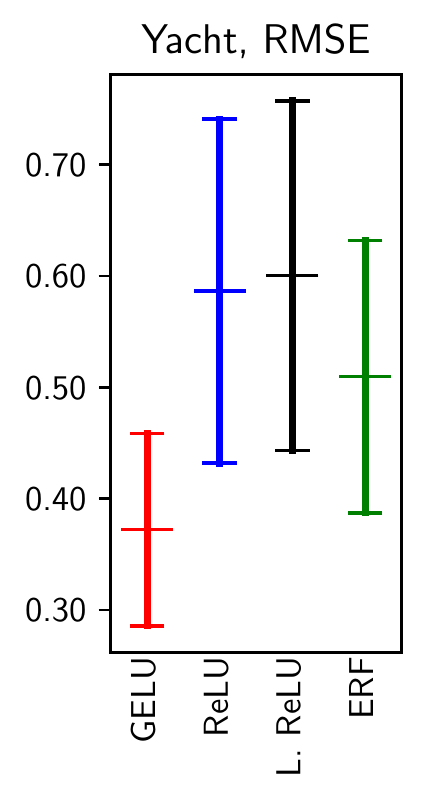}
\caption{RMSE for equivalent single-hidden-layer GPs. Mean $\pm 2$ standard errors (over 20 runs).}
\label{fig:bench_regression_RMSE}
\end{figure}

\begin{figure}[t]
\centering
\includegraphics[scale=0.35]{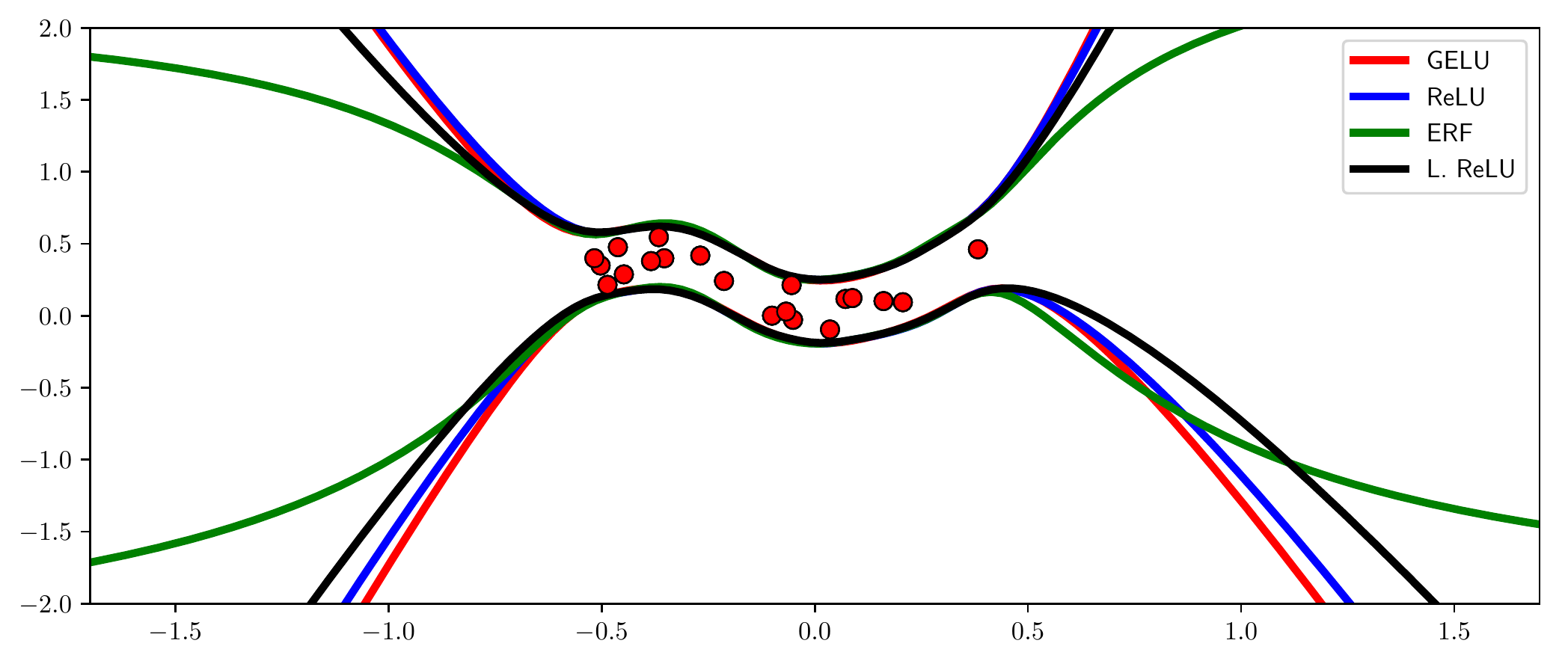}
\caption{Posterior predictive $\pm 2$ standard deviations.}
\label{fig:bench_regression_toy}
\end{figure}

\begin{figure}[t]
\centering
\begin{minipage}{0.9\columnwidth}
\includegraphics[scale=0.2]{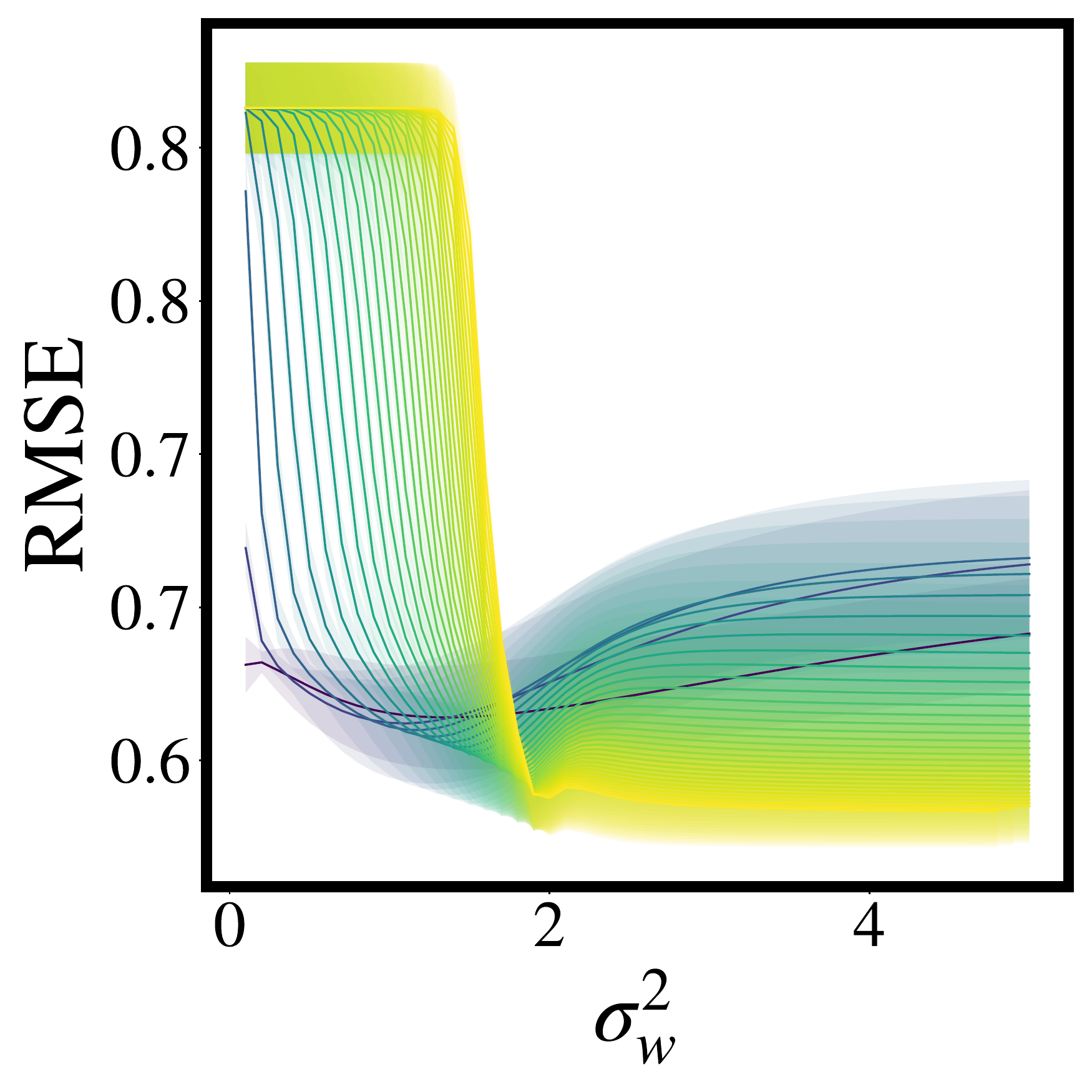} 
\includegraphics[scale=0.2]{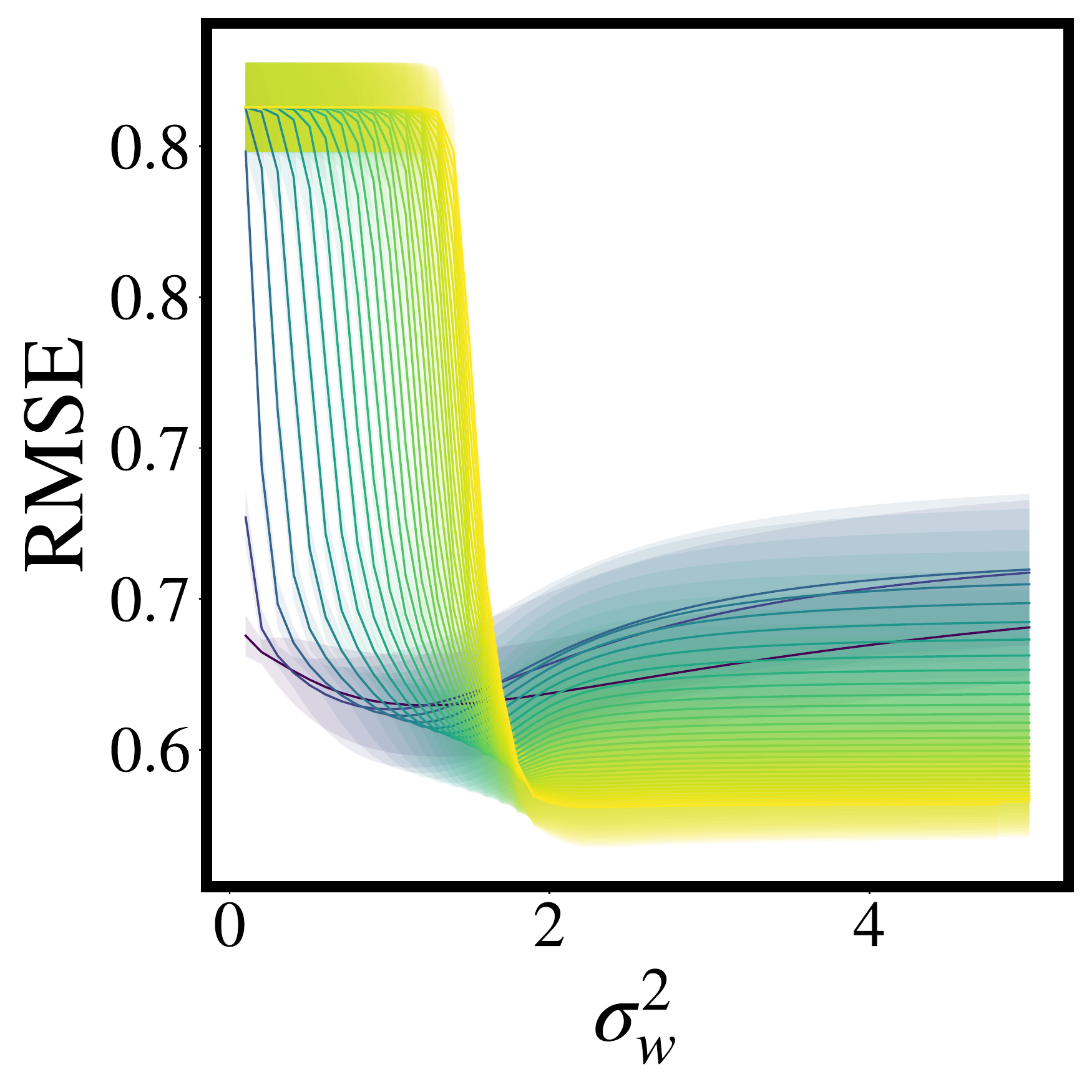}  \\
\includegraphics[scale=0.2]{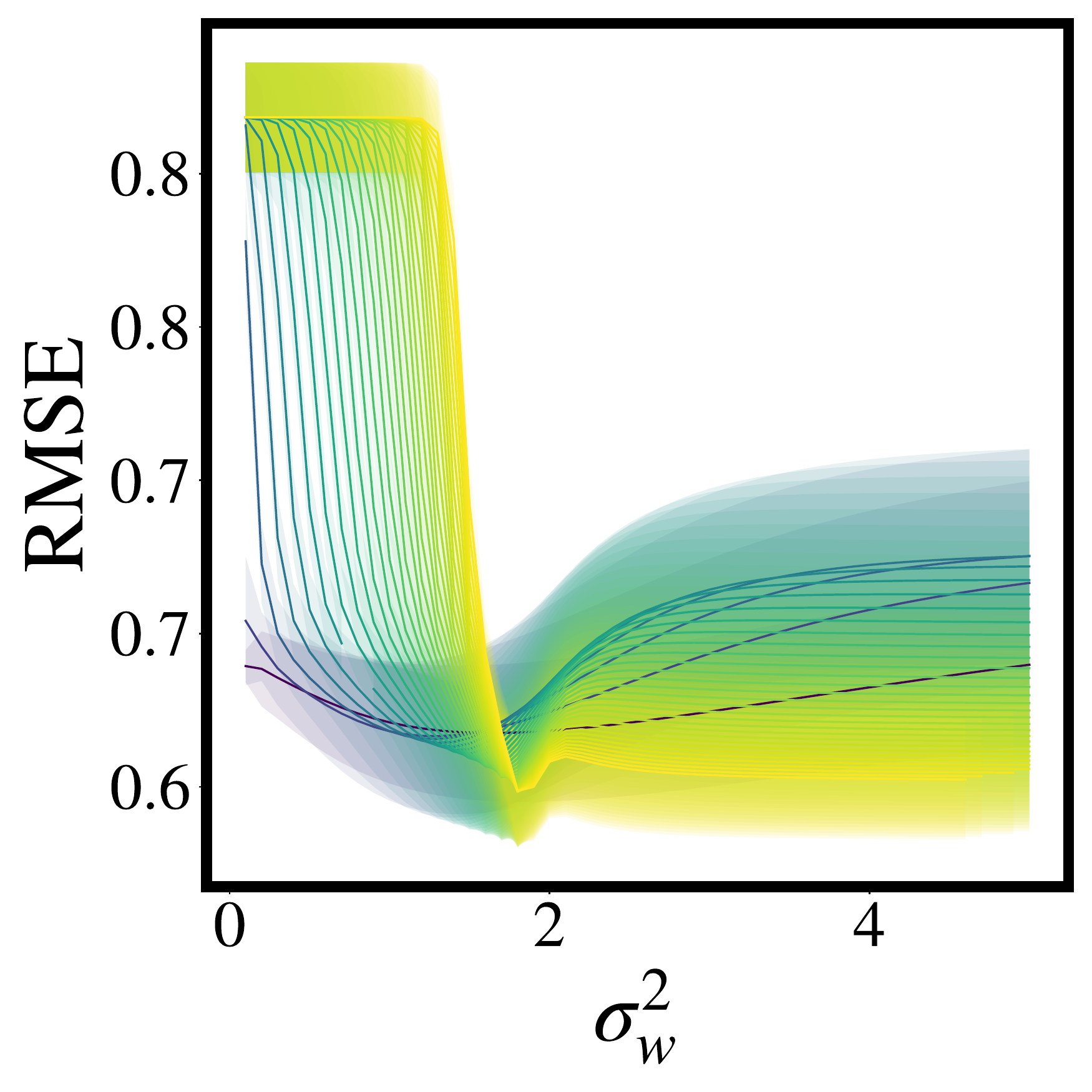} 
\includegraphics[scale=0.2]{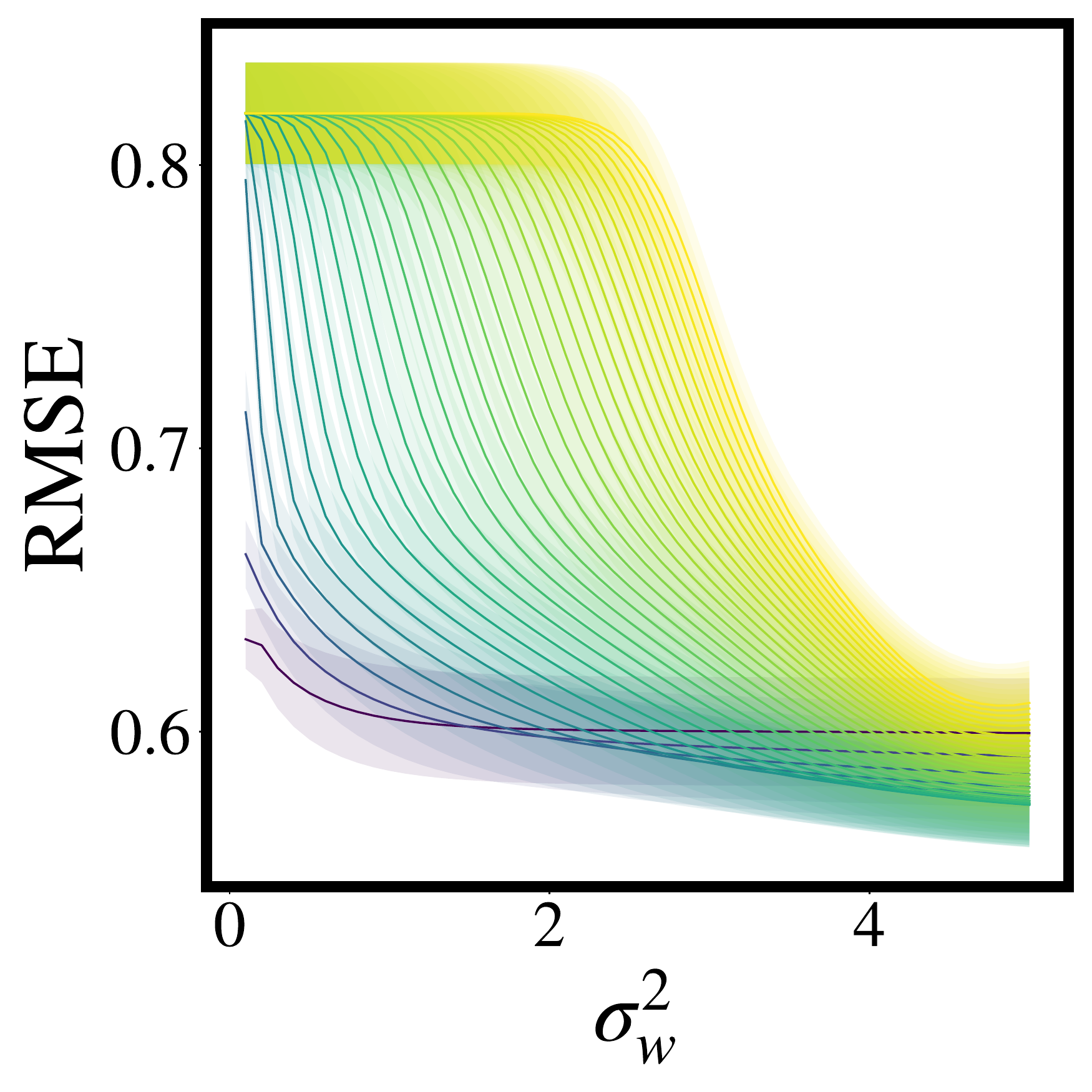} 
\end{minipage}
\begin{minipage}{0.09\columnwidth}
\hspace{-0.5cm}\includegraphics[scale=0.17]{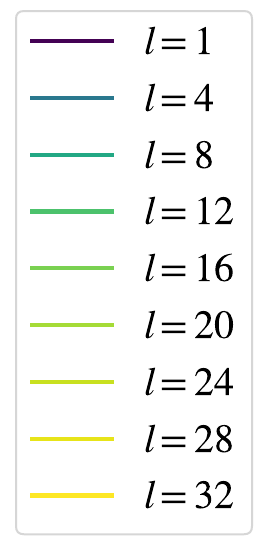}
\end{minipage}
\caption{ RMSE against $\sigma_w^2$ for equivalent $l$ layer GPs, Wine dataset. Shaded region shows $\pm 1$ standard deviation. (Clockwise from top left) ReLU, GELU, LReLU, ERF.}
\label{fig:bench_regression_deep}
\end{figure}

\begin{table*}
\centering
\caption{Best performing models for each kernel over the grid search.}
\resizebox{\textwidth}{!}{
\begin{tabular}{l|ccc|ccc|ccc|ccc}
& &ReLU & &&GELU & && LReLU && & ERF & \\
         & RMSE         & $\sigma_w^2$ & $\ell$  & RMSE         & $\sigma_w^2$ & $\ell$ & RMSE         & $\sigma_w^2$ & $\ell$ & RMSE         & $\sigma_w^2$ & $\ell$\\ \hline
Boston & $ 2.85 \pm 0.64 $ & $ 1.90 $ & $ 7 $ & $ 2.86 \pm 0.65 $ & $ 1.80 $ & $ 6 $ & $ \mathbf{2.60 \pm 1.07} $ & $ \mathbf{2.00} $ & $ \mathbf{32} $ & $ 2.69 \pm 0.95 $ & $ 5.00 $ & $ 2 $ \\
Concrete & $ 5.22 \pm 0.55 $ & $ 5.00 $ & $ 2 $ & $ \mathbf{5.21 \pm 0.56} $ & $ \mathbf{5.00} $ & $ \mathbf{2} $ & $ 5.23 \pm 0.44 $ & $ 3.30 $ & $ 3 $ & $ 5.63 \pm 0.46 $ & $ 5.00 $ & $ 2 $ \\
Energy & $ \mathbf{0.89 \pm 0.11} $ & $ \mathbf{5.00} $ & $ \mathbf{2} $ & $ 0.92 \pm 0.12 $ & $ 5.00 $ & $ 2 $ & $ 2.77 \pm 0.31 $ & $ 0.10 $ & $ 1 $ & $ 2.79 \pm 0.23 $ & $ 0.10 $ & $ 1 $  \\
Wine & $ 1.15 \pm 0.13 $ & $ 5.00 $ & $ 4 $ & $ 1.17 \pm 0.14 $ & $ 5.00 $ & $ 4 $ & $ \mathbf{1.04 \pm 0.12} $ & $\mathbf{ 5.00} $ & $ \mathbf{5} $ & $ 3.96 \pm 0.75 $ & $ 5.00 $ & $ 1 $ \\
Yacht & $ 0.58 \pm 0.01 $ & $ 4.80 $ & $ 32 $ & $ 0.58 \pm 0.01 $ & $ 2.30 $ & $ 32 $ & $ 0.60 \pm 0.02 $ & $ 1.80 $ & $ 29 $ & $ \mathbf{0.57 \pm 0.02} $ & $ \mathbf{5.00} $ & $ \mathbf{8} $
\end{tabular}}
\label{tab:deep}
\end{table*}
\normalsize 

\paragraph{Shallow models.} \emph{Do differences in priors induced by the various activation functions affect empirical performance?} Using the limiting GP allows us to remove the interaction between $\psi$ and optimisation, and purely consider the effect of $\psi$ on the functional prior. Figure \ref{fig:bench_regression_toy} shows the predictive distribution of GPs with GELU, ReLU, LReLU and ERF kernels on a toy regression task. ERF has different extrapolation properties due to being a bounded activation, whilst the others appear qualitatively similar, though with extrapolation variance decreasing in the order GELU/ReLU/LReLU.

Figure~\ref{fig:bench_regression_RMSE} shows benchmark results for single-hidden-layer GPs using a 90\%/10\% training/test split. See Appendix \ref{sec:appendix_shallow_models} for more details and plots. All kernels perform comparably; gains can be made by selecting a kernel suited to the dataset. Results are most different for ERF --- either negatively (Concrete, Energy) or positively (Boston, Protein, Wine). Differences are observed between GELU/ReLU/LReLU. For example, GELU offers an advantage in Naval and Yacht, and LReLU performs poorly on Protein. 

None of the kernels consistently outperform the others. This is expected behaviour, similar to how a Matern kernel might outperform a squared exponential kernel only some of the time on real-world datasets. The purpose of the experiment was to evaluate whether different $\psi$ result in a strong enough difference in priors that empirical performance differences can be observed. Having answered in the positive, we posit that for finite-width networks, the difference in performance found by varying $\psi$ may partially derive from differences in the induced prior. This is in contrast to previously cited reasons such as bias shift and its relation to natural gradient~\cite{clevert2015fast}.

\paragraph{Deep models. }\emph{How does the performance of models vary with depth?}
We randomly shuffled the data into an $80/20\%$ train/test split $5$ times. For each split, we ran GP regression with an additive iid Gaussian noise model having variance fixed at $0.1$. We varied the depth $\ell\in [1,32]$ in steps of $1$ and the weight and bias variances (which were constrained to be equal in each layer) $\sigma_w^2\in [0.1, 5]$ in steps of $0.1$. For each setting, we measured the RMSE between the mean of the GP prediction and the true regression targets.
\interfootnotelinepenalty=10000
Figure~\ref{fig:bench_regression_deep} shows the average RMSE on the Wine dataset over $5$ shuffles. Other datasets are given in Appendix~\ref{app:deep}. We make two qualitative observations. Firstly, $\ell>1$ models out-perform $\ell=1$ models. Secondly, the RMSE changes smoothly in both depth and $\sigma_w^2$. The visual smoothness is \emph{not} due to averaging over $5$ trials; smoothness is also observed when we plot results from only $1$ random shuffling. Table~\ref{tab:deep} shows the best models obtained over the grid search for each kernel.

\subsection{Overfitting and underfitting}
We empirically investigate the relationship between depth and training/testing error. When the covariance function has a unique fixed point, we expect to see underfitting at large depth, since large depth will push the kernel towards the unique fixed point, that is, a constant normalised kernel. On the other hand, when the covariance function does not have a unique fixed point, we might expect to see overfitting as model complexity may increase with depth. 

We build predictor variables of a training dataset by uniformly sampling $\mathbf{x} \in \mathbb{R}^2$ on the unit disc at heading $\gamma$. We then sample training targets through the mapping $y=f(\gamma) + \epsilon$ for a number of different choices of $f$, where $\epsilon$ is additive Gaussian noise with variance fixed at $0.1$. We find the posterior predictive mean of a Gaussian process with iterated GELU and ReLU covariance functions of depth $L$ between $1$ and $100$ with a choice of $\sigma_w$ according to Figure~\ref{fig:gelu_elu_roots}. We repeat this process with a new random training set $10$ times. Each repetition, we find the mean-squared error (MSE) between the posterior mean of the Gaussian process over the training set and a test set built from a (deterministic) uniform grid of size $100$. Figure~\ref{fig:test_train_overfit} shows the resulting train and test errors on one choice of $f$, and Appendix~\ref{app:train_test_curves} shows other choices of $f$. Figure~\ref{fig:simplicity} shows a more direct illustration of the function fit on one of the random data samples, with other choices of $f$ shown in Appendix~\ref{app:train_test_curves}.

\begin{figure}[t]
    \centering
    \includegraphics[scale=0.25]{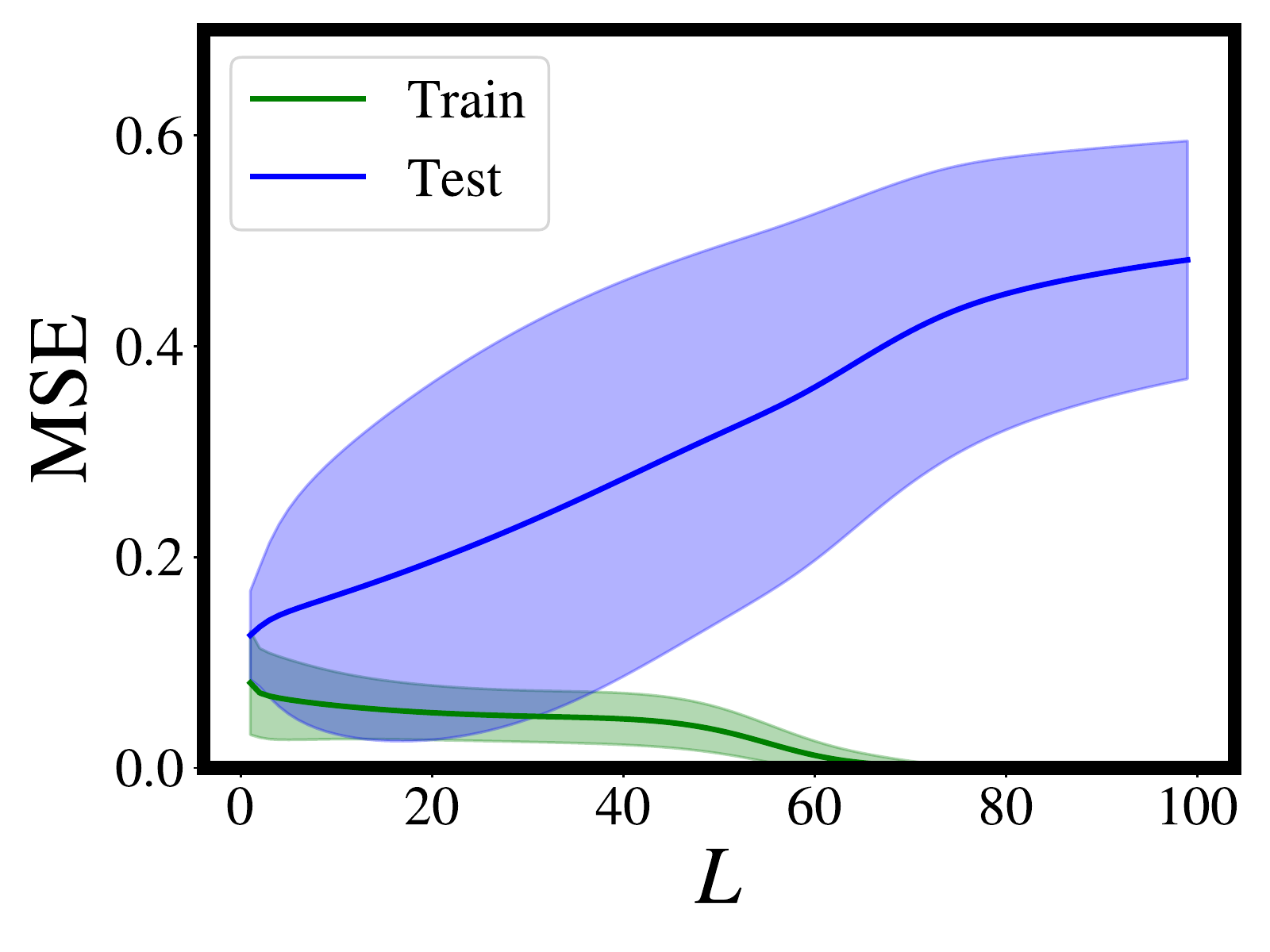}
    \includegraphics[scale=0.25]{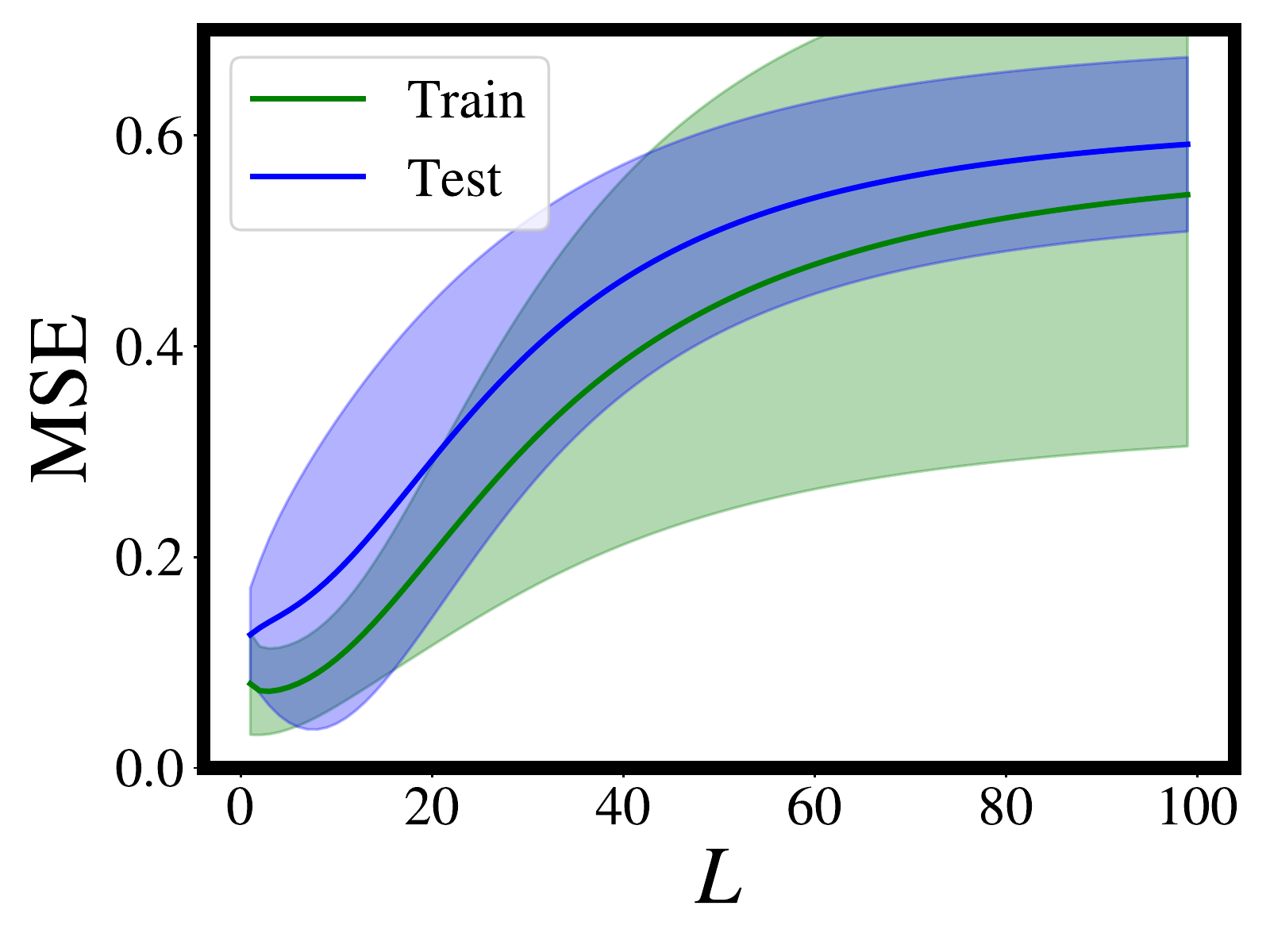}
    \caption{Training and testing errors for GPs with covariance functions corresponding to infinitely wide MLPs of increasing depth $L$ using the same example as in Figure~\ref{fig:simplicity}. Solid curve shows the mean over $10$ training data samples, and the shaded region shows $\pm$ two standard deviations. By examining whether the training and testing error increases or decreases with depth, we observe that the GELU and ReLU respectively overfit and underfit with depth. See Appendix~\ref{app:train_test_curves} for more error curves. (Left) GELU (Right) ReLU. }
    \label{fig:test_train_overfit}
\end{figure}

\section{Discussion and conclusion}
We introduced two new positive semi-definite kernels arising from the infinite-width limit of Bayesian GELU or ELU NNs. We provided visualisations of these kernels for varying depths. We  introduced a general framework for understanding the fixed-point dynamics of such kernels and their NTK counterparts. Using this framework, we showed that unlike the ReLU, the GELU and ELU kernels are able to avoid unique fixed points. We empirically verified that finite-width NNs are able to avoid unique kernel fixed points in Figures~\ref{fig:gelu_kernel} and~\ref{fig:elu_kernel}. We applied our kernels in the setting of shallow and deep GP regression, finding that for some problems specific kernels are more appropriate, and that the GELU kernel is competitive with the ReLU kernel.

Investigations into implicit regularisation consider the role of one or all of (a) the architecture, (b) the learning algorithm and (c) the data sampling process. \citet{neyshabur2014search} argue that (b) leads implicitly to low-norm solutions, explaining the generalisation ability of deep NNs. On the other hand,~\citet{derezinski2019exact} construct a data sampling distribution (c) that explains double descent and implicit regularisation in linear models. Similar to our work but not considering the NTK, the signal propagation literature~\cite{schoenholz2016deep, NIPS2016_6322} explains simplicity biases in randomly initialised networks (a). They develop objects similar to $\frac{\partial g_3}{\partial \rho}$, but require bounded activations, and seem to also require some notion of differentiability. Our analysis considers (a) and (b) but not (c). We have recently been made aware of the concurrent work of~\cite{huang2020deep}, who also study the degeneracy of processes induced through ReLU activations. Their focus is on both (a) and (b), arguing that the NTK for residual architectures does not suffer from this degeneracy.

Knowing the gradient of the kernel with respect to $(\sigma_w^{(l)})^2$ and $(\sigma_b^{(l)})^2$ is useful for both empirical Bayesian methodologies (e.g. optimising the marginal likelihood using LBFGS) and hierarchical models (e.g. using HMC to integrate out the hyperprior). As we detail in Appendix~\ref{app:chainrule}, our results may be used to find this gradient. Theorem~\ref{thm:jacobian} provides $7$ of the $9$ elements of the Jacobian. The other $2$ can only easily be evaluated in special cases (e.g. LReLU results in a diagonal Jacobian). In future work, it may be interesting to extend Theorem~\ref{thm:jacobian} to cover the remaining elements of the Jacobian.

While~\citet{lee2019wide} found close agreement between NNs and their corresponding limiting GPs, several authors~\cite{neal1995bayesian, mackay2003information, der2006beyond, matthews2018gaussian, chizat2019lazy, tsuchida2019richer, allen2019learning, favaro2020stable, aitchison2020bigger} have argued against the use of GP models as a means to understand the success of deep learning. If deep learning's performance can be explained using GPs, why do NN models outperform their limiting GP counterparts?~\citet{arora2019exact} attain $77\%$ test accuracy on CIFAR10 using a limiting GP arising from a trained CNN, while a ResNet is able to achieve $96\%$~\cite{springenberg2015striving}. On the other hand,~\citet{arora2019harnessing} find that GP models are competitive on \emph{small} datasets. It remains to determine if this difference in performance is due to the tricks for which equivalence in the GP setting have not yet been fully explored, or if it is the result of some deeper property of GPs.~\citet{lee2020finite} empirically explore some of these questions including the effects of finite width, architectures, weight decay and the performance difference between infinite Bayesian and NTK models. While we acknowledge the limitations of the infinitely wide approach, we believe it warrants further exploration, if not to understand the power of deep learning, at least to investigate its generalisation abilities. The purpose of our study was not to optimise any architecture for performance on a particular problem, but rather to develop results under the GP framework that contribute to our understanding of generalisation in the overparameterised setting.

\section*{Acknowledgements}
This work was partially funded by CSIRO's Machine Learning and Artificial Intelligence Future Science Platform. TP was funded through EPSRC (EP/N509620/1). CvdH was supported by ACEMS under ARC grant number CE140100049. We would like to thank Bob Williamson for a helpful discussion.

\begin{quote}
\begin{small}
\bibliography{main}

\begin{thebibliography}{53}
\providecommand{\natexlab}[1]{#1}
\providecommand{\url}[1]{\texttt{#1}}
\providecommand{\urlprefix}{URL }
\expandafter\ifx\csname urlstyle\endcsname\relax
  \providecommand{\doi}[1]{doi:\discretionary{}{}{}#1}\else
  \providecommand{\doi}{doi:\discretionary{}{}{}\begingroup
  \urlstyle{rm}\Url}\fi

\bibitem[{Agarwal, Meehan, and O'regan(2001)}]{agarwal2001fixed}
Agarwal, R.~P.; Meehan, M.; and O'regan, D. 2001.
\newblock \emph{Fixed point theory and applications}, volume 141.
\newblock Cambridge university press.

\bibitem[{Aitchison(2020)}]{aitchison2020bigger}
Aitchison, L. 2020.
\newblock Why bigger is not always better: on finite and infinite neural
  networks.
\newblock In \emph{International Conference on Machine Learning}, 156--164.

\bibitem[{Allen-Zhu, Li, and Liang(2019)}]{allen2019learning}
Allen-Zhu, Z.; Li, Y.; and Liang, Y. 2019.
\newblock Learning and generalization in overparameterized neural networks,
  going beyond two layers.
\newblock In \emph{Advances in neural information processing systems},
  6155--6166.

\bibitem[{Arora et~al.(2019)Arora, Du, Hu, Li, Salakhutdinov, and
  Wang}]{arora2019exact}
Arora, S.; Du, S.~S.; Hu, W.; Li, Z.; Salakhutdinov, R.~R.; and Wang, R. 2019.
\newblock On exact computation with an infinitely wide neural net.
\newblock In \emph{Advances in Neural Information Processing Systems},
  8139--8148.

\bibitem[{Arora et~al.(2020)Arora, Du, Li, Salakhutdinov, Wang, and
  Yu}]{arora2019harnessing}
Arora, S.; Du, S.~S.; Li, Z.; Salakhutdinov, R.; Wang, R.; and Yu, D. 2020.
\newblock Harnessing the Power of Infinitely Wide Deep Nets on Small-data
  Tasks.
\newblock \emph{International Conference on Learning Representations} .

\bibitem[{Billingsley(1995)}]{billingsley2008probability}
Billingsley, P. 1995.
\newblock \emph{Probability and measure}.
\newblock John Wiley \& Sons, 3 edition.

\bibitem[{Bui et~al.(2016)Bui, Hern{\'a}ndez-Lobato, Hernandez-Lobato, Li, and
  Turner}]{bui2016deep}
Bui, T.; Hern{\'a}ndez-Lobato, D.; Hernandez-Lobato, J.; Li, Y.; and Turner, R.
  2016.
\newblock Deep Gaussian processes for regression using approximate expectation
  propagation.
\newblock In \emph{International conference on machine learning}, 1472--1481.

\bibitem[{Chizat, Oyallon, and Bach(2019)}]{chizat2019lazy}
Chizat, L.; Oyallon, E.; and Bach, F. 2019.
\newblock On Lazy Training in Differentiable Programming.
\newblock In \emph{Advances in Neural Information Processing Systems}.

\bibitem[{Cho and Saul(2009)}]{NIPS2009_3628}
Cho, Y.; and Saul, L.~K. 2009.
\newblock Kernel Methods for Deep Learning.
\newblock In \emph{Advances in Neural Information Processing Systems},
  342--350.

\bibitem[{Chua et~al.(2018)Chua, Calandra, McAllister, and
  Levine}]{chua2018deep}
Chua, K.; Calandra, R.; McAllister, R.; and Levine, S. 2018.
\newblock Deep reinforcement learning in a handful of trials using
  probabilistic dynamics models.
\newblock In \emph{Advances in Neural Information Processing Systems},
  4754--4765.

\bibitem[{Clevert, Unterthiner, and Hochreiter(2016)}]{clevert2015fast}
Clevert, D.-A.; Unterthiner, T.; and Hochreiter, S. 2016.
\newblock Fast and accurate deep network learning by exponential linear units
  ({ELU}s).
\newblock \emph{The International Conference on Learning Representations} .

\bibitem[{Der and Lee(2006)}]{der2006beyond}
Der, R.; and Lee, D.~D. 2006.
\newblock Beyond {G}aussian processes: On the distributions of infinite
  networks.
\newblock In \emph{Advances in Neural Information Processing Systems},
  275--282.

\bibitem[{Derezi{\'n}ski, Liang, and Mahoney(2019)}]{derezinski2019exact}
Derezi{\'n}ski, M.; Liang, F.; and Mahoney, M.~W. 2019.
\newblock Exact expressions for double descent and implicit regularization via
  surrogate random design.
\newblock \emph{arXiv preprint arXiv:1912.04533} .

\bibitem[{Devlin et~al.(2019)Devlin, Chang, Lee, and
  Toutanova}]{devlin2019bert}
Devlin, J.; Chang, M.; Lee, K.; and Toutanova, K. 2019.
\newblock BERT: Pre-training of Deep Bidirectional Transformers for Language
  Understanding.
\newblock In \emph{Proceedings of the 2019 Conference of the North American
  Chapter of the Association for Computational Linguistics: Human Language
  Technologies}, 4171--4186.

\bibitem[{Elfwing, Uchibe, and Doya(2018)}]{elfwing2018sigmoid}
Elfwing, S.; Uchibe, E.; and Doya, K. 2018.
\newblock Sigmoid-weighted linear units for neural network function
  approximation in reinforcement learning.
\newblock \emph{Neural Networks} 107: 3--11.

\bibitem[{Favaro, Fortini, and Peluchetti(2020)}]{favaro2020stable}
Favaro, S.; Fortini, S.; and Peluchetti, S. 2020.
\newblock Stable behaviour of infinitely wide deep neural networks.
\newblock \emph{arXiv preprint arXiv:2003.00394} .

\bibitem[{Garriga-Alonso, Rasmussen, and Aitchison(2018)}]{garriga2018deep}
Garriga-Alonso, A.; Rasmussen, C.~E.; and Aitchison, L. 2018.
\newblock Deep convolutional networks as shallow {G}aussian processes.
\newblock \emph{arXiv preprint arXiv:1808.05587} .

\bibitem[{Hasselblatt and Katok(2003)}]{hasselblatt2003first}
Hasselblatt, B.; and Katok, A. 2003.
\newblock \emph{A first course in dynamics: with a panorama of recent
  developments}.
\newblock Cambridge University Press.

\bibitem[{He et~al.(2015)He, Zhang, Ren, and Sun}]{he2015delving}
He, K.; Zhang, X.; Ren, S.; and Sun, J. 2015.
\newblock Delving deep into rectifiers: Surpassing human-level performance on
  imagenet classification.
\newblock In \emph{Proceedings of the IEEE international conference on computer
  vision}, 1026--1034.

\bibitem[{Hendrycks and Gimpel(2016)}]{hendrycks2016gaussian}
Hendrycks, D.; and Gimpel, K. 2016.
\newblock Gaussian error linear units ({GELU}s).
\newblock \emph{arXiv preprint arXiv:1606.08415} .

\bibitem[{Hern{\'{a}}ndez-Lobato and Adams(2015)}]{HernandezLobato2015}
Hern{\'{a}}ndez-Lobato, J.~M.; and Adams, R.~P. 2015.
\newblock {Probabilistic Backpropagation for Scalable Learning of {B}ayesian
  Neural Networks}.
\newblock In \emph{Proceedings of the 32nd International Conference on Machine
  Learning}.

\bibitem[{Huang et~al.(2020)Huang, Wang, Tao, and Zhao}]{huang2020deep}
Huang, K.; Wang, Y.; Tao, M.; and Zhao, T. 2020.
\newblock Why Do Deep Residual Networks Generalize Better than Deep Feedforward
  Networks?--A Neural Tangent Kernel Perspective.
\newblock \emph{Advances in Neural Information Processing Systems} 33.

\bibitem[{Jacot, Gabriel, and Hongler(2018)}]{jacot2018neural}
Jacot, A.; Gabriel, F.; and Hongler, C. 2018.
\newblock Neural tangent kernel: Convergence and generalization in neural
  networks.
\newblock In \emph{Advances in neural information processing systems},
  8571--8580.

\bibitem[{Jones(1982)}]{JonesD.S.DouglasSamuel1982Ttog}
Jones, D.~S. 1982.
\newblock \emph{The theory of generalised functions}.
\newblock Cambridge ; New York: Cambridge University Press, 2nd ed. edition.

\bibitem[{Klambauer et~al.(2017)Klambauer, Unterthiner, Mayr, and
  Hochreiter}]{klambauer2017self}
Klambauer, G.; Unterthiner, T.; Mayr, A.; and Hochreiter, S. 2017.
\newblock Self-normalizing neural networks.
\newblock In \emph{Advances in neural information processing systems},
  971--980.

\bibitem[{Le~Roux and Bengio(2007)}]{le2007continuous}
Le~Roux, N.; and Bengio, Y. 2007.
\newblock Continuous neural networks.
\newblock In \emph{Artificial Intelligence and Statistics}, 404--411.

\bibitem[{Lee et~al.(2018)Lee, Bahri, Novak, Schoenholz, Pennington, and
  Sohl-Dickstein}]{lee2017deep}
Lee, J.; Bahri, Y.; Novak, R.; Schoenholz, S.~S.; Pennington, J.; and
  Sohl-Dickstein, J. 2018.
\newblock Deep neural networks as {G}aussian processes.
\newblock \emph{The International Conference on Learning Representations} .

\bibitem[{Lee et~al.(2020)Lee, Schoenholz, Pennington, Adlam, Xiao, Novak, and
  Sohl-Dickstein}]{lee2020finite}
Lee, J.; Schoenholz, S.; Pennington, J.; Adlam, B.; Xiao, L.; Novak, R.; and
  Sohl-Dickstein, J. 2020.
\newblock Finite versus infinite neural networks: an empirical study.
\newblock \emph{Advances in Neural Information Processing Systems} 33.

\bibitem[{Lee et~al.(2019)Lee, Xiao, Schoenholz, Bahri, Novak, Sohl-Dickstein,
  and Pennington}]{lee2019wide}
Lee, J.; Xiao, L.; Schoenholz, S.; Bahri, Y.; Novak, R.; Sohl-Dickstein, J.;
  and Pennington, J. 2019.
\newblock Wide neural networks of any depth evolve as linear models under
  gradient descent.
\newblock In \emph{Advances in neural information processing systems},
  8570--8581.

\bibitem[{MacKay(2003)}]{mackay2003information}
MacKay, D.~J. 2003.
\newblock \emph{Information theory, inference and learning algorithms}, 547.
\newblock Cambridge university press.

\bibitem[{Matthews et~al.(2018)Matthews, Rowland, Hron, Turner, and
  Ghahramani}]{matthews2018gaussian}
Matthews, A. G. d.~G.; Rowland, M.; Hron, J.; Turner, R.~E.; and Ghahramani, Z.
  2018.
\newblock Gaussian process behaviour in wide deep neural networks.
\newblock \emph{The International Conference on Learning Representations} .

\bibitem[{Meronen, Irwanto, and Solin(2020)}]{meronen2020stationary}
Meronen, L.; Irwanto, C.; and Solin, A. 2020.
\newblock Stationary Activations for Uncertainty Calibration in Deep Learning.
\newblock \emph{Advances in Neural Information Processing Systems} 33.

\bibitem[{Neal(1995)}]{neal1995bayesian}
Neal, R.~M. 1995.
\newblock \emph{Bayesian learning for neural networks}.
\newblock Ph.D. thesis, University of Toronto.

\bibitem[{Neyshabur, Tomioka, and Srebro(2015)}]{neyshabur2014search}
Neyshabur, B.; Tomioka, R.; and Srebro, N. 2015.
\newblock In search of the real inductive bias: On the role of implicit
  regularization in deep learning.
\newblock \emph{International Conference on Learning Representations (workshop
  track)} .

\bibitem[{Novak et~al.(2020)Novak, Xiao, Hron, Lee, Alemi, Sohl-Dickstein, and
  Schoenholz}]{neuraltangents2020}
Novak, R.; Xiao, L.; Hron, J.; Lee, J.; Alemi, A.~A.; Sohl-Dickstein, J.; and
  Schoenholz, S.~S. 2020.
\newblock Neural Tangents: Fast and Easy Infinite Neural Networks in Python.
\newblock In \emph{International Conference on Learning Representations}.
\newblock \urlprefix\url{https://github.com/google/neural-tangents}.

\bibitem[{Novak et~al.(2019)Novak, Xiao, Lee, Bahri, Yang, Hron, Abolafia,
  Pennington, and Sohl-Dickstein}]{novak2018bayesian}
Novak, R.; Xiao, L.; Lee, J.; Bahri, Y.; Yang, G.; Hron, J.; Abolafia, D.~A.;
  Pennington, J.; and Sohl-Dickstein, J. 2019.
\newblock Bayesian Deep Convolutional Networks with Many Channels are
  {G}aussian Processes.
\newblock \emph{The International Conference on Learning Representations} .

\bibitem[{Pearce et~al.(2019)Pearce, Tsuchida, Zaki, Brintrup, and
  Neely}]{pearce2019expressive}
Pearce, T.; Tsuchida, R.; Zaki, M.; Brintrup, A.; and Neely, A. 2019.
\newblock Expressive Priors in {B}ayesian Neural Networks: Kernel Combinations
  and Periodic Functions.
\newblock \emph{Uncertainty in Artificial Intelligence} .

\bibitem[{Poole et~al.(2016)Poole, Lahiri, Raghu, Sohl-Dickstein, and
  Ganguli}]{NIPS2016_6322}
Poole, B.; Lahiri, S.; Raghu, M.; Sohl-Dickstein, J.; and Ganguli, S. 2016.
\newblock Exponential expressivity in deep neural networks through transient
  chaos.
\newblock In Lee, D.~D.; Sugiyama, M.; Luxburg, U.~V.; Guyon, I.; and Garnett,
  R., eds., \emph{Advances in Neural Information Processing Systems 29},
  3360--3368.

\bibitem[{Radford et~al.(2018)Radford, Narasimhan, Salimans, and
  Sutskever}]{radford2018improving}
Radford, A.; Narasimhan, K.; Salimans, T.; and Sutskever, I. 2018.
\newblock Improving language understanding by generative pre-training.
\newblock \emph{Preprint} .

\bibitem[{Ramachandran, Zoph, and Le(2017)}]{ramachandran2017searching}
Ramachandran, P.; Zoph, B.; and Le, Q.~V. 2017.
\newblock Searching for activation functions.
\newblock \emph{arXiv preprint arXiv:1710.05941} .

\bibitem[{Rosenbaum(1961)}]{rosenbaum1961moments}
Rosenbaum, S. 1961.
\newblock Moments of a truncated bivariate normal distribution.
\newblock \emph{Journal of the Royal Statistical Society: Series B
  (Methodological)} 23(2): 405--408.

\bibitem[{Schoenholz et~al.(2017)Schoenholz, Gilmer, Ganguli, and
  Sohl-Dickstein}]{schoenholz2016deep}
Schoenholz, S.~S.; Gilmer, J.; Ganguli, S.; and Sohl-Dickstein, J. 2017.
\newblock Deep information propagation.
\newblock \emph{International Conference on Learning Representations} .

\bibitem[{Sheppard(1899)}]{sheppard1899iii}
Sheppard, W.~F. 1899.
\newblock On the application of the theory of error to cases of normal
  distribution and normal correlation.
\newblock \emph{Philosophical Transactions of the Royal Society of London.
  Series A, Containing Papers of a Mathematical or Physical Character} (192):
  140.

\bibitem[{Springenberg et~al.(2015)Springenberg, Dosovitskiy, Brox, and
  Riedmiller}]{springenberg2015striving}
Springenberg, J.; Dosovitskiy, A.; Brox, T.; and Riedmiller, M. 2015.
\newblock Striving for Simplicity: The All Convolutional Net.
\newblock In \emph{International Conference on Learning Representations
  (workshop track)}.

\bibitem[{Tsuchida, Roosta, and Gallagher(2018)}]{tsuchida2018invariance}
Tsuchida, R.; Roosta, F.; and Gallagher, M. 2018.
\newblock Invariance of Weight Distributions in Rectified {MLP}s.
\newblock In \emph{International Conference on Machine Learning}, 5002--5011.

\bibitem[{Tsuchida, Roosta, and
  Gallagher(2019{\natexlab{a}})}]{tsuchida2018exchangeability}
Tsuchida, R.; Roosta, F.; and Gallagher, M. 2019{\natexlab{a}}.
\newblock Exchangeability and Kernel Invariance in Trained {MLP}s.
\newblock \emph{International Joint Conference on Artificial Intelligence} .

\bibitem[{Tsuchida, Roosta, and
  Gallagher(2019{\natexlab{b}})}]{tsuchida2019richer}
Tsuchida, R.; Roosta, F.; and Gallagher, M. 2019{\natexlab{b}}.
\newblock Richer priors for infinitely wide multi-layer perceptrons.
\newblock \emph{arXiv preprint arXiv:1911.12927} .

\bibitem[{Valle-P{\'e}rez, Camargo, and Louis(2019)}]{valle2019deep}
Valle-P{\'e}rez, G.; Camargo, C.~Q.; and Louis, A.~A. 2019.
\newblock Deep learning generalizes because the parameter-function map is
  biased towards simple functions.
\newblock \emph{International Conference on Learning Representations} .

\bibitem[{Williams(1997)}]{williams1997computing}
Williams, C.~K. 1997.
\newblock Computing with infinite networks.
\newblock In \emph{Advances in neural information processing systems},
  295--301.

\bibitem[{Yang(2019{\natexlab{a}})}]{yang2019scaling}
Yang, G. 2019{\natexlab{a}}.
\newblock Scaling limits of wide neural networks with weight sharing:
  {G}aussian process behavior, gradient independence, and neural tangent kernel
  derivation.
\newblock \emph{arXiv preprint arXiv:1902.04760} .

\bibitem[{Yang(2019{\natexlab{b}})}]{NIPS2019_9186}
Yang, G. 2019{\natexlab{b}}.
\newblock Wide Feedforward or Recurrent Neural Networks of Any Architecture are
  {G}aussian Processes.
\newblock In \emph{Advances in Neural Information Processing Systems},
  9947--9960.

\bibitem[{Yang and Salman(2019)}]{yang2019fine}
Yang, G.; and Salman, H. 2019.
\newblock A fine-grained spectral perspective on neural networks.
\newblock \emph{arXiv preprint arXiv:1907.10599} .

\bibitem[{Zhang et~al.(2017)Zhang, Bengio, Hardt, Recht, and
  Vinyals}]{zhang2017understanding}
Zhang, C.; Bengio, S.; Hardt, M.; Recht, B.; and Vinyals, O. 2017.
\newblock Understanding deep learning requires rethinking generalization.
\newblock \emph{International Conference on Learning Representations} .

\end{thebibliography}
\end{small}
\end{quote}
\ifthenelse{\boolean{make_appendix}}{
\onecolumn
\appendix
\section{Derivation of GELU kernel}
\label{app:gelu_kernel}
We would like to evaluate the kernel (up to a scaling and offset)
\begin{align*}
k(\mathbf{x}_1, \mathbf{x}_2) = s_1 s_2 \mathbb{E}\Big[ & Z_1 Z_2 \Phi(s_1 Z_1) \Phi \big( s_2 Z_2 \big) \Big].
\end{align*}
We split our evaluation of the kernel into two cases, first when $\theta \in (0, \pi)$, and second when $\theta\in \{0, \pi\}$.
\subsection{$\theta \in (0, \pi)$}
We introduce dummy variables $\beta_1, \beta_2$ and define
\begin{align*}
\kappa(\beta_1, \beta_2) = s_1 s_2 \mathbb{E}\Big[ Z_1 Z_2 \Phi \Big( \beta_1 s_1  Z_1 \Big) \Phi \Big( \beta_2s_2 Z_2 \Big) \Big].
\end{align*}
The mixed derivative $\frac{\partial^2 \kappa}{\partial\beta_1\,\partial\beta_2}$ is
\begin{align*}
& s_1^2 s_2^2 \mathbb{E}\Big[ Z_1^2 Z_2^2 \phi\Big( \beta_1 s_1 Z_1 \Big) \phi \Big( \beta_2 s_2 Z_2 \Big) \Big] \\
&= s_1^2 s_2^2 \int z_1^2 z_2^2 \phi\Big( \beta_1 s_1 z_1 \Big) \phi \Big( \beta_2 s_2 z_2 \Big) \phi_2(z_1, z_2) \,dz_1 \,dz_2 \numberthis \label{eq:mix_der}.
\end{align*}
The product of the normal PDFs is given by
\begin{align*}
\frac{1}{2\pi\sin\theta}\Big(\frac{1}{\sqrt{2\pi}}\Big)^2\exp\Big( -\frac{1}{2} (\mathbf{z}^\top S^{-1} \mathbf{z} + \mathbf{z}^\top \beta \mathbf{z} ) \Big),
\end{align*}
where $S^{-1} = \frac{1}{\sin^2\theta} \begin{bmatrix}
   1      & -\cos\theta  \\
   -\cos\theta       & 1
\end{bmatrix}$ and $\beta =  \begin{bmatrix}
    s_1^2 \beta_1^2      & 0  \\
    0       & s_2^2 \beta_2^2
\end{bmatrix}$. Now $S^{-1} + \beta$ is the inverse of a positive definite covariance matrix, with
\begin{align*}
S^{-1} + \beta &= \frac{1}{\sin^2\theta} \begin{bmatrix}
   1 + s_1^2\beta_1^2 \sin^2\theta     & -\cos\theta  \\
   -\cos\theta       & 1 + s_2^2\beta_2^2 \sin^2\theta
\end{bmatrix},
\end{align*}
having determinant
\begin{align*}
\csc^4\theta \big( (1 + s_1^2\beta_1^2\sin^2\theta)(1 + s_2^2\beta_2^2\sin^2\theta ) - \cos^2\theta \big)
\end{align*}
and inverse
\begin{align*}
&\big( (1 + s_1^2\beta_1^2\sin^2\theta)(1 + s_2^2\beta_2^2\sin^2\theta ) - \cos^2\theta \big)^{-1} \\
&\phantom{{}={}} \sin^2\theta \begin{bmatrix}
   1 + s_2^2\beta_2^2\sin^2\theta     & \cos\theta  \\
   \cos\theta       & 1 + s_1^1\beta_2^2\sin^2\theta
\end{bmatrix}.
\end{align*}
We may therefore write~\eqref{eq:mix_der} as 
\begin{align*}
\frac{s_1^2 s_2^2}{2\pi\sin\theta} \det(S^{-1}+\beta)^{-\frac{1}{2}} \mathbb{E} [U_1^2 U_2^2],
\end{align*}
where $(U_1, U_2)^\top$ has covariance matrix $C=(S^{-1}+\beta)^{-1}$. The expectation $\mathbb{E} [U_1^2 U_2^2]$ has a known form, and is given by
\begin{align*}
\mathbb{E} [U_1^2 U_2^2] &= C_{11}C_{22}+2C_{12}^2 \\
&= \big( (1 + s_1^2\beta_1^2\sin^2\theta)(1 + s_2^2\beta_2^2\sin^2\theta ) - \cos^2\theta \big)^{-2} \\
&\phantom{{}={}} \big( (1 + s_1^2\beta_1^2\sin^2\theta)(1 + s_2^2\beta_2^2\sin^2\theta) + 2\cos^2\theta \big)\\
&\phantom{{}={}}\sin^4\theta .
\end{align*}
Finally,
\begin{align*}
\frac{\partial^2 \kappa}{\partial\beta_1\,\partial\beta_2} &=  \frac{\sin^5\theta s_1^2 s_2^2}{2\pi }\\
 &\Big[ \big( (1 + s_1^2\beta_1^2\sin^2\theta)(1 + s_2^2\beta_2^2\sin^2\theta ) - \cos^2\theta \big)^{-5/2} \\
&\big( (1 + s_1^2\beta_1^2\sin^2\theta)(1 + s_2^2\beta_2^2\sin^2\theta ) + 2\cos^2\theta \big) \Big].
\end{align*}
We also have the boundary conditions
\begin{align*}
\frac{\partial \kappa}{\partial \beta_1} \Big|_{\beta_2 = 0} = \frac{\partial \kappa}{\partial \beta_2} \Big|_{\beta_1 = 0} &= 0, \quad \text{ and} \\
\kappa(0,0) &= \frac{s_1 s_2 }{4} \cos\theta.
\end{align*}
The solution to this PDE can be found by direct integration with integration constants due to the conditions. The solution evaluated at $\beta_1=\beta_2=1$ is 
\begin{align*}
&\frac{s_1^2 s_2^2}{2\pi }\Big[ \frac{\frac{1}{2}(\cos(2\theta) + 3) + s_1^2 + s_2^2 +s_1^2s_2^2 \sin^2\theta}{(1+s_1^2)(1+s_2^2)\sqrt{1+s_1^2+s_2^2+s_1^2s_2^2\sin^2\theta}}\\
+ &\frac{\cos\theta}{s_1s_2} \tan^{-1}\Big( \frac{\cos\theta s_1s_2}{ \sqrt{1+s_1^2+s_2^2+s_1^2s_2^2\sin^2\theta}} \Big) \Big] + \frac{s_1s_2}{4}\cos\theta.
\end{align*}
\subsection{$\theta \in \{0, \pi\}$}
\label{app:open_interval}
We may simply evaluate the result obtained on $(0, \pi)$ at $0$ and $\pi$. To see this, observe that $k$ is continuous with respect to $\theta$ on $[0, \pi]$. Firstly,
\begin{align*}
\psi( s_1 Z_1 ) \psi \big( s_2 Z_2) &\leq \max\{ \psi^2( s_1 Z_1),  \psi^2( s_2 Z_2) \} \\
&\leq \psi^2( s_1 Z_1) + \psi^2( s_2 Z_2),
\end{align*}
which has finite expectation. Let $G_1 \independent G_2 \sim \mathcal{N}(0, 1)$.  By dominated convergence, 
\begin{align*}
&\phantom{{}={}} \lim_{\theta \to 0} k(\cos \theta) \\
&= \mathbb{E}\big[\lim_{\theta \to 0} \psi\big( s_1G_1 \big) \psi \big( s_2 (\cos\theta G_1 + \sqrt{1-\cos^2\theta} G_2) \big)  \big] \\
&= \mathbb{E}\big[\psi\big( s_1G_1 \big)  \psi\big( s_2G_1 \big) \big] \\
&= k(1),
\end{align*}
and similarly for the case $\theta \to \pi$.

\section{Derivation of ELU kernel}
\label{app:elu_kernel}
We would like to evaluate the kernel (up to a scaling and offset)
\begin{align*}
\mathbb{E} \Big[ &\Big( \lambda \Theta(\widetilde{Z}_1 {+} \widetilde{\mu}_1) \big( \widetilde{Z}_1 {+} \widetilde{\mu}_1 \big)  + \lambda \alpha \Theta(-\widetilde{Z}_1 {-} \widetilde{\mu}_1)\big( e^{\widetilde{Z}_1 + \widetilde{\mu}_1} {-} 1\big) \Big) \\
&\Big( \lambda\Theta(\widetilde{Z}_2 {+} \widetilde{\mu}_2) \big(  \widetilde{Z}_2 {+} \widetilde{\mu}_2 \big) + \lambda \alpha \Theta(-\widetilde{Z}_2 {-} \widetilde{\mu}_2)\big( e^{\widetilde{Z}_2 + \widetilde{\mu}_2} {-} 1\big) \Big) \Big],
\end{align*}
where $\widetilde{Z}_i = s_i Z_i$. Note we have included parameters $\lambda$ and $\alpha$ corresponding to the SELU, which recovers the ELU when $\lambda=\alpha=1$. We may expand the expectation into the sum of
\begin{align*}
E_1 &= \lambda^2 \mathbb{E} \Big[ \Theta(\widetilde{Z}_1 {+} \widetilde{\mu}_1) ( \widetilde{Z}_1 {+} \widetilde{\mu}_1 ) \Theta(\widetilde{Z}_2 {+} \widetilde{\mu}_2) (\widetilde{Z}_2 {-} \widetilde{\mu}_2 ) \Big], \numberthis \label{eq:E1} \\
E_2 &= \lambda^2 \alpha \mathbb{E} \Big[  \Theta(\widetilde{Z}_1 {+} \widetilde{\mu}_1) ( \widetilde{Z}_1 + \widetilde{\mu}_1 )  \Theta(-\widetilde{Z}_2 {-} \widetilde{\mu}_2)\big( e^{\widetilde{Z}_2 {+} \widetilde{\mu}_2} {-} 1\big) \Big], \numberthis \label{eq:E2} \\
E_3 &= \lambda^2 \alpha \mathbb{E} \Big[  \Theta(\widetilde{Z}_2 {+} \widetilde{\mu}_2) ( \widetilde{Z}_2 {+} \widetilde{\mu}_2 )  \Theta(-\widetilde{Z}_1 {-} \widetilde{\mu}_1)\big( e^{\widetilde{Z}_1 + \widetilde{\mu}_1} {-} 1\big) \Big], \numberthis \label{eq:E3}
\end{align*}
and
\begin{align*}
E_4 = \lambda^2 \alpha^2 \mathbb{E} \Big[ &\Theta(-\widetilde{Z}_1 {-} \widetilde{\mu}_1)\big( e^{\widetilde{Z}_1 + \widetilde{\mu}_1} {-} 1\big) \Theta(-\widetilde{Z}_2 {-} \widetilde{\mu}_2)\big( e^{\widetilde{Z}_2 + \widetilde{\mu}_2} {-} 1\big) \Big]. \numberthis \label{eq:E4}
\end{align*}
In the following sections we evaluate each term in the sum. It suffices to evaluate the kernel on the open interval $(0, \pi)$ by the same argument as in \S~\ref{app:open_interval}.
\subsection{$E_1$}
The integral~\eqref{eq:E1} is a non-trivial generalisation of the arc-cosine kernel of degree $1$~\cite{NIPS2009_3628} with $\bm{\mu}\neq0$, and is evaluated by~\citet{tsuchida2019richer}. 
\subsection{$E_2$ and $E_3$}
The second integral in~\eqref{eq:E2} is
\begin{align*}
\frac{s_1}{2\pi \sin\theta} \int_{\mathbb{R}^2} & \Theta(z_1 + \widetilde{\mu}_1/s_1) ( z_1 + \widetilde{\mu}_1/s_1 )  \Theta(-z_2 - \widetilde{\mu}_2/s_2)\\
&\big( e^{s_2 z_2 - \widetilde{\mu}_2} - 1\big) \exp\Big( -\frac{1}{2} \mathbf{z}^\top S^{-1} \mathbf{z} \Big) \, dz_1 \, dz_2.
\end{align*}
We may complete the square of the exponentiated terms,
\begin{align*}
&\phantom{{}={}}\exp\Big( -\frac{1}{2} \big( \mathbf{z}^\top S^{-1} \mathbf{z} - 2s_2z_2 \big) \Big) \\
&= \exp\Big( -\frac{1}{2} \Big( \big( \mathbf{z}- s_2 S_{:,2} \big)^\top S^{-1} \big( \mathbf{z} - s_2 S_{:,2} \big) - s_2^2 \Big) \Big),
\end{align*}
where $S_{:,2}$ denotes the second column of $S$. $E_2$ is then
\begin{align*}
&\phantom{{}={}} e^{-\widetilde{\mu}_2+s_2^2/2} s_1  \mathbb{E} \Big[ \Theta(Z_1 + s_2 \cos \theta + \widetilde{\mu}_1/s_1) \\
&\phantom{{}={}}(Z_1 + s_2 \cos \theta + \widetilde{\mu}_1/s_1) \Theta(-Z_2- s_2 - \widetilde{\mu}_2/s_2) \Big] - \\
&\phantom{{}={}}s_1 \mathbb{E} \Big[  \Theta(Z_1 + \widetilde{\mu}_1/s_1) (Z_1 + \widetilde{\mu}_1/s_1 )  \Theta(-Z_2 - \widetilde{\mu}_2/s_2) \Big].
\end{align*}

Both of these expectations can be related to the first moment of the truncated standard bivariate normal distribution, which has a known form. Let  $(Y_1, Y_2)$ be distributed according to the standard bivariate normal distribution with correlation $-\cos\theta$, and let $h, k \in \mathbb{R}$. Defining $M$ for convenience as follows,~\citet{rosenbaum1961moments} gives
\begin{align*}
&\phantom{{}={}}M(h,k) \\
&= \mathbb{E} \Big[  \Theta(Y_1-h) Y_1  \Theta(Y_2-k) \Big] \\
&= \frac{1}{2\pi\sin\theta} \int_k^\infty \int_h^\infty z_1 \exp\Big( - \frac{1}{2\sin^2\theta}( z_1^2 + 2\cos\theta z_1 z_2 + z_2^2) \Big) \, dz_1 \,dz_2 \\
&= \phi(h) \Bigg( 1- \bm{\Phi}\Big( \frac{k + h \cos\theta }{\sin\theta} \Big) \Bigg)  -  \cos\theta \phi(k) \Bigg( 1- \bm{\Phi} \Big( \frac{h + k \cos\theta}{\sin\theta} \Big) \Bigg).
\end{align*}
Therefore,
\begin{align*}
&\phantom{{}={}}s_1 \mathbb{E} \Big[  \Theta(Z_1 + \widetilde{\mu}_1/s_1) (Z_1 + \widetilde{\mu}_1/s_1)   \Theta(-Z_2-\widetilde{\mu}_2/s_2) \Big] \\
&=  s_1\Bigg( M \Big( -\frac{\widetilde{\mu}_1}{s_1},\frac{\widetilde{\mu}_2}{s_2} \Big) + \frac{\widetilde{\mu}_1}{s_1} \Phi\Big( -\frac{\widetilde{\mu}_1}{s_1}, \frac{\widetilde{\mu}_2}{s_2}; -\cos\theta \Big)  \Bigg),
\end{align*}
and
\begin{align*}
&\phantom{{}={}} \mathbb{E} \Big[ \Theta(Z_1 + s_2 \cos \theta + \widetilde{\mu}_1/s_1)(Z_1 + s_2 \cos \theta + \widetilde{\mu}_1/s_1)  \Theta(-Z_2- s_2 - \widetilde{\mu}_2/s_2) \Big] \\
&=  M\Big(-s_2 \cos\theta - \frac{\widetilde{\mu}_1}{s_1}, s_2 + \frac{\widetilde{\mu}_2}{s_2} \Big) +  \Big( s_2 \cos+ \frac{\widetilde{\mu}_1}{s_1} \Big)  \Phi \Big( s_2\cos\theta- \frac{\widetilde{\mu}_1}{s_1}, -s_2+ \frac{\widetilde{\mu}_2}{s_2}; -\cos\theta \Big).
\end{align*}
\subsection{$E_4$}
\begin{align*}
E_4 &= \lambda^2 \alpha^2 \mathbb{E} \Big[ \Theta(-Z_1-\widetilde{\mu}_1/s_1) \Theta(-Z_2-\widetilde{\mu}_2/s_2)\big( e^{s_1 Z_1 + s_2 Z_2+\widetilde{\mu}_1+\widetilde{\mu}_2} - e^{s_1 Z_1+\widetilde{\mu}_1} - e^{s_2 Z_2+\widetilde{\mu}_2} + 1\big) \Big].
\end{align*}
Each of the four terms may be understood as scales of special cases of the function 
\begin{align*}
e_4(a,b) &= \mathbb{E} \Big[ \Theta(-Z_1-\widetilde{\mu}_1/s_1) \Theta(-Z_2-\widetilde{\mu}_2/s_2) e^{a Z_1 + bZ_2} \Big]
\end{align*}
for $a, b \in \mathbb{R}$. Completing the square, $e_4(a,b)$ is given by
\begin{align*}
&\phantom{{}={}} \int_{-\infty}^{-\widetilde{\mu}_2/s_2} \int_{-\infty}^{-\widetilde{\mu}_1/s_1}  \frac{1}{2\pi  \sin\theta}  \exp\Big( -\frac{1}{2} (\mathbf{z}^\top \Sigma^{-1} \mathbf{z} - 2az_1 - 2bz_2) \Big) \,dz_1 \,dz_2 \\
&= \int_{\widetilde{\mu}_2/s_2}^\infty \int_{\widetilde{\mu}_1/s_1}^\infty  \frac{1}{2\pi \sin\theta}  \exp\Big( -\frac{1}{2} (\mathbf{z}^\top \Sigma^{-1} \mathbf{z} + 2az_1 + 2bz_2) \Big) \,dz_1 \,dz_2 \\
&= \frac{\exp\Big( \frac{1}{2} (a,b) \Sigma (a,b)^\top \Big)}{2\pi  \sin\theta} \int_{\widetilde{\mu}_2/s_2}^\infty \int_{\widetilde{\mu}_1/s_1}^\infty \,dz_1 \,dz_2 \exp\Big( -\frac{1}{2} \big( \mathbf{z} + \Sigma (a,b)^\top \big)^\top \Sigma^{-1} \big( \mathbf{z} + \Sigma (a,b)^\top \big) \Big)   \\
&= \Phi\big((\widetilde{\mu}_1/s_1, \widetilde{\mu}_2/s_2)^\top-\Sigma (a,b)^\top ; \cos\theta \big) \exp\Big( \frac{1}{2} (a,b) \Sigma (a,b)^\top \Big).
\end{align*}


\section{Proof of Theorem~\ref{thm:jacobian}}
We are interested in studying expectations of objects that may not be integrable functions with respect to the Gaussian measure, namely \emph{tempered distributions}. In order to do so, we denote the set of Schwartz functions on the plane by $D(\mathbb{R}^2)$ and its dual space, the space of tempered distributions, by $D'(\mathbb{R}^2)$, and observe that $\phi\in D(\mathbb{R}^2)$. For $q\in D(\mathbb{R}^2)$ and $T\in D'(\mathbb{R}^2)$, write $\langle T\,,\,q \rangle$ for the dual pairing of $T$ and $q$. For any $q\in D(\mathbb{R}^2)$ we can then define an operator $q^\star$ on $D'(\mathbb{R}^2)$ via the natural injection into the dual space of $D'(\mathbb{R}^2)$ by setting $q^\star(T) = \langle T,q\rangle$ for any $T\in D'(\mathbb{R}^2)$. We can then define the expectation of the distribution by defining $\mathbb{E}[T] = \langle T,\phi\rangle$, which agrees with the usual definition whenever $T$ can be represented by an integrable function. Following~\citet{JonesD.S.DouglasSamuel1982Ttog}, we define the derivative of a tempered distribution via $\langle \frac{\partial}{\partial x}T\,,\,q \rangle = -\langle T\,,\,\frac{\partial}{\partial x}q \rangle$, and when $T$ can be represented by a locally integrable function $f$, we abuse notation to write $\int_{\mathbb{R}^2}\frac{\partial f}{\partial x}\phi:= -\int_{\mathbb{R}^2}f\frac{\partial \phi}{\partial x}.$ This is analogously extended for higher order derivatives by ensuring integration by parts holds whenever $f$ is smooth.

\label{app:jac_proof}
\begin{proof}
We begin by evaluating the derivative of $g_3$ with respect to $\rho$. Let $\bm{\omega}$ denote the Lebesgue measure. Note that the mapping $T_{\psi}:D(\mathbb{R}^2)\to\mathbb{R}$ satisfying
\begin{align*}
&T_\psi(q) = \int \psi(s_1 \omega_1) \psi \big( s_2 (\omega_1 \rho + \omega_2 \sqrt{1-\rho^2}) \big) q(\bm{\omega}) d\bm{\omega} 
\end{align*}
is a tempered distribution since for all $q\in D(\mathbb{R}^2)$, 
\begin{align*}
&\phantom{{}\leq{}}\int  \Big| \psi(s_1 \omega_1) \psi \big( s_2(\omega_1 \rho + \omega_2 \sqrt{1-\rho^2}) q(\bm{\omega}) \big) \Big|   d\bm{\omega} \leq \int  p(\bm{\omega}) \Big| q(\bm{\omega}) \Big| d\bm{\omega} < \infty,
\end{align*}
where $p$ is some polynomial. Let $\phi_\rho$ denote the PDF of a standard bivariate Gaussian with correlation $\rho$. Differentiating $g_3$ with respect to $\rho$ we find
\begin{align*}
\frac{\partial g_3}{\partial \rho} &= \sigma_w^2 \frac{\partial}{\partial \rho} \langle T_\psi, \, \phi_\rho \rangle /\sqrt{g_1g_2}\\
&=\sigma_w^2 \mathbb{E}\big[ \psi(s_1 G_1) \psi'\big( s_2 (G_1 \rho + G_2 \sqrt{1 - \rho ^2} )  \big) (G_1 - G_2 \cot\theta) s_2 \big]/\sqrt{g_1g_2}.
\end{align*}
Let $(Z_1, Z_2, Z_3)$ be multivariate Gaussian, each element having zero mean, unit variance and correlation structure $\mathbb{E}[Z_1 Z_2] = \rho$,  $\mathbb{E}[Z_1 Z_3] = 0$, and $\mathbb{E}[Z_3 Z_2] = \sin\theta$. Then
\begin{align*}
&\frac{\partial g_3}{\partial \rho} = \frac{\sigma_w^2 s_2}{\sqrt{g_1g_2}} \mathbb{E}\Big[\psi( s_1 Z_1)  \psi'( s_2 Z_2)(Z_1 - Z_3 \cot\theta) \Big].
\end{align*}
Two applications of a multivariate version of Stein's lemma for tempered distributions (Lemma~\ref{lem:general_stein} in Appendix~\ref{app:stein}) yield
\begin{align*}
\frac{\partial g_3}{\partial \rho} &= \frac{ \sigma_w^2s_2}{\sqrt{g_1 g_2}} \Big[ s_1 \mathbb{E} \big[ \psi'( s_1 Z_1)  \psi'(s_2 Z_2)  \big] +  \rho s_2 \mathbb{E} \big[ \psi(s_1 Z_1)  \psi''(s_2 Z_2)  \big]  - \rho s_2 \mathbb{E} \big[ \psi(s_1 Z_1)  \psi''(s_2 Z_2)  \big] \Big] \\
&=\frac{ \sigma_w^2s_1 s_2}{\sqrt{g_1 g_2}} \mathbb{E} \big[ \psi'( s_1 Z_1)  \psi'(s_2 Z_2)  \big].
\end{align*}
Note $g_3$ is infinitely differentiable in $\rho$ on $(-1, 1)$ (see Appendix~\ref{app:inf_dif}). The Jacobian is triangular, so the eigenvalues of the Jacobian are simply its diagonal elements. The other diagonal entries may be evaluated by analogous calculations to the above. For example,
\begin{align*}
    \lambda_1 &= \sigma_w^2\mathbb{E}\big[ \psi(s_1Z_1) \psi'(s_1Z_1)Z_1]/s_1 \\
    &= \sigma_w^2\mathbb{E}\big[ \psi'(s_1Z_1) \psi'(s_1Z_1)] + \sigma_w^2\mathbb{E}\big[ \psi(s_1Z_1) \psi''(s_1Z_1)] \\
    &= \frac{\sigma_w^2}{2} \Big( \mathbb{E}\big[ 2\psi'(s_1Z_1) \psi'(s_1Z_1)] + \mathbb{E}\big[ \psi(s_1Z_1) \psi''(s_1Z_1)] + \mathbb{E}\big[ \psi''(s_1Z_1) \psi(s_1Z_1)]\Big) \\
    &= \frac{\sigma_w^2}{2} \mathbb{E}\big[ \big(\psi^2 \big)'' (s_1Z_1) ], \\
    &= \frac{\sigma_w^2}{2} \mathbb{E}\big[ \big(\psi^2 \big)'' (U) ], \quad U \sim \mathcal{N}(0, s_1^2) \\
    &= \frac{\sigma_w^2}{2} \int_{-\infty}^\infty \big(\psi^2 \big)'' (u) \frac{1}{\sqrt{2\pi}s_1} e^{-\frac{u^2}{2s_1^2 }} \, du
\end{align*}
where $\big(\psi^2 \big)''$ denotes the second derivative of the square of $\psi$. Now note that since expectation of $\psi$ can be seen as the application of a distribution (generalised function) $\psi$ to the Guassian PDF, which is a Schwartz function, the following holds:
\begin{align*}
    \frac{\sigma_w^2}{2} \int_{-\infty}^\infty \big(\psi^2 \big)'' (u) \frac{1}{\sqrt{2\pi}s_1} e^{-\frac{u^2}{2s_1^2 }} \, du &= \frac{\sigma_w^2}{2} \int_{-\infty}^\infty \psi^2 (u) \frac{d^2}{du^2} \Big( \frac{1}{\sqrt{2\pi}s_1} e^{-\frac{u^2}{2s_1^2 }} \Big) \, du \\
    &= \frac{\sigma_w^2}{2} \mathbb{E}\Big[ \frac{U^2-s_1^2}{s_1^4} \psi^2(U) \Big] \\
    &= \frac{\sigma_w^2}{2s_1^2} \mathbb{E}\Big[ \big(Z_1^2-1\big) \psi^2(s_1 Z_1 ) \Big].
\end{align*}

An quicker method (but one quite distinct from the method used to find $\lambda_3$) for obtaining $\lambda_1$ and $\lambda_2$ is by expressing the kernel as an expectation of correlated Gaussians with correlation $\rho$, differentiating under the integral, twice applying integration by parts and noting that the Gaussian PDF decays faster than any polynomial.
\end{proof}

\section{Stein's lemma for tempered distributions}
\label{app:stein}
\begin{lemma}[Stein's lemma, tempered distribution]
Let $g$ be a tempered distribution and $X\sim \mathcal{N}(0,1)$. Then $\mathbb{E}|g(x)| < \infty$ and $\mathbb{E}|g'(x)| < \infty$, where $g'$ is the distributional derivative of $g$. Furthermore,
$$ \mathbb{E} [X g(X)] = \mathbb{E} [g'(X)].$$
\end{lemma}
\begin{proof}
This follows from the definition of the derivative for tempered distributions, and the fact that the Gaussian PDF is an element of the Schwartz space.
\end{proof}

\begin{lemma}[Multivariate Stein's lemma, tempered distribution]
Let $h$ be a tempered distribution and $\mathbf{X}$ be Gaussian with mean $\mathbf{0}$. Then $\mathbb{E}|h(\mathbf{X})| < \infty$ and $\mathbb{E}| \partial/\partial X_1 h(\mathbf{X})| < \infty$, where $\partial/\partial X_1 h(\mathbf{X})$ is the distributional derivative of $h$ with respect to the first coordinate. Furthermore,
$$ \mathbb{E} [X_1 h(\mathbf{X})] = \sum_{i=1}^n \mathbb{E} [X_1 X_i] \mathbb{E} \Big[ \frac{\partial}{\partial X_i} h(\mathbf{X}) \Big].$$
\label{lem:general_stein}
\end{lemma}
\begin{proof}
First, we prove the statement when $X_1, ..., X_n$ are independent with standard deviations $1$. Stein's lemma says
$$ \mathbb{E}[X_1 g( X_1)] =  \mathbb{E}[g'( X_1)],$$
for all tempered distributions $g$. Now for all $i$ and any tempered distribution $f$,
\begin{align*} 
\mathbb{E}[X_i f(X_1, ..., X_n)]  &= \mathbb{E} \Big[ \mathbb{E}\big[ X_i f(X_1, ..., X_n) \mid \{X_j \}_{j=1, j\neq i}^n  \big] \Big]  \\
& = \mathbb{E} \Big[ (\partial/\partial X_i) f(X_1, ..., X_n)  \Big],
\end{align*}
or in vector notation,
$$ \text{Cov}[\mathbf{X}, f(\mathbf{X}) ] = \mathbb{E}[\nabla f(\mathbf{X}) ].$$

We apply an affine transformation to $\mathbf{X}$, $\mathbf{Z} = \Sigma^{(1/2)}\mathbf{X} + \boldsymbol{\mu}$. We absorb the affine transform into $f$, $f(\mathbf{X})=h(\Sigma^{(1/2)} \mathbf{X} + \boldsymbol{\mu})$ for some $h$. We have
 \begin{align*} 
\text{Cov} [\mathbf{Z}, h(\mathbf{Z}) ] &= \text{Cov} [\Sigma^{(1/2)}\mathbf{X}+ \boldsymbol{\mu}, f(\mathbf{X})]\\
&=\text{Cov} [\Sigma^{(1/2)}\mathbf{X}, f(\mathbf{X})] \\
&=\Sigma^{(1/2)}\text{Cov} [\mathbf{X}, f(\mathbf{X})] \\
&=\Sigma^{(1/2)}\mathbb{E}[\nabla f(\mathbf{X})] \\
&=\Sigma\,\mathbb{E}[\nabla h(\mathbf{Z})].
\end{align*} 

We may extract the first entry of the vector, yelding
$$\text{Cov} [Z_1 h(\mathbf{Z})] = \sum_{i=1}^n \mathbb{E}[Z_i Z_1] \mathbb{E} [ (\partial/\partial Z_i) h(\mathbf{Z})].$$
\end{proof}

\section{Kernel is infinitely differentiable}
\label{app:inf_dif}
\begin{lemma}
In the same setting as Theorem~\ref{thm:jacobian}, the kernel is infinitely differentiable in $\rho$, $s_1$ and $s_2$.
\end{lemma}
\begin{proof}
Let $\hat{\phi}_{\rho, s_1, s_2}$ denote the PDF of the bivariate Gaussian having variances $s_1^2$ and $s_2^2$ and correlation $\rho$. Define
$$\kappa(\rho) = \int_{\mathbb{R}^2} \hat{\phi}_{\rho, s_1, s_2}(\mathbf{z})\psi(z_1)\psi(z_2)\,dz.$$
The mean value theorem says that for any $a,b \in (-1, 1)$ and $a \leq \rho_1,\rho_2 \leq b$ we have
\begin{align*}
&\phantom{{}={}}  \kappa(\rho_1) - \kappa(\rho_2)  \\
    &= \int_{\mathbb{R}^2} \big( \hat{\phi}_{\rho_1, s_1, s_2}(\mathbf{z}) - \hat{\phi}_{\rho_2, s_1, s_2}(\mathbf{z}) \big) \psi(z_1) \psi(z_2) \,dz\\
    &= ( \rho_1 - \rho_2 ) \int_{\mathbb{R}^2}  \frac{\partial\hat{\phi}_{\rho, s_1, s_2} (\mathbf{z}) }{\partial \rho} \Big|_{\rho = \rho_3} \psi(z_1)\psi(z_2)  \,dz \\
\end{align*}
for some $\rho_3\in(\rho_1,\rho_2)$. So $|\kappa(\rho_1) - \kappa(\rho_2)|\leq M_{a,b}|\rho_1 - \rho_2|$ where $M_{a,b} =  \sup\limits_{\rho\in(a, b)}  \Big| \int_{\mathbb{R}^2}  \frac{\partial\hat{\phi}_{\rho, s_1, s_2} (\mathbf{z}) }{\partial \rho} \Big|_{\rho = \rho_3} \psi(z_1)\psi(z_2)  \,dz \Big|$.
Note that $M_{a,b}$ is finite because each element in the supremum is the integral of a Schwartz function. The same argument applies to derivatives of any order of, since each derivative is also an element of the Schwarts space.

The same argument applies to $s_1$ and $s_2$. 
\end{proof}

\section{Derivative at endpoints}
\label{app:endpoints}
\begin{lemma}
In the same setting as Theorem~\ref{thm:jacobian}, with the additional assumptions that $\psi$ is continuous almost everywhere, the expression for $\lambda_3$ extend to $\rho \in [-1, 1]$ provided \emph{the expression for} $\lambda_3$ is finite over $[-1, 1]$.
\end{lemma}
\begin{proof}
Under the same definition for $\kappa$ as in Appendix~\ref{app:inf_dif}
$$\kappa(\rho) = \int_{\mathbb{R}^2} \hat{\phi}_{\rho, s_1, s_2}(\mathbf{z})\psi(z_1)\psi(z_2)\,dz = \int_{\mathbb{R}^2} \phi(\mathbf{z}) \psi(s_1 z_1) \psi\big( s_2(z_1\rho + z_2\sqrt{1-\rho^2}) \big)\,d\mathbf{z},$$
where $\phi$ is the PDF of uncorrelated standard bivariate Gaussian random variables. Note that $\kappa$ is continuous in $\rho$ since
\begin{align*}
&\phantom{{}={}} \lim\limits_{\rho \to c} \int_{\mathbb{R}^2} \phi(\mathbf{z}) \psi(s_1 z_1) \psi\big( s_2(z_1\rho + z_2\sqrt{1-\rho^2}) \big)\,d\mathbf{z} \\
&= \int_{\mathbb{R}^2} \phi(\mathbf{z}) \psi(s_1 z_1) \lim\limits_{\rho \to c} \psi\big( s_2(z_1\rho + z_2\sqrt{1-\rho^2}) \big)\,d\mathbf{z},
\end{align*}
with the interchange of the limit being justified by dominated convergence, since
\begin{align*}
&\phantom{{}={}}\phi(\mathbf{z}) \psi(s_1 z_1) \psi\big( s_2(z_1\rho + z_2\sqrt{1-\rho^2}) \big) \\
&\leq \phi(\mathbf{z}) \text{max}\big\{ \psi^2(s_1 z_1), \psi^2\big(s_2(z_1\rho + z_2\sqrt{1-\rho^2}) \big\}  \\
&\leq \phi(\mathbf{z}) \Big(\psi^2(s_1 z_1)+ \psi^2\big(s_2(z_1\rho + z_2\sqrt{1-\rho^2}) \Big),
\end{align*}
which has integral $\mathbb{E} [\psi^2(s_1Z_1)] + \mathbb{E} [\psi^2(s_2Z_1)]$.

Then supposing WLOG $\rho_0 < \rho_1$ for $\rho_0, \rho_1 \in [-1, 1]$,
\begin{align*}
\big| \kappa(\rho_1) - \kappa(\rho_0) \big| &= \lim\limits_{a \to \rho_0^+}  \lim\limits_{b \to \rho_1^-} \big| \kappa(b) - \kappa(a) \big| \\
&\leq \lim\limits_{a \to \rho_0^+}  \lim\limits_{b \to \rho_1^-} M_{-1,1}|b - a| \\
&= M_{-1,1}|\rho_1 - \rho_0|,
\end{align*}
where $\kappa$ and $M_{-1,1}$ are as in Appendix~\ref{app:inf_dif}. Note $M_{-1,1}$ is finite since
\begin{align*}
M_{-1,1} &= \sup\limits_{\rho\in(-1, 1)} \Big| \frac{\partial}{\partial \rho} \int_{\mathbb{R}^2}\hat{\phi}_{1, 1, 1}(\mathbf{z})\psi(s_1z_1) \psi\big(s_2(z_1\rho + z_2\sqrt{1-\rho^2} z_2)\big)\,dz\Big| \\
&= \sup\limits_{\rho\in(-1, 1)} \Big| \int_{\mathbb{R}^2} \phi(\mathbf{z}) \psi(s_1 z_1) \psi'\big( s_2(z_1\rho + z_2\sqrt{1-\rho^2}) \big)s_2 \Big( z_1 + \frac{z_2\rho}{(1-\rho^2)^{1/2}} \Big)\,d\mathbf{z} \Big| \\
&= \sup\limits_{\rho\in(-1, 1)}\Big| s_1 s_2 \mathbb{E}[ \psi'(s_1 Z_1) \psi'(s_2 Z_2)]\Big| < \infty,
\end{align*}
where the last equality is due to two applications of Stein's lemma, as in the proof of Theorem~\ref{thm:jacobian}.

$\kappa$ therefore has finite derivative on $[-1, 1]$.
We have
\begin{align*}
\frac{d\kappa}{d\rho} &= \int_{\mathbb{R}^2} \phi(\mathbf{z}) \psi(s_1 z_1) \psi'\big( s_2(z_1\rho + z_2\sqrt{1-\rho^2}) \big)s_2 \Big( z_1 + \frac{z_2\rho}{(1-\rho^2)^{1/2}} \Big)\,d\mathbf{z}.
\end{align*}
Two applications of Stein's lemma as in the proof of Theorem~\ref{thm:jacobian} gives the desired result.
\end{proof}

\section{Example eigenvalues}
\label{app:eigen}
\subsection{GELU}
\label{app:gelu_eigen}
We would like to evaluate
\begin{align*}
\dot{k}(\mathbf{x}, \mathbf{x}') &= \sigma^2 \mathbb{E} \big[ \big(\Phi(\sigma \Vert \mathbf{x} \Vert Z_1) + (\sigma \Vert \mathbf{x} \Vert Z_1)\phi(\sigma \Vert \mathbf{x} \Vert Z_1) \big) \big(\Phi(\sigma \Vert \mathbf{x}' \Vert Z_2) + (\sigma \Vert \mathbf{x}' \Vert Z_2)\phi(\sigma \Vert \mathbf{x}' \Vert Z_2) \big) \big] \\
&=\sigma^2 \Bigg( \mathbb{E} \big[ \Phi(\sigma \Vert \mathbf{x} \Vert Z_1)\Phi(\sigma \Vert \mathbf{x}' \Vert Z_2) \big] + \mathbb{E} \big[ (\Phi(\sigma \Vert \mathbf{x} \Vert Z_1) (\sigma \Vert \mathbf{x}' \Vert Z_2)\phi(\sigma \Vert \mathbf{x}' \Vert Z_2) \big]  + \\
& \phantom{{}=\Bigg( \sigma^2 } \mathbb{E} \big[ (\Phi(\sigma \Vert \mathbf{x}' \Vert Z_1) (\sigma \Vert \mathbf{x} \Vert Z_2)\phi(\sigma \Vert \mathbf{x} \Vert Z_2) \big] + \mathbb{E}\big[ (\sigma \Vert \mathbf{x} \Vert Z_1)\phi(\sigma \Vert \mathbf{x} \Vert Z_1) (\sigma \Vert \mathbf{x}' \Vert Z_2)\phi(\sigma \Vert \mathbf{x}' \Vert Z_2) \big] \Bigg)
\end{align*}
The first term may be related to the result of~\citet{williams1997computing} since
\begin{align*}
&\phantom{{}={}}\mathbb{E}\big[ \Phi(\sigma \Vert \mathbf{x} \Vert Z_1) \Phi(\sigma \Vert \mathbf{x}' \Vert Z_2) \big] \\
&= \frac{1}{4}\mathbb{E} \Big[ \big(1 + \erf (\sigma \Vert \mathbf{x} \Vert Z_1/\sqrt{2}) \big)\big(1 + \erf (\sigma \Vert \mathbf{x}' \Vert Z_2/\sqrt{2}) \big)\Big] \\
&= \frac{1}{4} \Bigg( 1 + \mathbb{E}  \Big[  \erf(\sigma \Vert \mathbf{x} \Vert Z_1 /\sqrt{2}) \erf(\sigma \Vert \mathbf{x}' \Vert Z_2/\sqrt{2})     \Big] \Bigg) \\
&= \frac{1}{4} \Bigg( 1 + \frac{2}{\pi}\sin^{-1}\frac{\sigma^2 \Vert \mathbf{x} \Vert \Vert \mathbf{x}' \Vert \cos\theta}{\sqrt{1 + \sigma^2 \Vert \mathbf{x} \Vert^2}\sqrt{1 + \sigma^2 \Vert \mathbf{x}' \Vert^2}} \Bigg).
\end{align*}
The middle cross-terms are equal after permutations of $\mathbf{x}$ and $\mathbf{x}'$ by exchangeability of $Z_1$ and $Z_2$, and are given by
\begin{align*}
h(\beta) = \mathbb{E} \big[\Phi(\beta \sigma \Vert \mathbf{x} \Vert Z_1) (\sigma \Vert \mathbf{x} \Vert Z_2)\phi(\sigma \Vert \mathbf{x} \Vert Z_2) \big]
\end{align*}
evaluated at $\beta=1$. Differentiating under the integral, we obtain the initial value problem
\begin{align*}
\frac{dh}{d\beta} &= \mathbb{E}\big[ \sigma \Vert \mathbf{x} \Vert Z_1 \phi(\beta \sigma \Vert \mathbf{x} \Vert Z_1) \sigma \Vert \mathbf{x}' \Vert Z_2 \phi(\sigma \Vert \mathbf{x}'\Vert Z_2)  \big]\\
&= \frac{\sigma^2 \Vert \mathbf{x} \Vert \Vert \mathbf{x}' \Vert}{2\pi} \int Z_1 Z_2 \frac{1}{2\pi\sin\theta} \exp \Bigg(-\frac{1}{2} \mathbf{z}^\top \frac{1}{\sin^2\theta} \begin{bmatrix}
1 & -\cos\theta \\
-\cos\theta & 1
\end{bmatrix} \mathbf{z} \Bigg) \\
& \phantom{{}=\frac{\sigma^2 \Vert \mathbf{x} \Vert \Vert \mathbf{x}' \Vert}{2\pi} \int Z_1 Z_2 \frac{1}{2\pi\sin\theta} } \exp \Bigg(-\frac{1}{2} \mathbf{z}^\top \begin{bmatrix}
\beta^2 \sigma^2 \Vert \mathbf{x} \Vert^2  &0 \\
0 & \sigma^2 \Vert \mathbf{x}' \Vert^2
\end{bmatrix} \mathbf{z} \Bigg) \,d\mathbf{z} \\
&= \frac{\sigma^2 \Vert \mathbf{x} \Vert \Vert \mathbf{x}' \Vert}{2\pi} \int Z_1 Z_2 \frac{1}{2\pi\sin\theta} \exp \Bigg(-\frac{1}{2} \mathbf{z}^\top S^{-1} \mathbf{z} \Bigg)  \,d\mathbf{z},
\end{align*}
where
\begin{align*}
S^{-1} &= \frac{1}{\sin^2\theta} \begin{bmatrix}
1+\sin^2\theta \beta^2 \sigma^2 \Vert \mathbf{x} \Vert^2 & -\cos\theta \\
-\cos\theta & 1 + \sin^2\theta \sigma^2 \Vert \mathbf{x}' \Vert^2
\end{bmatrix}. \\
\implies S &= \frac{1}{1+\beta^2\sigma^2 \Vert \mathbf{x} \Vert^2 + \sigma^2 \Vert \mathbf{x}' \Vert^2 + \sin^2\theta \sigma^4 \Vert \mathbf{x} \Vert^2 \Vert \mathbf{x}' \Vert^2 } \begin{bmatrix}
1+\sin^2\theta \beta^2 \sigma^2 \Vert \mathbf{x} \Vert^2 & -\cos\theta \\
-\cos\theta & 1 + \sin^2\theta \sigma^2 \Vert \mathbf{x}' \Vert^2
\end{bmatrix} \\
\det S &= \frac{\sin^2\theta}{1+\beta^2 \sigma^2 \Vert \mathbf{x} \Vert^2 + \sigma^2 \Vert \mathbf{x}' \Vert^2 + \sin^2\theta \sigma^4 \Vert \mathbf{x} \Vert^2 \Vert \mathbf{x}' \Vert^2 \beta^2}.
\end{align*}
We then have that
\begin{align*}
\frac{dh}{d\beta} &= \frac{\sigma^2 \Vert \mathbf{x} \Vert \Vert \mathbf{x}' \Vert \sqrt{\det S}}{2\pi\sin\theta} \int Z_1 Z_2 \frac{1}{2\pi\sqrt{\det S}} \exp \Bigg(-\frac{1}{2} \mathbf{z}^\top S^{-1} \mathbf{z} \Bigg)  \,d\mathbf{z} \\
&= \frac{\sigma^2 \Vert \mathbf{x} \Vert \Vert \mathbf{x}' \Vert  \sqrt{\det S}}{2\pi\sin\theta} \cos\theta \sqrt{S_{11} S_{22}} \\
&=\frac{\sigma^2 \Vert \mathbf{x} \Vert \Vert \mathbf{x}' \Vert \cos\theta}{2\pi} \big( 1+\beta^2\sigma^2 \Vert \mathbf{x} \Vert^2 + \sigma^2 \Vert \mathbf{x}' \Vert^2 + \sin^2\theta \sigma^4 \Vert \mathbf{x} \Vert^2 \Vert \mathbf{x}' \Vert^2  \big)^{-3/2}  \\
&\phantom{{}={}} \sqrt{1+\sin^2\theta \beta^2 \sigma^2 \Vert \mathbf{x} \Vert^2} \sqrt{1+\sin^2\theta \sigma^2 \Vert \mathbf{x}' \Vert^2}, \\
&\geq \frac{\sigma^2 \Vert \mathbf{x} \Vert \Vert \mathbf{x}' \Vert \cos\theta}{2\pi} \big( 1+\beta^2\sigma^2 \Vert \mathbf{x} \Vert^2 + \sigma^2 \Vert \mathbf{x}' \Vert^2 + \sin^2\theta \sigma^4 \Vert \mathbf{x} \Vert^2 \Vert \mathbf{x}' \Vert^2  \big)^{-3/2}  \\
\quad h(0) &= 0,
\end{align*}
with solution
\begin{align*}
h(\beta) \geq \frac{\cos\theta \sigma^2 \Vert \mathbf{x} \Vert \Vert \mathbf{x}' \Vert \beta}{2\pi(\sigma^2 \Vert \mathbf{x}' \Vert^2+1)\sqrt{\beta^2 \sigma^4 \Vert \mathbf{x} \Vert^2 \Vert \mathbf{x}' \Vert^2 \sin^2\theta  + \beta^2\sigma^2 \Vert \mathbf{x} \Vert^2 + \sigma^2 \Vert \mathbf{x}' \Vert^2+ 1}}. \numberthis \label{eq:h_int}
\end{align*}

The last term,
\begin{align*}
&\phantom{{}={}}\sigma^2 \Vert \mathbf{x} \Vert  \Vert \mathbf{x}' \Vert \mathbb{E} \big[ Z_1 Z_2 \phi(\sigma \Vert \mathbf{x} \Vert Z_1) \phi(\sigma \Vert \mathbf{x}' \Vert Z_2) \big],
\end{align*}
is simply $\frac{dh}{d\beta}$ evaluated at $\beta=1$, and satisfies
\begin{align*}
\frac{dh}{d\beta} \Big|_{\beta = 1}&\geq \frac{\sigma^2 \Vert \mathbf{x} \Vert  \Vert \mathbf{x}' \Vert}{2\pi} \frac{\cos\theta}{\big(1+\sigma^2 \Vert \mathbf{x} \Vert^2 + \sigma^2 \Vert \mathbf{x}' \Vert^2 + \sigma^4 \Vert \mathbf{x} \Vert^2  \Vert \mathbf{x}' \Vert^2 \sin^2\theta\big)^{3/2}}.
\end{align*}

\subsection{ELU}
\label{app:elu_eigen}
The generalised derivative of the ELU is
$$\phi'(z) = \Theta(z) + \delta(z)z + \Theta(-z)e^z + \delta(-z)(1-e^z).$$
The second and last terms may be treated as zero since they vanish under integration. We would like to evaluate
\begin{align*}
\dot{k}(\mathbf{x}, \mathbf{x}') &= \sigma^2 \mathbb{E} \big[ \big( \Theta(\sigma \Vert \mathbf{x} \Vert Z_1) + \Theta(-\sigma \Vert \mathbf{x} \Vert Z_1)e^{\sigma \Vert \mathbf{x} \Vert Z_1} \big) \big( \Theta(\sigma \Vert \mathbf{x}' \Vert Z_2) + \Theta(-\sigma \Vert \mathbf{x}' \Vert Z_2)e^{\sigma \Vert \mathbf{x}' \Vert Z_2} \big)\big] \\
&= \sigma^2 \Bigg( \mathbb{E} \big[ \Theta(\sigma \Vert \mathbf{x} \Vert Z_1)  \Theta(\sigma \Vert \mathbf{x}' \Vert Z_2) \big] +  \mathbb{E} \big[  \Theta(\sigma \Vert \mathbf{x} \Vert Z_1)  \Theta(-\sigma \Vert \mathbf{x}' \Vert Z_2)e^{\sigma \Vert \mathbf{x}' \Vert Z_2} \big] + \\
&\phantom{{}=\sigma^2 \Bigg(} \mathbb{E} \big[ \Theta(-\sigma \Vert \mathbf{x} \Vert Z_1)e^{\sigma \Vert \mathbf{x} \Vert Z_1} \Theta(\sigma \Vert \mathbf{x}' \Vert Z_2) + \big] + \mathbb{E} \big[ \Theta(-\sigma \Vert \mathbf{x} \Vert Z_1)e^{\sigma \Vert \mathbf{x} \Vert Z_1} \Theta(-\sigma \Vert \mathbf{x}' \Vert Z_2)e^{\sigma \Vert \mathbf{x}' \Vert Z_2} \big]  \Bigg)
\end{align*}
The first term is given by Sheppard's identity~\citep{sheppard1899iii}, or as an arc-cosine kernel~\citep{NIPS2009_3628}.

The cross-terms are equal after permuting $\mathbf{x}$ and $\mathbf{x}'$ by exchangeability of $Z_1$ and $Z_2$, and can be evaluated by completing the square of the exponential terms.
\begin{align*}
&\phantom{{}={}} \mathbb{E} \big[  \Theta(\sigma \Vert \mathbf{x} \Vert Z_1)\Theta(-\sigma \Vert \mathbf{x}' \Vert Z_2)e^{\sigma \Vert \mathbf{x}' \Vert Z_2} \big] \\
&= \frac{1}{2\pi\sin\theta} \int \Theta(z_1) \Theta(-z_2) \exp\Bigg( -\frac{1}{2} \bigg( \mathbf{z}^\top S^{-1} \mathbf{z} - 2 \sigma \Vert \mathbf{x}' \Vert z_2 \bigg) \Bigg) \, d\mathbf{z}, \qquad S = \begin{bmatrix}
1 & \cos\theta \\
\cos\theta & 1
\end{bmatrix} \\
&=  \frac{1}{2\pi\sin\theta} \int \Theta(z_1) \Theta(-z_2) \exp\Bigg( -\frac{1}{2} \bigg( \mathbf{z}-S \begin{pmatrix}
0 \\
\sigma \Vert \mathbf{x}' \Vert
\end{pmatrix} \bigg)^\top S^{-1}  \bigg( \mathbf{z}-S \begin{pmatrix}
0 \\
\sigma \Vert \mathbf{x}' \Vert
\end{pmatrix} \bigg)  \Bigg) \, d\mathbf{z} \exp\Big( \frac{1}{2} \sigma^2 \Vert \mathbf{x}' \Vert^2 \Big) \\
&= \mathbb{E} \big[ \Theta(\hat{Z}_1) \Theta(\hat{Z}_2) \big] e^{\frac{1}{2} \sigma^2 \Vert \mathbf{x}' \Vert^2 }, \qquad (\hat{Z}_1, \hat{Z}_2) \sim \mathcal{N}\Big( \begin{pmatrix} \cos\theta \sigma \Vert \mathbf{x}' \Vert \\ \sigma \Vert \mathbf{x}' \Vert \end{pmatrix},  S\Big) \\
&= \mathbb{E} \big[ \Theta(Z_1 + \sigma \Vert \mathbf{x}' \Vert \cos\theta) \Theta(-Z_2 - \sigma \Vert \mathbf{x}' \Vert )\big] e^{\frac{1}{2} \sigma^2 \Vert \mathbf{x}' \Vert^2 } \\
&= e^{\frac{1}{2} \sigma^2 \Vert \mathbf{x}' \Vert^2 } \Phi(\sigma \Vert \mathbf{x}' \Vert \cos\theta, -\sigma \Vert \mathbf{x}' \Vert; -\cos\theta).
\end{align*}

The last term is evaluated similarly,
\begin{align*}
&\phantom{{}={}} \mathbb{E} \big[  \Theta(-\sigma \Vert \mathbf{x} \Vert Z_1)\Theta(-\sigma \Vert \mathbf{x}' \Vert Z_2)e^{\sigma  (\Vert \mathbf{x} \Vert Z_1 + \Vert \mathbf{x}' \Vert Z_2)} \big] \\
&=\frac{1}{2\pi \sin\theta} \int \Theta(-z_1) \Theta(-z_2) \exp\Bigg( -\frac{1}{2} \Big( \mathbf{z}^\top S^{-1} \mathbf{z}  - 2\begin{pmatrix}
\sigma \Vert \mathbf{x} \Vert \\ \sigma \Vert \mathbf{x}' \Vert
\end{pmatrix} \mathbf{z} 
\Big) \Bigg) \,d\mathbf{z}  \\
&=\frac{1}{2\pi \sin\theta} \int \Theta(-z_1) \Theta(-z_2) \exp\Bigg( -\frac{1}{2} \Bigg( \mathbf{z} - S \begin{pmatrix}
\sigma \Vert \mathbf{x} \Vert \\
\sigma \Vert \mathbf{x}' \Vert
\end{pmatrix} \Bigg)^\top S^{-1} \Bigg( \mathbf{z} - S \begin{pmatrix}
\sigma \Vert \mathbf{x} \Vert \\
\sigma \Vert \mathbf{x}' \Vert
\end{pmatrix} \Bigg) \Bigg) \,d\mathbf{z}  \\
&= \mathbb{E} \big[ \Theta(-\hat{Z}_1) \Theta(-\hat{Z}_2) \big] e^{\frac{1}{2} \begin{pmatrix}
\sigma \Vert \mathbf{x} \Vert \\ \sigma \Vert \mathbf{x}' \Vert
\end{pmatrix}^\top S \begin{pmatrix}
\sigma \Vert \mathbf{x} \Vert \\ \sigma \Vert \mathbf{x}' \Vert
\end{pmatrix} }, \qquad (\hat{Z}_1, \hat{Z}_2) \sim \mathcal{N}\Big( S\begin{pmatrix} \sigma \Vert \mathbf{x} \Vert \\ \sigma \Vert \mathbf{x}' \Vert \end{pmatrix},  S\Big) \text{ and if } \Vert \mathbf{x} \Vert = \Vert \mathbf{x}' \Vert, \\
&= e^{\sigma^2 \Vert \mathbf{x} \Vert^2(1+\cos\theta) } \mathbb{E} \big[ \Theta \big(-Z_1 - \sigma \Vert \mathbf{x} \Vert (1+\cos\theta) \big)  \Theta\big(-Z_2 - \sigma \Vert \mathbf{x} \Vert(1+\cos\theta) \big)\big]  \\
&= e^{\sigma^2 \Vert \mathbf{x} \Vert^2(1+\cos\theta) } \Phi \Big( -\sigma \Vert \mathbf{x} \Vert (1+\cos\theta), -\sigma \Vert \mathbf{x} \Vert(1+\cos\theta); \cos\theta \Big).
\end{align*}

\section{Degenerate priors and posteriors}
\label{app:degenerate}

\propDegeneratePrior*
\begin{proof}
The variance of the Gaussian random variable $f^{(L)}(\mathbf{x}_1) - f^{(L)}(\mathbf{x}_2)$ is
\begin{align*}
    (\sigma^{(L)})^2 := \text{Var}\big( f^{(L)}(\mathbf{x}_1) - f^{(L)}(\mathbf{x}_2) \big) &= \text{Var} f^{(L)}(\mathbf{x}_1) + \text{Var} f^{(L)}(\mathbf{x}_2) - 2\text{Cov} \big( f^{(L)}(\mathbf{x}_1), f^{(L)}(\mathbf{x}_2) \big) \\
    &= k^{(L)}(\mathbf{x}_1, \mathbf{x}_1) + k^{(L)}(\mathbf{x}_2, \mathbf{x}_2) - 2 k^{(L)}(\mathbf{x}_1, \mathbf{x}_2).
\end{align*}
Taking the limit as $L \to \infty$, we obtain that $\lim\limits_{L \to \infty} \text{Var}\big( f^{(L)}(\mathbf{x}_1) - f^{(L)}(\mathbf{x}_2) \big) =0$. The mean is $0$ by assumption. The characteristic function of the sequence of random variables $\{f^{(L)}(\mathbf{x}_1) - f^{(L)}(\mathbf{x}_2) \}_{L=1}^\infty$ therefore converges to
\begin{align*}
    \lim\limits_{L \to \infty} e^{-\frac{1}{2}(\sigma^{(L)}t)^2} = 1.
\end{align*}
We may therefore apply L\'evy's continuity theorem~\citep[Theorem 26.3]{billingsley2008probability} (which among other things says that pointwise convergence of the characteristic function implies convergence in distribution), concluding that $f^{(L)}(\mathbf{x}_1) - f^{(L)}(\mathbf{x}_2)$ converges in distribution to the random variable that is almost surely $0$.
\end{proof}

\propDegeneratePosterior*
\begin{proof}
Let $\mathbf{X}_*$ be the $2 \times n^{(0)}$ matrix containing $\mathbf{x}_1^\top$ as the first row and $\mathbf{x}_2^\top$ as the second row. The posterior predictive is
\begin{align*}
 \mathbf{f}^{(L)}(\mathbf{X}_*) \mid \mathbf{X}, \mathbf{Y}, \mathbf{X}_* &\sim  \mathcal{N}\big(\overline{\mathbf{f}}^{(L)}_*, \text{cov}(\mathbf{f}^{(L)}_*) \big), \quad \text{ where}
 \end{align*}
 \begin{align*}
    \overline{\mathbf{f}}^{(L)}_* &= k^{(L)}(\mathbf{X}_*, \mathbf{X}) \big( k^{(L)} (\mathbf{X}, \mathbf{X}) + \sigma_n^2 I \big)^{-1} \mathbf{Y}, \quad \text{ and} \\
    \text{cov}(\mathbf{f}_*^{(L)}) &= k^{(L)} k^{(L)}(\mathbf{X}_*, \mathbf{X}_*) - k^{(L)} k^{(L)}(\mathbf{X}_*, \mathbf{X})  \big( k^{(L)} (\mathbf{X}, \mathbf{X}) + \sigma_n^2 I \big)^{-1} k^{(L)}(\mathbf{X}, \mathbf{X}_*) .
\end{align*}
 Denote by $D_{2 \times N}(\mathbf{X})$ the $2 \times N$ matrix with $ij$th entry being equal to $D(\mathbf{r}_j)$ where $\mathbf{r}^\top_j$ is the $j$th row of $\mathbf{X}$. Let $D_{N \times 2}(\mathbf{X}) = D_{2 \times N}(\mathbf{X})^\top$. Note both rows of $D_{2 \times N}(\mathbf{X})$ are the same. Similarly, both columns of $D_{N \times 2}(\mathbf{X})$ are the same. Denote by $f^{(L)}_*(\mathbf{x}_1)$ and $f^{(L)}_*(\mathbf{x}_2)$ the random variables representing the posterior predictive process evaluated at $\mathbf{x}_1$ and $\mathbf{x}_2$ respectively. The mean of the Gaussian random variable $f^{(L)}_*(\mathbf{x}_1) - f^{(L)}_*(\mathbf{x}_2)$ satisfies
\begin{align*}
    \mu^{(L)} :&= \Big( k^{(L)}(\mathbf{x}_1, \mathbf{X}) - k^{(L)}(\mathbf{x}_2, \mathbf{X}) \Big) \big( k^{(L)} (\mathbf{X}, \mathbf{X}) + \sigma_n^2 I \big)^{-1} \mathbf{Y} \\
    \lim\limits_{L \to \infty} \mu^{(L)} &= \Big( D_{1 \times N}(\mathbf{X}) - D_{1 \times N}(\mathbf{X}) \Big) \lim\limits_{L \to \infty}  \big( k^{(L)} (\mathbf{X}, \mathbf{X}) + \sigma_n^2 I \big)^{-1} \mathbf{Y} = 0.
\end{align*}
Here on the second line we used the product rule for limits. Note that $\lim\limits_{L \to \infty} \big( k^{(L)} (\mathbf{X}, \mathbf{X}) + \sigma_n^2 I \big)^{-1}$ exists and every element is finite. This is because $k^{(L)} (\mathbf{X}, \mathbf{X}) + \sigma_n^2 I $ is a positive definite matrix (with non-zero determinant), so that the inverse operation is continuous with respect to $L$.

Letting $C_{2 \times 2}$ denote the $2 \times 2$ matrix where every element is $C$, the covariance matrix of the posterior process satisfies
\begin{align*}
    \lim\limits_{L \to \infty} \text{cov}(\mathbf{f}_*) &= \lim\limits_{L \to \infty} k^{(L)}(\mathbf{X}_*, \mathbf{X}_*) - \lim\limits_{L \to \infty} k^{(L)}(\mathbf{X}_*, \mathbf{X})  \big( k^{(L)} (\mathbf{X}, \mathbf{X}) + \sigma_n^2 I \big)^{-1} k^{(L)}(\mathbf{X}, \mathbf{X}_*) \\
    &= C_{2 \times 2} - D_{2 \times N}(\mathbf{X}) \lim\limits_{L \to \infty} \big( k^{(L)} (\mathbf{X}, \mathbf{X}) + \sigma_n^2 I \big)^{-1} D_{N \times 2}(\mathbf{X}).
\end{align*}
Note that every element of this matrix is the same. Therefore, $\lim\limits_{L \to \infty} \text{Var} \big(f^{(L)}_*(\mathbf{x}_1) - f^{(L)}_*(\mathbf{x}_2)\big) = 0.$ The characteristic function of $f^{(L)}_*(\mathbf{x}_1) - f^{(L)}_*(\mathbf{x}_2)$ therefore converges to
\begin{align*}
    \lim\limits_{L \to \infty} e^{-i\mu^{(L)}t} e^{-\frac{1}{2}\Big(\text{Var} \big(f^{(L)}_*(\mathbf{x}_1) - f^{(L)}_*(\mathbf{x}_2)\big) t \Big)^2} = 1.
\end{align*}
We may therefore apply L\'evy's continuity theorem~\citep[Theorem 26.3]{billingsley2008probability} (which among other things says that pointwise convergence of the characteristic function implies convergence in distribution), concluding that $f^{(L)}(\mathbf{x}_1) - f^{(L)}(\mathbf{x}_2)$ converges in distribution to the random variable that is almost surely $0$.
\end{proof}

\corDegenerateLReLU*
\begin{proof}
We have that
\begin{align*}
    \mathbb{E} \big[ \psi'( Z_1)  \psi'( Z_2)  \big] &= (1-a)^2 \mathbb{P}(Z_1 > 0, Z_2 >0) + \\
    &\phantom{{}={}} a(1-a) \mathbb{P}(Z_1 > 0) + a(1-a) \mathbb{P}(Z_2 > 0) + a^2 \\
    &= \frac{(1-a)^2}{2\pi}(\pi - \theta) + a,
\end{align*}
and
\begin{align*}
\mathbb{E}[\psi^2(Z_1)] &= (1 - a)^2 \frac{1}{2} + 2(1-a)a \mathbb{E}[\Theta(Z_1) Z_1^2]+ a^2 \\
&= (1 - a)^2 \frac{1}{2} + (1-a)a + a^2 \\
&=  \frac{(1 - a)^2}{2} + a,
\end{align*}
implying by Theorem~\ref{thm:jacobian} that $\lambda_3 < 1$. Then by Corollary~\ref{cor:abs_hom_fp}, $g_3(s^2, s^2, \cdot)$ admits a unique fixed point at $1$.

Let $\mathbf{x}_1, \mathbf{x}_2 \in \mathcal{X}_*$ and let $\Vert \mathbf{x}_1 \Vert^2 = s^2$ on the chosen hypersphere. As shown above, $g_3(s^2, s^2, \cdot)$ admits a unique fixed point at $1$ and $s^2$ is a fixed point of $g_1(\cdot, s_2^2, \rho)$ for any $s_2^2$ and $\rho$. Starting from any $\rho = \cos^{-1}\frac{\mathbf{x}_1 \cdot \mathbf{x}_2}{\Vert \mathbf{x}_1 \Vert \Vert \mathbf{x}_2 \Vert} \in [-1, 1]$, we have that $L$ compositions of $\mathbf{g}$ applied to $(s^2, s^2, \rho)$ converges to $(s^2, s^2, 1)$ as $L \to \infty$. 

With a slight abuse of notation, denote $L$ compositions of $g_1(\cdot, s_2^2, \rho)$ (which does not depend on $s_2^2$ or $\rho$) by $g_1^{(L)}$, and likewise for $g_2^{(L)}$. Also denote by $g_3^{(L)}$ the $L$-times composition of $g_3$.

We have that
\begin{align*}
    \lim\limits_{L \to \infty} k^{(L)} (\mathbf{x}_1, \mathbf{x}_2) &= \lim\limits_{L \to \infty} \sqrt{g_1^{(L)}(s^2)g_2^{(L)}(s^2)} \lim\limits_{L \to \infty} g_3^{(L)}(s^2, s^2, \rho) \\
    &= s^2.
\end{align*}

Also, for any $\mathbf{x}_3 \in \mathcal{X} \supseteq \mathcal{X}_*$, noting that $\Vert\mathbf{x}_3 \Vert^2$ is a fixed point of $g_1$ and $g_2$,
    \begin{align*}
    \lim\limits_{L \to \infty} k^{(L)} (\mathbf{x}_1, \mathbf{x}_3) = \lim\limits_{L \to \infty} k^{(L)} (\mathbf{x}_2, \mathbf{x}_3) &= \lim\limits_{L \to \infty} \sqrt{g_1^{(L)}(s^2)g_2^{(L)}(\Vert \mathbf{x}_3 \Vert^2)} \lim\limits_{L \to \infty} g_3^{(L)}(s^2, \Vert \mathbf{x}_3 \Vert^2, \rho) \\
    &= s \Vert \mathbf{x}_3 \Vert.
\end{align*}
We may now apply Propositions~\ref{propDegeneratePrior} and~\ref{propDegeneratePosterior}.
\end{proof}

\section{Extension to NTK}
\label{app:extension_ntk}
The problem in applying Theorem~\ref{thm:banach} to \S~\ref{sec:ntk_jacob} is that Theorem~\ref{thm:banach} requires our iterates to be over a closure of a strictly convex set. Inspecting~\eqref{eq:ntk}, we see that the NTK $T^{(l+1)}(\mathbf{x}_1, \mathbf{x}_2)$ is a sum of $l+1$ terms
\begin{align*}
T^{(l+1)}(\mathbf{x}_1, \mathbf{x}_2) &= k^{(l+1)}(\mathbf{x}_1, \mathbf{x}_2) + \sum_{j=1}^{l} k^{(j)}(\mathbf{x}_1, \mathbf{x}_2) \prod_{p=j+1}^{l+1} \dot{k}^{(p)}(\mathbf{x}_1, \mathbf{x}_2),
\end{align*}
which can grow like $l$ and push iterates outside of a bounded set. We are therefore motivated to consider a rescaled NTK, 
\begin{align*}
T^{(l+1)}(\mathbf{x}_1, \mathbf{x}_2) &:= \frac{1}{l+1} \Big( k^{(l+1)}(\mathbf{x}_1, \mathbf{x}_2) + \sum_{j=1}^{l} k^{(j)}(\mathbf{x}_1, \mathbf{x}_2) \prod_{p=j+1}^{l+1} \dot{k}^{(p)}(\mathbf{x}_1, \mathbf{x}_2) \Big), \\
&= \frac{1}{l+1} \Big( l T^{(l)}\dot{k}^{(l+1)}(\mathbf{x}_1, \mathbf{x}_2) + k^{(l+1)}(\mathbf{x}_1, \mathbf{x}_2) \Big)
\end{align*}
to ensure that the sum stays of the right order across iterates of depth. In order to do so, we construct an augmented state space that additionally allows us to track the reciprocal $\tau$ of $l+1$,
\begin{align*}
h_3(k, T, \tau) &= \sigma_w^2\mathbb{E}\big[ \psi( s Z_1) \psi( s Z_2) \big] + \sigma_b^2 , \\
h_4(k, T, \tau) &= \tau \Big( \big(\tau^{-1}-1\big) T\sigma_w^2 \mathbb{E}\big[ \psi'( s Z_1) \psi'( s Z_2) \big] + \sigma_w^2\mathbb{E}\big[ \psi( s Z_1) \psi( s Z_2) \big] + \sigma_b^2 \Big), \\
h_5(k, T, \tau)&= \frac{1}{\frac{1}{\tau} + 1},
\end{align*} 
so that repeated applications of $h_5(\cdot, \cdot, \tau)$ generates the sequence $\frac{1}{2}, \frac{1}{3}, \frac{1}{4}, \dotsc$. Here we have assumed $h_1$ and $h_2$ stay constant through the iterates, as in \S~\ref{sec:ntk_jacob}. If $|\lambda_3| < 1$, then $|h_4(k, T, \tau)| < 1$. This allows us to consider $(k, T, \tau)$ as elements in the closure of $ \mathcal{S}$, where $\mathcal{S}=(-s^2, s^2)^2 \times (0, \frac{1}{2}] $. The Jacobian is then
\begin{align*}
J = \begin{pmatrix}
\frac{\partial h_3}{\partial k}  & \frac{\partial h_3}{\partial T} & \frac{\partial h_3}{\partial \tau} \\[6pt]
\frac{\partial h_4}{\partial k}  & \frac{\partial h_4}{\partial T} & \frac{\partial h_4}{\partial \tau} \\[6pt]
\frac{\partial h_5}{\partial k}  & \frac{\partial h_5}{\partial T} & \frac{\partial h_5}{\partial \tau} 
\end{pmatrix} &= \begin{pmatrix}
\frac{\partial h_3}{\partial k}  & 0 & 0 \\[6pt]
\frac{\partial h_4}{\partial k}  & \frac{\partial h_4}{\partial T} & \frac{\partial h_4}{\partial \tau} \\[6pt]
0  & 0 & \frac{1}{(1+\tau)^2}
\end{pmatrix},
\end{align*}
with eigenvalues given by the diagonal elements,
\begin{align*}
\frac{\partial h_3}{\partial k}  &= \frac{\partial h_3 }{\partial \rho }\big( \frac{\partial k}{\partial \rho} \big)^{-1} = \frac{\partial g_3 }{\partial \rho }\sqrt{h_1 h_2} (s_1 s_2)^{-1}  = (\sigma^*)^2 \mathbb{E}\big[ \psi'( s_1 Z_1) \psi'( s_2 Z_2) \big] =\lambda_3, \\
 \frac{\partial h_4}{\partial T} &= (1-\tau) \lambda_3 < \lambda_3, \text{ and} \\
 \frac{1}{(1+\tau)^2} &< 1.
\end{align*}

\section{Experimental details}
\citet{HernandezLobato2015} introduce a \emph{regression} benchmark consisting of $9$ datasets. These datasets contain $N$ datapoints and $d$ dimensions, as indicated in Table~\ref{tab_bench_all}.
\subsection{Shallow models}
\label{sec:appendix_shallow_models}

We largely followed the experimental protocol originally established by \citet{HernandezLobato2015} which was designed to assess uncertainty estimates of Bayesian NNs. Since we used GPs, we found it necessary to deviate in several places as follows.

We performed exact GP regression, and hence didn't scale up to the larger datasets using sparse methods. We excluded Song Year (the largest dataset with $N=515,345$) and uniformly subsampled the Protein dataset down to $N=10,000$ datapoints. We had four hyperparameters to tune; first layer weight variance, first layer bias variance, final layer weight variance, data noise variance. We excluded the final bias to reduce the search space; this bias term should not be required due to our standardisation prepossessing. We allowed these to vary in the below ranges.

$\sigma^2_{w0} \in \{ 10,3.5,3,2.5,2,1.5,1,0.5,0.1\}$

$\sigma^2_{b0} \in \{ 10,3.5,3,2.5,2,1.5,1,0.5,0.1\}$

$\sigma^2_{w1} \in \{ 10,3.5,3,2.5,2,1.5,1,0.5,0.1\}$

$\sigma^2_\epsilon \in \{ 1e0,1e-1,1e-2,1e-3,1e-4,1e-5\}$

To choose hyperparameters, we ran a full grid search over these values using $70\%$ for training and $20\%$ as a validation set. We selected according to NLL on the validation set (setting according to marginal likelihood produced comparable though slightly worse results). Unlike the original protocol, we performed this tuning only on the first train/test split of each dataset, holding hyperparameters constant for the remaining splits. The leaky ReLU kernel has an additional parameter, $\alpha$, corresponding the negative slope. This was fixed at $\alpha=0.2$.

Test/train splits were randomly shuffled but used the same random seed across kernels. Experiments were repeated 20 times for all datasets except Protein which was only repeated 5 times. Following previous works reporting on the benchmark, we include results in tabular format in Table \ref{tab_bench_all}. These are consistent with Figures \ref{fig:bench_regression_RMSE} and~\ref{fig:bench_regression_nll}. The experimental evaluation of~\citet{bui2016deep} performs the same benchmark on a number of other methods. In order to compare our results to others, we include their benchmark results in Table~\ref{tab:bench_other}.

\begin{table*}%
\caption{Performance on benchmark regression datasets using infinite wide Bayesian neural networks. Mean $\pm$ 1 standard error.}
\vskip 0.15in
\begin{center}
\resizebox{\textwidth}{!}{
\begin{tabular}{ l c c | r r r r | r r r r}
\multicolumn{3}{c}{}  & \multicolumn{4}{c}{\textbf{NLL}} &
\multicolumn{4}{|c}{\textbf{RMSE}} \\

\multicolumn{1}{c}{} 
& \multicolumn{1}{c}{$N$}
& \multicolumn{1}{c}{$d$}
& \multicolumn{1}{c}{GELU}
&  \multicolumn{1}{c}{ReLU} 
&  \multicolumn{1}{c}{L. ReLU} 
&  \multicolumn{1}{c}{ERF}
& \multicolumn{1}{|c}{GELU}
&  \multicolumn{1}{c}{ReLU} 
&  \multicolumn{1}{c}{L. ReLU} 
&  \multicolumn{1}{c}{ERF} \\ 
 
\hline 
Boston & 506 & 13 & 2.54 $\pm$ 0.05 & 2.54 $\pm$ 0.05 & 2.55 $\pm$ 0.05 & 2.48 $\pm$ 0.05 & 3.09 $\pm$ 0.18 & 3.11 $\pm$ 0.18 & 3.14 $\pm$ 0.18 & 2.89 $\pm$ 0.16 \\
Concrete & 1030 & 8 & 3.05 $\pm$ 0.01 & 3.05 $\pm$ 0.01 & 3.05 $\pm$ 0.01 & 3.07 $\pm$ 0.01 & 5.09 $\pm$ 0.12 & 5.04 $\pm$ 0.12 & 4.99 $\pm$ 0.12 & 5.08 $\pm$ 0.12 \\
Energy & 768 & 8 & 0.82 $\pm$ 0.02 & 0.83 $\pm$ 0.02 & 0.86 $\pm$ 0.02 & 0.88 $\pm$ 0.03 & 0.59 $\pm$ 0.02 & 0.57 $\pm$ 0.02 & 0.58 $\pm$ 0.02 & 0.64 $\pm$ 0.02 \\
Kin8nm & 8192 & 8 & -1.18 $\pm$ 0.01 & -1.19 $\pm$ 0.01 & -1.20 $\pm$ 0.01 & -1.18 $\pm$ 0.01 & 0.07 $\pm$ 0.00 & 0.07 $\pm$ 0.00 & 0.07 $\pm$ 0.00 & 0.07 $\pm$ 0.00 \\
Naval & 11,934 & 16 & -7.89 $\pm$ 0.00 & -7.87 $\pm$ 0.00 & -7.87 $\pm$ 0.00 & -7.88 $\pm$ 0.00 & 0.00 $\pm$ 0.00 & 0.00 $\pm$ 0.00 & 0.00 $\pm$ 0.00 & 0.00 $\pm$ 0.00 \\
Power & 9568& 4 & 2.88 $\pm$ 0.01 & 2.97 $\pm$ 0.03 & 2.88 $\pm$ 0.01 & 2.88 $\pm$ 0.01 & 3.93 $\pm$ 0.04 & 3.65 $\pm$ 0.04 & 3.89 $\pm$ 0.04 & 3.91 $\pm$ 0.04 \\
Protein &45,730 & 9 & 2.94 $\pm$ 0.01 & 2.94 $\pm$ 0.01 & 2.99 $\pm$ 0.00 & 2.82 $\pm$ 0.01 & 4.01 $\pm$ 0.02 & 3.99 $\pm$ 0.02 & 4.22 $\pm$ 0.02 & 3.92 $\pm$ 0.02 \\
Wine & 1599 & 11 & -0.04 $\pm$ 0.04 & 0.01 $\pm$ 0.04 & -0.06 $\pm$ 0.04 & -0.13 $\pm$ 0.04 & 0.65 $\pm$ 0.01 & 0.66 $\pm$ 0.01 & 0.66 $\pm$ 0.01 & 0.60 $\pm$ 0.01 \\
Yacht & 308 & 6 & 0.02 $\pm$ 0.07 & 0.48 $\pm$ 0.09 & 0.52 $\pm$ 0.09 & 0.17 $\pm$ 0.07 & 0.37 $\pm$ 0.04 & 0.59 $\pm$ 0.08 & 0.60 $\pm$ 0.08 & 0.51 $\pm$ 0.06 \\
\end{tabular}
}
\end{center}
\vskip -0.1in
\label{tab_bench_all}
\end{table*}

\begin{table*}%
\caption{Performance (RMSE) on benchmark regression datasets of other methods, plus or minus . See~\citet{bui2016deep} for full details}
\vskip 0.15in
\begin{center}
\resizebox{\textwidth}{!}{
\begin{tabular}{ l c c | r r r r r r r r r r r r r r r r}
\multicolumn{1}{c}{} 
& \multicolumn{1}{c}{$N$}
& \multicolumn{1}{c}{$d$}
& \multicolumn{1}{c}{VI(G)-1}
&  \multicolumn{1}{c}{VI(KW)-1} 
&  \multicolumn{1}{c}{VI(KW)-2} 
&  \multicolumn{1}{c}{PBP-1}
&  \multicolumn{1}{c}{Dropout-1}
&  \multicolumn{1}{c}{SGLD-1} 
&  \multicolumn{1}{c}{SGLD-2} 
&  \multicolumn{1}{c}{HMC-1}
&  \multicolumn{1}{c}{GP 50} 
&  \multicolumn{1}{c}{DGP-1 50} 
&  \multicolumn{1}{c}{DGP-2 50} 
&  \multicolumn{1}{c}{DGP-3 50} 
&  \multicolumn{1}{c}{GP 100} 
&  \multicolumn{1}{c}{DGP-1 100} 
&  \multicolumn{1}{c}{DGP-2 100} 
&  \multicolumn{1}{c}{DGP-3 100} 
\\ 
\hline 
Boston & 506 & 13 & 4.32$\pm$ 0.29 & 2.67$\pm$ 0.11 & 3.06$\pm$ 0.13 & 3.01$\pm$ 0.18 & 2.97$\pm$ 0.85 & 2.21$\pm$ 0.10 & 1.96$\pm$ 0.10 & 2.76$\pm$ 0.20 & 2.43$\pm$ 0.12 & 3.02$\pm$ 0.20 & 2.38$\pm$ 0.12 & 2.33$\pm$ 0.12 & 2.39$\pm$ 0.12 & 3.56$\pm$ 0.29 & 2.38$\pm$ 0.11 & 2.38$\pm$ 0.12 \\
Concrete & 1030 & 8 & 7.13$\pm$ 0.12 & 4.99$\pm$ 0.14 & 5.09$\pm$ 0.12 & 5.67$\pm$ 0.09 & 5.23$\pm$ 0.53 & 4.19$\pm$ 0.13 & 3.70$\pm$ 0.13 & 4.12$\pm$ 0.14 & 5.55$\pm$ 0.12 & 7.33$\pm$ 0.25 & 4.64$\pm$ 0.11 & 4.66$\pm$ 0.13 & 4.78$\pm$ 0.12 & 6.03$\pm$ 0.17 & 4.16$\pm$ 0.13 & 4.23$\pm$ 0.12 \\
Energy & 768 & 8 & 2.65$\pm$ 0.08 & 2.50$\pm$ 0.06 & 1.59$\pm$ 0.13 & 1.80$\pm$ 0.05 & 1.66$\pm$ 0.19 & 1.15$\pm$ 0.03 & 0.91$\pm$ 0.03 & 0.48$\pm$ 0.01 & 1.02$\pm$ 0.02 & 0.84$\pm$ 0.03 & 0.57$\pm$ 0.02 & 0.54$\pm$ 0.01 & 0.87$\pm$ 0.03 & 0.80$\pm$ 0.03 & 0.56$\pm$ 0.02 & 0.56$\pm$ 0.01 \\
Kin8nm & 8192 & 8 & 0.10$\pm$ 0.00 & 0.02$\pm$ 0.00 & 0.01$\pm$ 0.00 & 0.10$\pm$ 0.00 & 0.10$\pm$ 0.00 & 0.02$\pm$ 0.00& 0.01$\pm$ 0.00 & 0.06$\pm$ 0.00 & 0.07$\pm$ 0.00 & 0.22$\pm$ 0.02 & 0.05$\pm$ 0.00 & 0.04$\pm$ 0.00 & 0.06$\pm$ 0.00 & 0.15$\pm$ 0.01 & 0.03$\pm$ 0.00 & 0.02$\pm$ 0.00 \\
Naval & 11934 & 16 & 0.01$\pm$ 0.00 & 0.00$\pm$ 0.00 & 0.00$\pm$ 0.00 & 0.01$\pm$ 0.00 & 0.01$\pm$ 0.00 & 0.00$\pm$ 0.00 & 0.00$\pm$ 0.00 & 0.00$\pm$ 0.00 & 0.00$\pm$ 0.00 & 0.01$\pm$ 0.00 & 0.01$\pm$ 0.00 & 0.00$\pm$ 0.00 & 0.00$\pm$ 0.00 & 0.01$\pm$ 0.00 & 0.00$\pm$ 0.00 & 0.00$\pm$ 0.00 \\
Power & 9568 & 4 & 4.33$\pm$ 0.04 & 3.45$\pm$ 0.03 & 2.35$\pm$ 0.05 & 4.12$\pm$ 0.03 & 4.02$\pm$ 0.18 & 2.42$\pm$ 0.02 & 1.25$\pm$ 0.02 & 3.73$\pm$ 0.04 & 3.75$\pm$ 0.03 & 4.71$\pm$ 0.09 & 3.60$\pm$ 0.03 & 3.60$\pm$ 0.04 & 3.60$\pm$ 0.04 & 4.08$\pm$ 0.08 & 3.21$\pm$ 0.06 & 3.18$\pm$ 0.05 \\
Protein & 45730 & 9 & 4.84$\pm$ 0.03 & 1.48$\pm$ 0.08 & 0.39$\pm$ 0.10 & 4.73$\pm$ 0.01 & 4.36$\pm$ 0.04 & 1.07$\pm$ 0.01 & 0.59$\pm$ 0.00 & 3.91$\pm$ 0.02 & 4.83$\pm$ 0.21 & 4.22$\pm$ 0.08 & 3.24$\pm$ 0.10 & 2.89$\pm$ 0.28 & 4.05$\pm$ 0.13 & 3.69$\pm$ 0.19 & 2.19$\pm$ 0.22 & 2.01$\pm$ 0.16 \\
Red wine & 1588 & 11 & 0.65$\pm$ 0.01 & 0.52$\pm$ 0.01 & 0.56$\pm$ 0.01 & 0.64$\pm$ 0.01 & 0.62$\pm$ 0.04 & 0.21$\pm$ 0.00 & 0.10$\pm$ 0.00 & 0.63$\pm$ 0.01 & 0.57$\pm$ 0.01 & 0.62$\pm$ 0.01 & 0.50$\pm$ 0.01 & 0.48$\pm$ 0.01 & 0.55$\pm$ 0.01 & 0.62$\pm$ 0.02 & 0.41$\pm$ 0.01 & 0.43$\pm$ 0.01 \\
Yacht & 308 & 6 & 6.89$\pm$ 0.67 & 1.30$\pm$ 0.08 & 1.55$\pm$ 0.07 & 1.01$\pm$ 0.05 & 1.11$\pm$ 0.38 & 1.32$\pm$ 0.08 & 2.48$\pm$ 0.18 & 0.56$\pm$ 0.05 & 1.15$\pm$ 0.09 & 1.58$\pm$ 0.37 & 0.98$\pm$ 0.09 & 0.93$\pm$ 0.09 & 1.16$\pm$ 0.07 & 1.84$\pm$ 0.26 & 1.06$\pm$ 0.14 & 0.91$\pm$ 0.08 \\
Year & 515345 & 90 & 9.03$\pm$ NA & 1.15$\pm$ NA & 0.70$\pm$ NA&8.88$\pm$ NA&8.85$\pm$ NA&0.07$\pm$ NA&0.04$\pm$ NA&NA$\pm$ NA&0.79$\pm$ NA&5.28$\pm$ NA&0.45$\pm$ NA&0.26$\pm$ NA&0.27$\pm$ NA&0.51$\pm$ NA&0.22$\pm$ NA&0.37$\pm$ NA \\
\end{tabular}
}
\end{center}
\vskip -0.1in
\label{tab:bench_other}
\end{table*}

\begin{figure}
\centering
	 \includegraphics[scale=0.7]{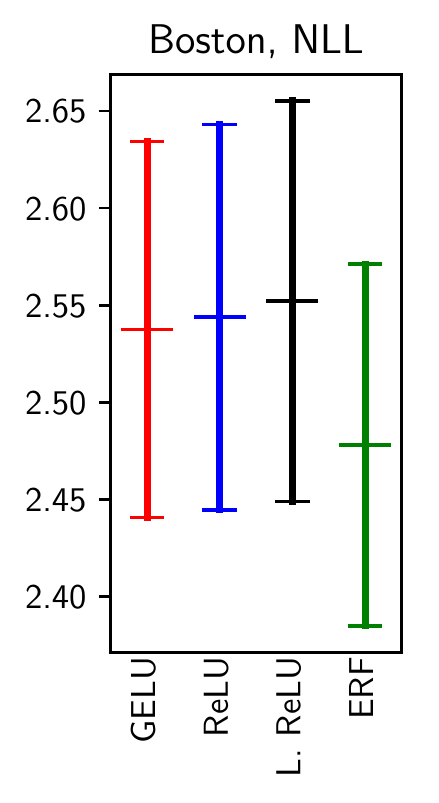}
	 \includegraphics[scale=0.7]{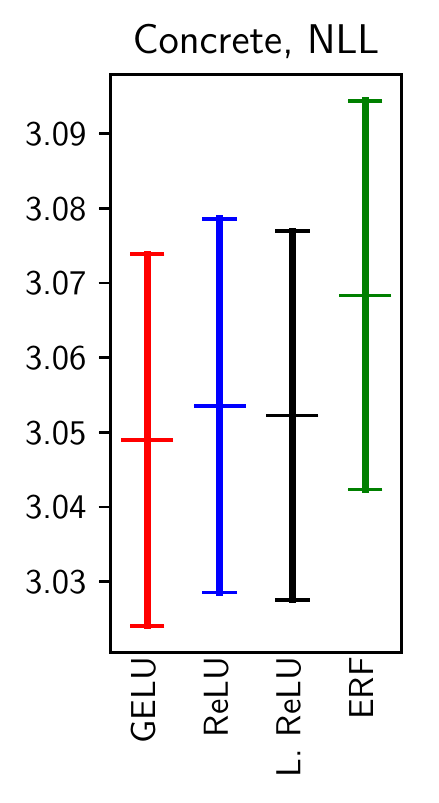}
	 \includegraphics[scale=0.7]{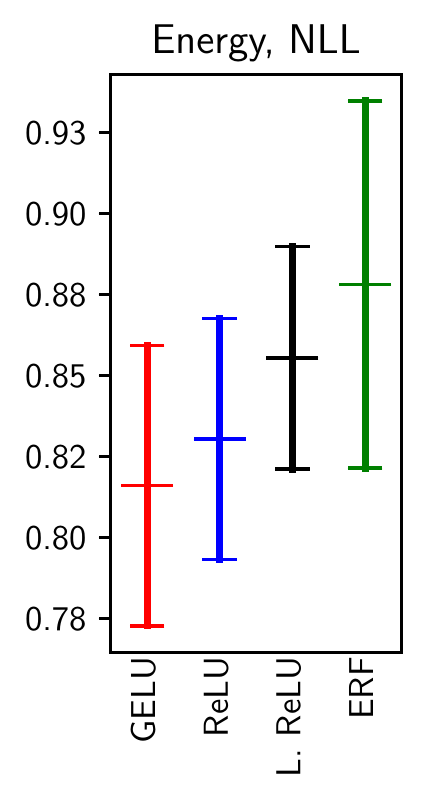} \\
	 \includegraphics[scale=0.7]{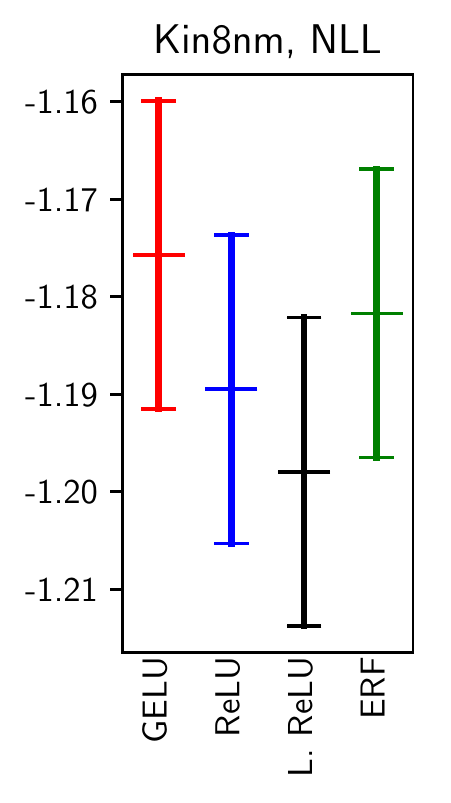}\hspace{-0.06in}
	 \includegraphics[scale=0.7]{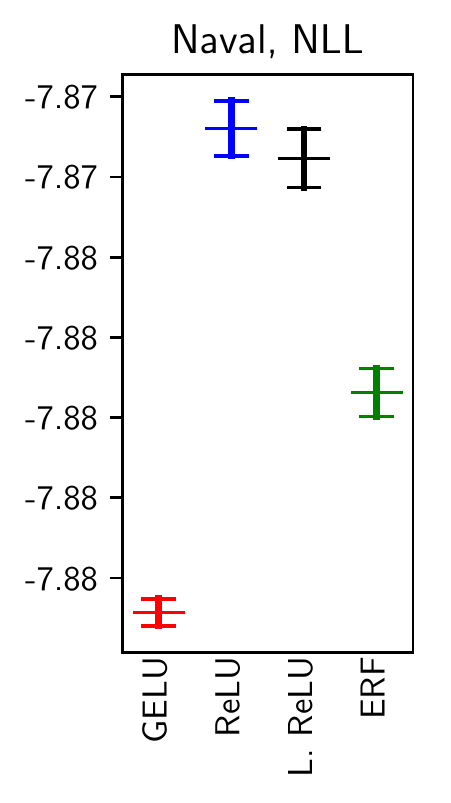}\hspace{-0.06in}
	 \includegraphics[scale=0.7]{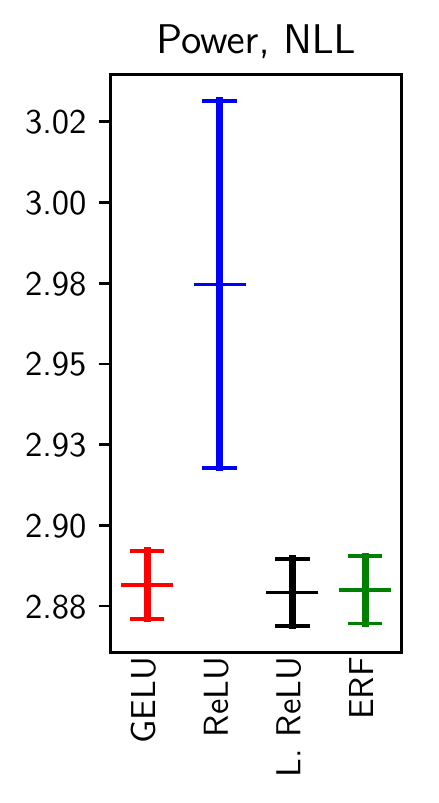} \\
	 \includegraphics[scale=0.7]{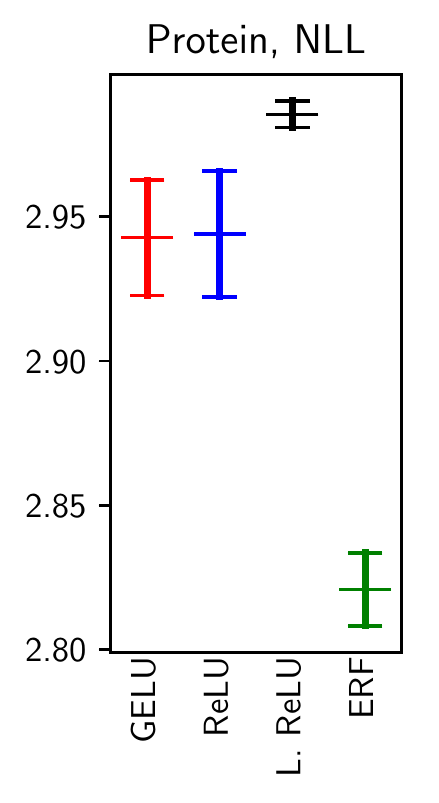}
	 \includegraphics[scale=0.7]{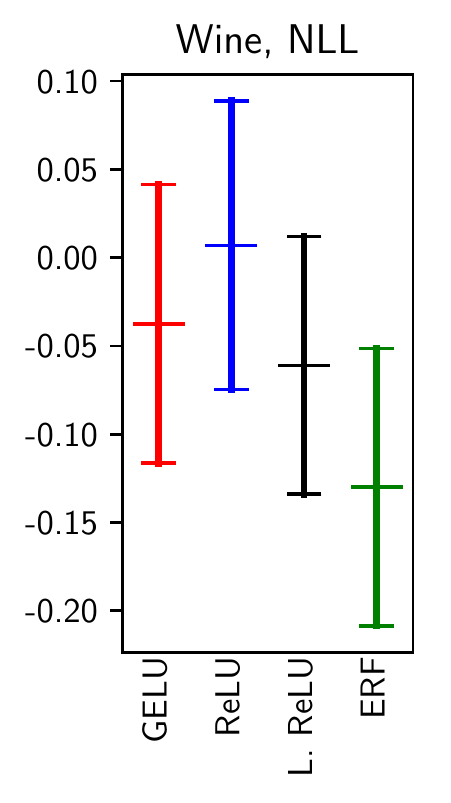}
	 \includegraphics[scale=0.7]{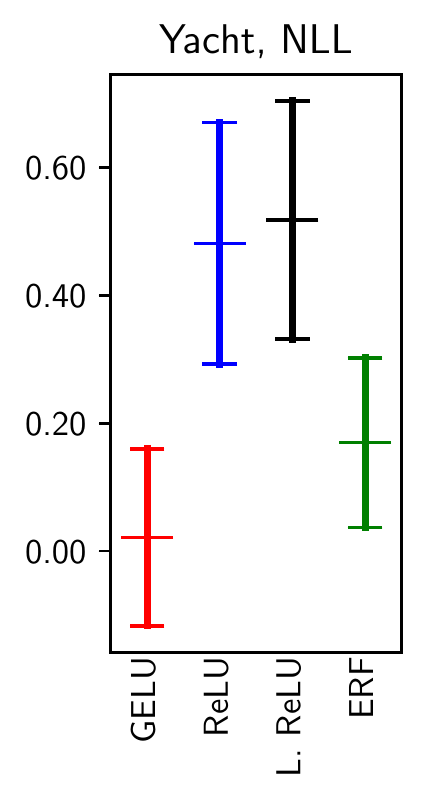}
	 \caption{NLL for single-hidden-layer GPs. Mean $\pm 2$ standard errors (over 20 runs).}
	 	 \label{fig:bench_regression_nll}
\end{figure}

\subsection{Deep models}
\label{app:deep}
We provide RMSE plots for each dataset in Figure~\ref{fig:deep_all}.
\begin{figure*}
\begin{center}
\includegraphics[scale=0.22]{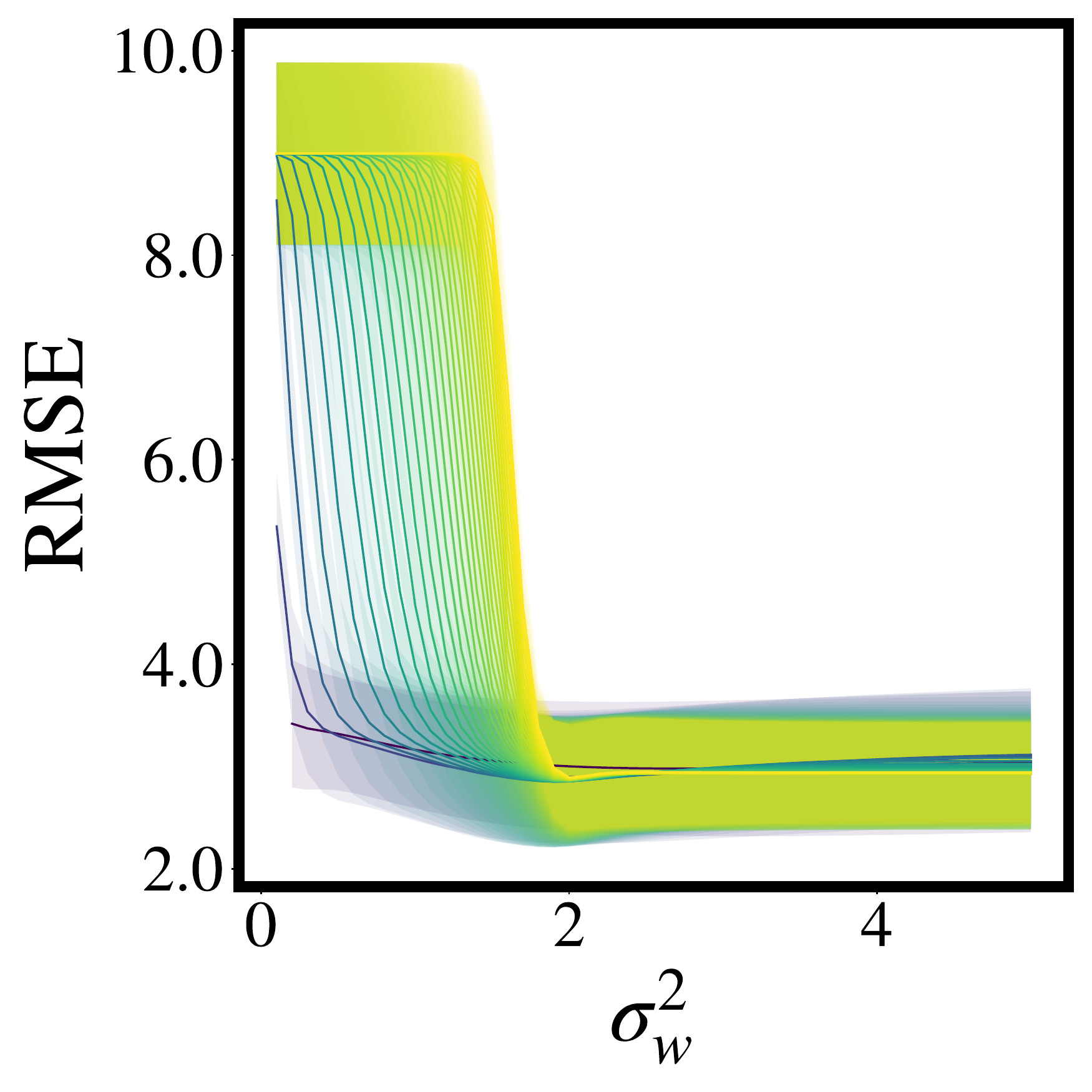}
\includegraphics[scale=0.22]{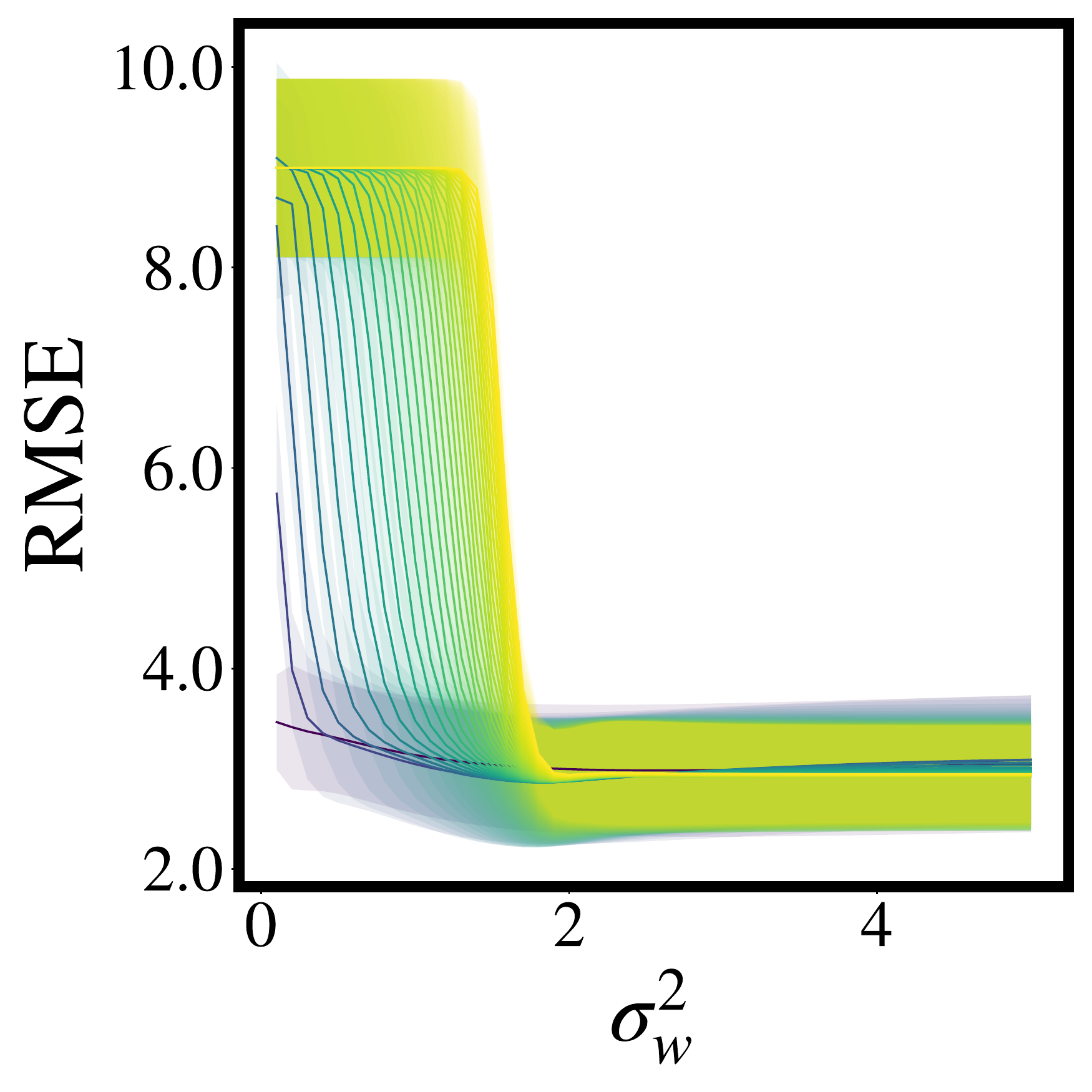}
\includegraphics[scale=0.22]{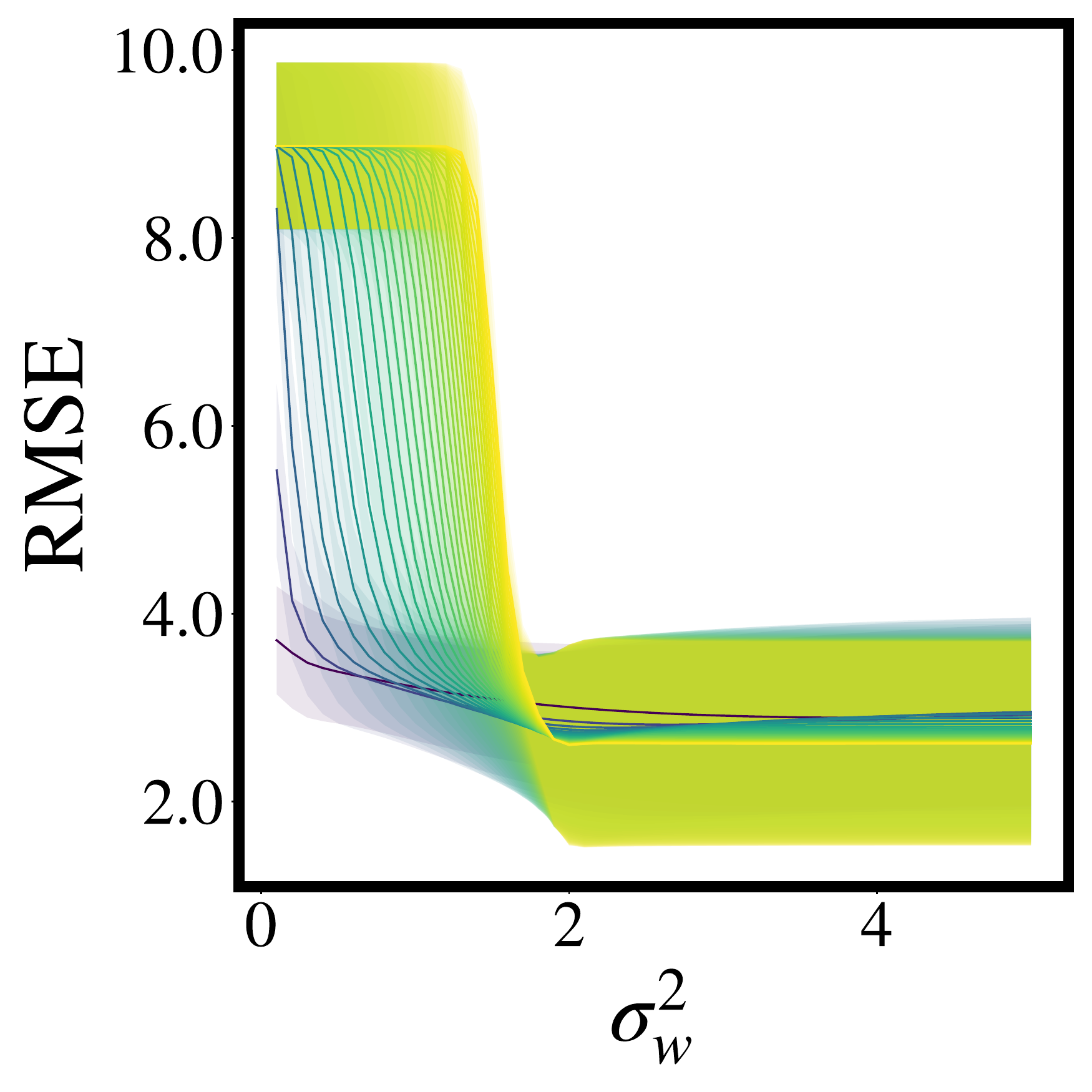}
\includegraphics[scale=0.22]{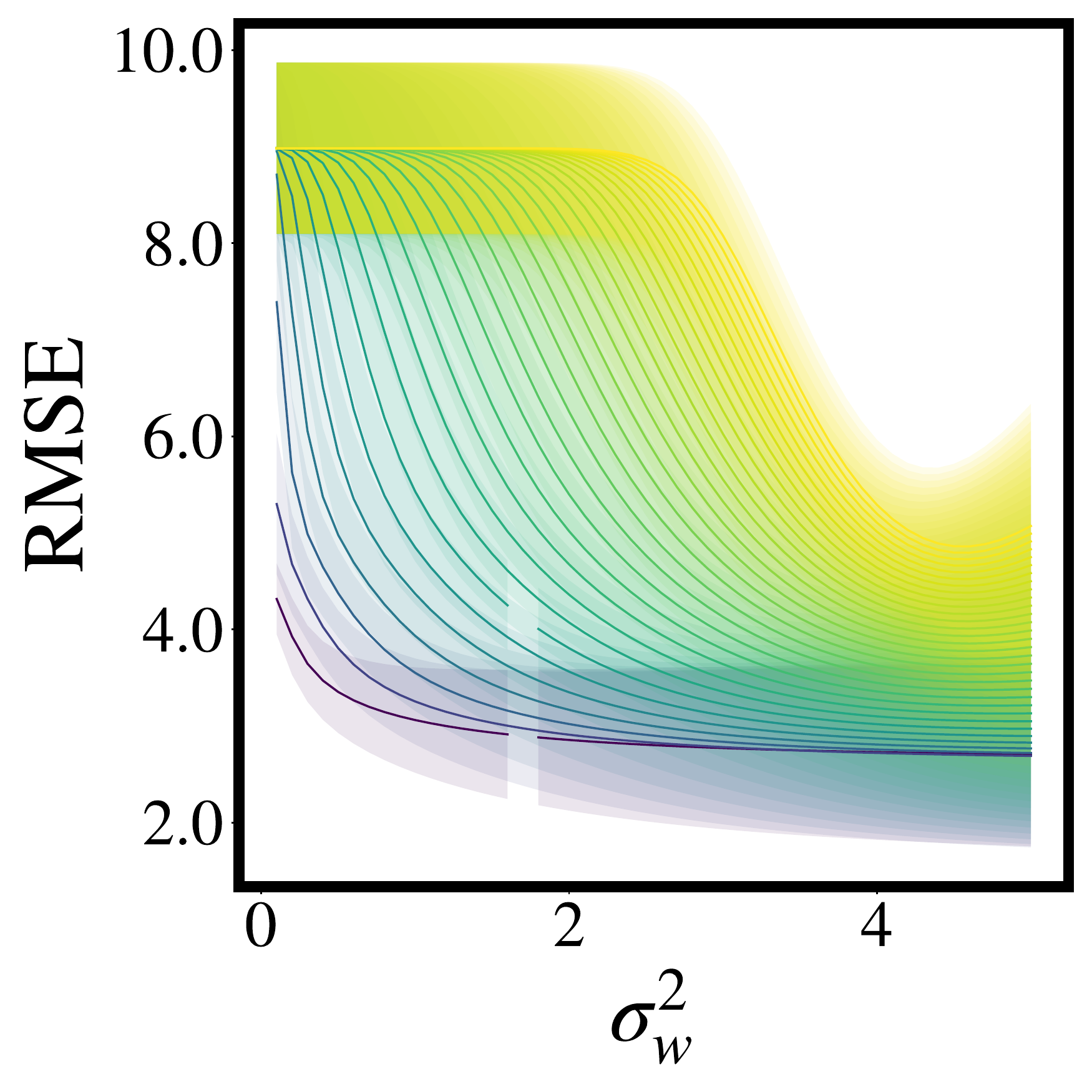} \\

\includegraphics[scale=0.22]{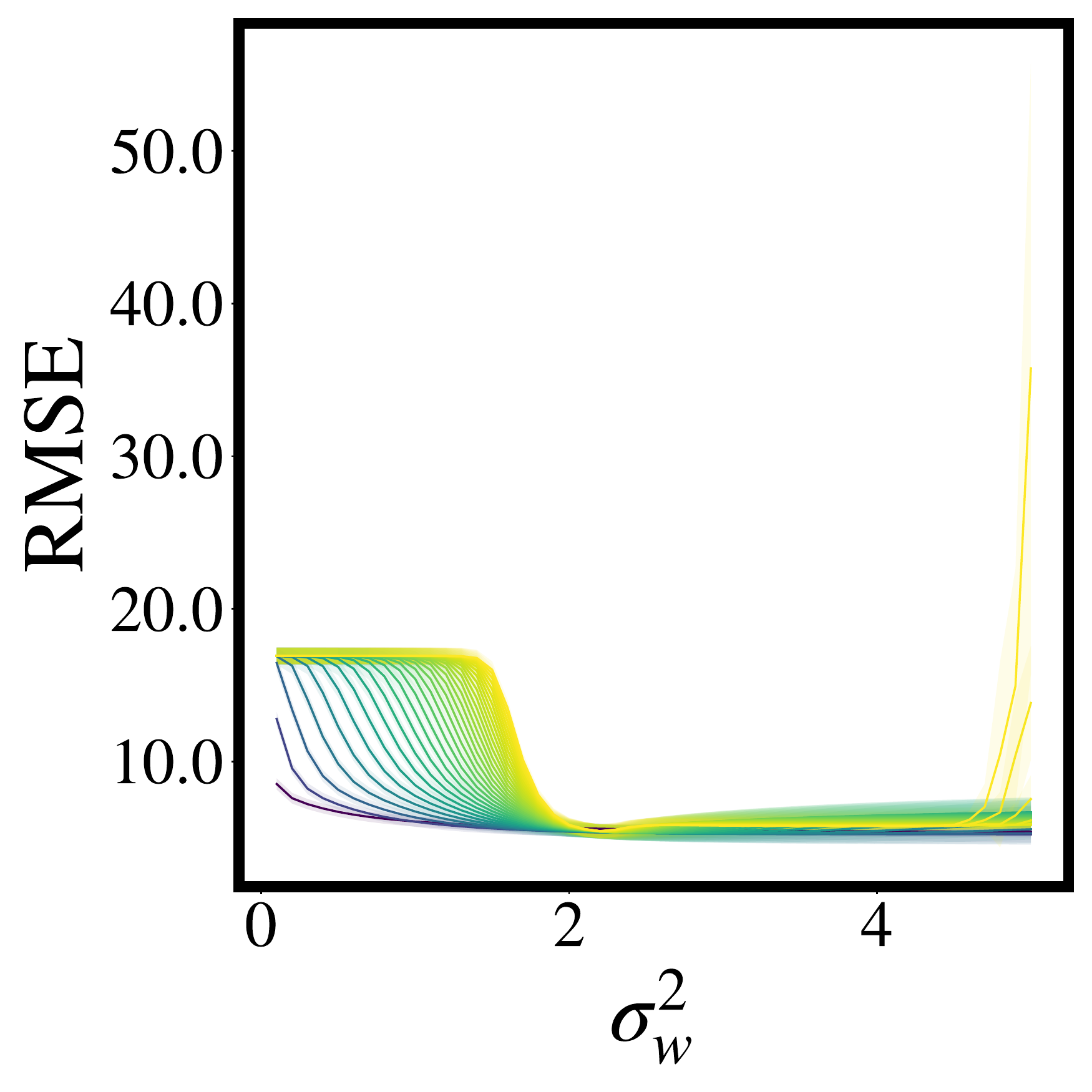}
\includegraphics[scale=0.22]{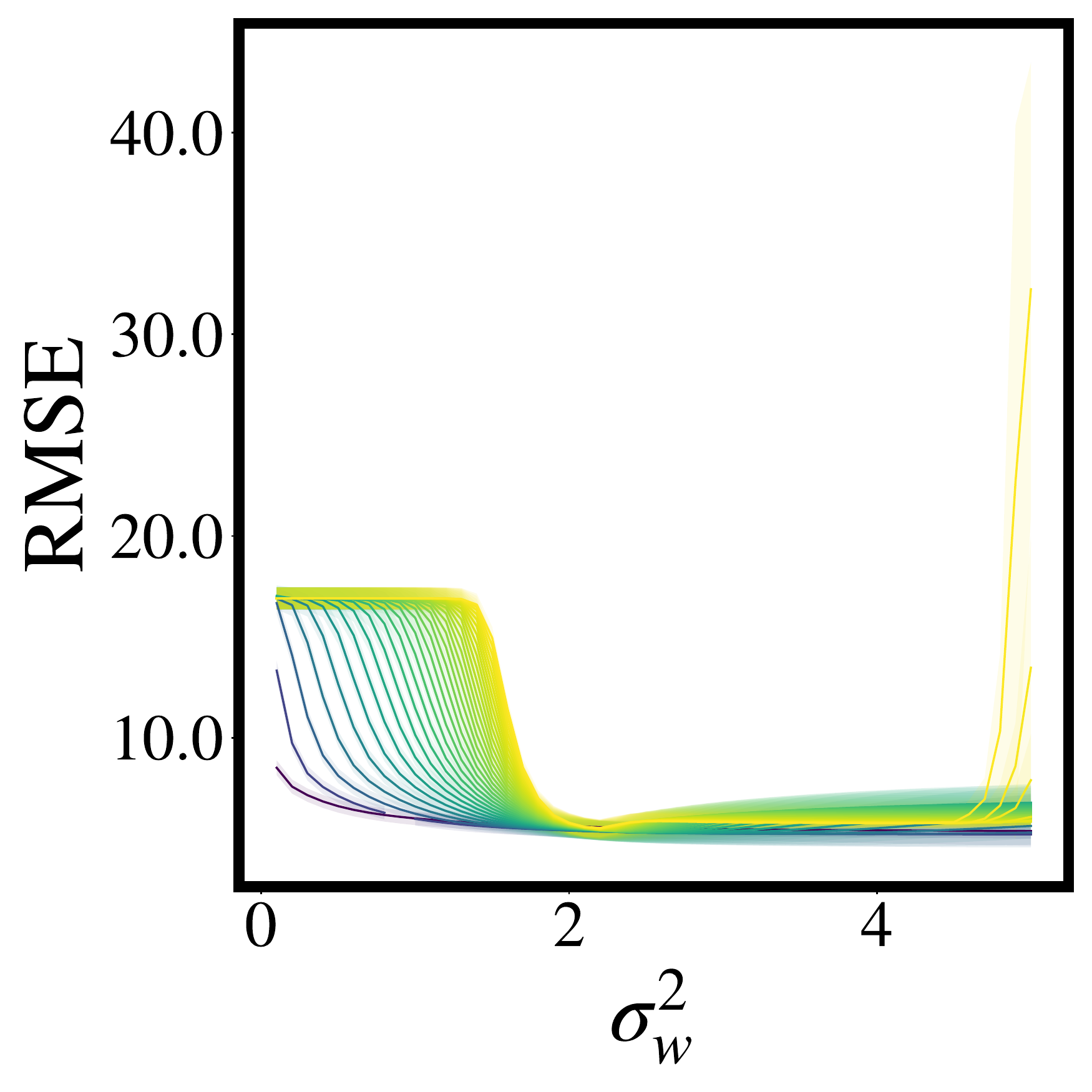}
\includegraphics[scale=0.22]{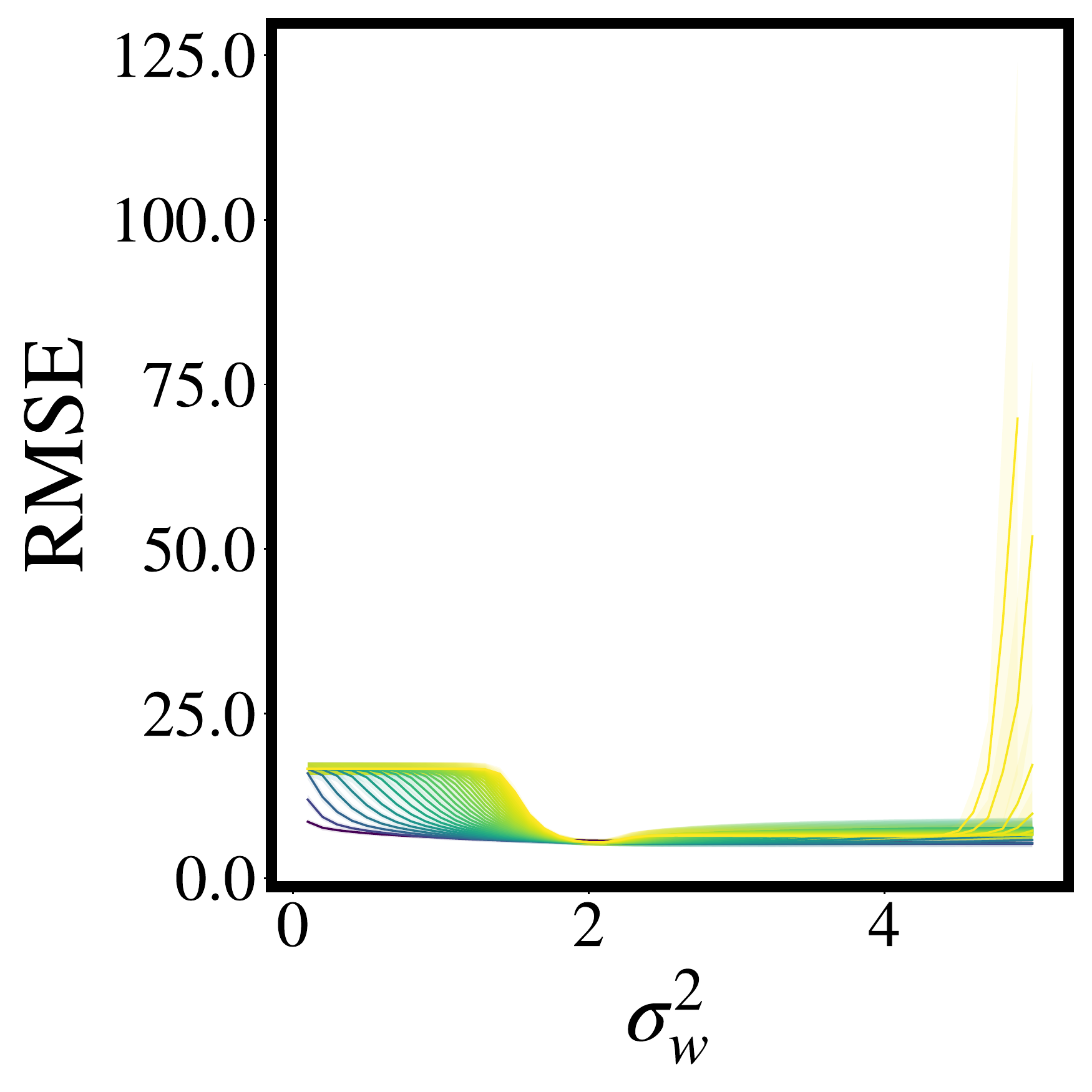}
\includegraphics[scale=0.22]{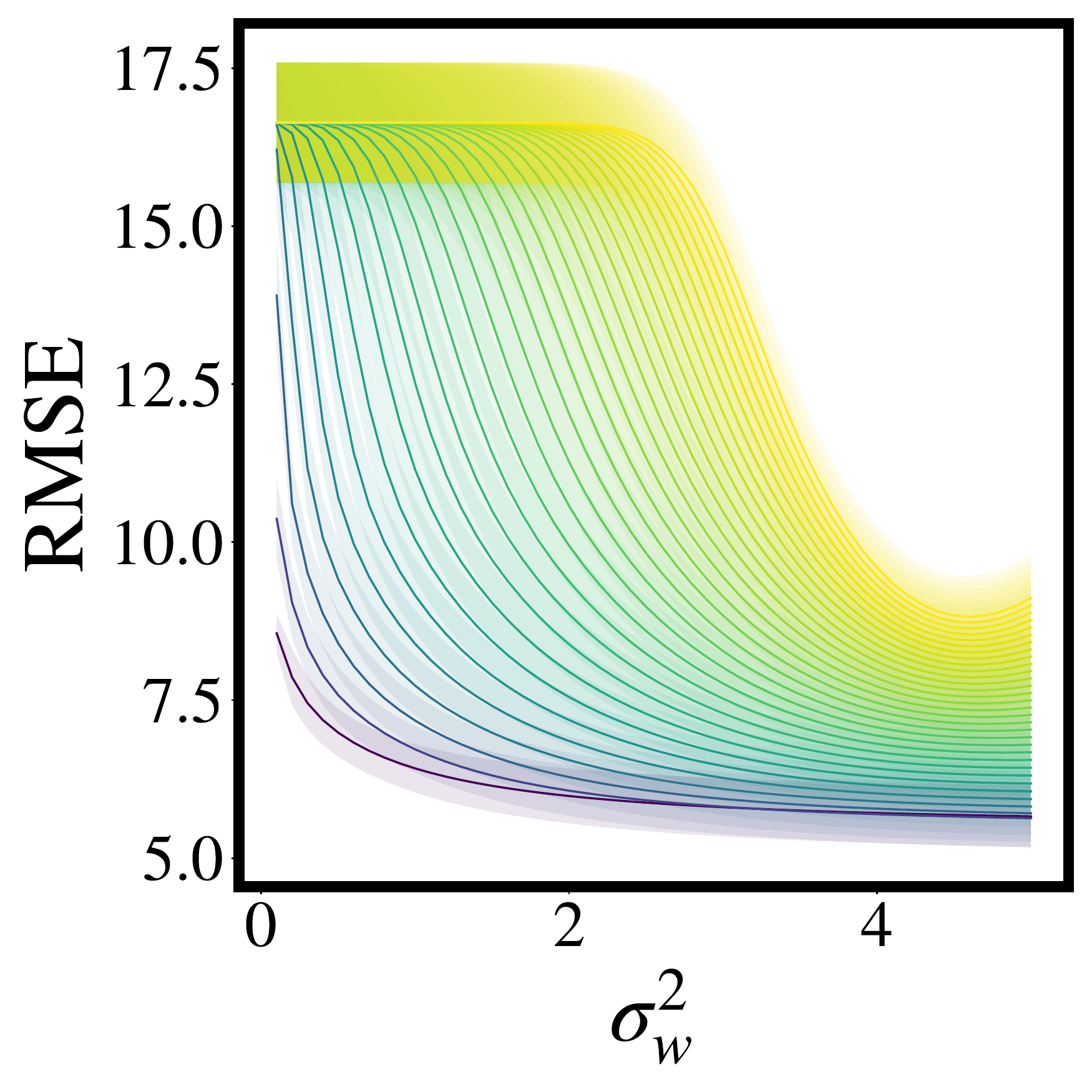}

\includegraphics[scale=0.22]{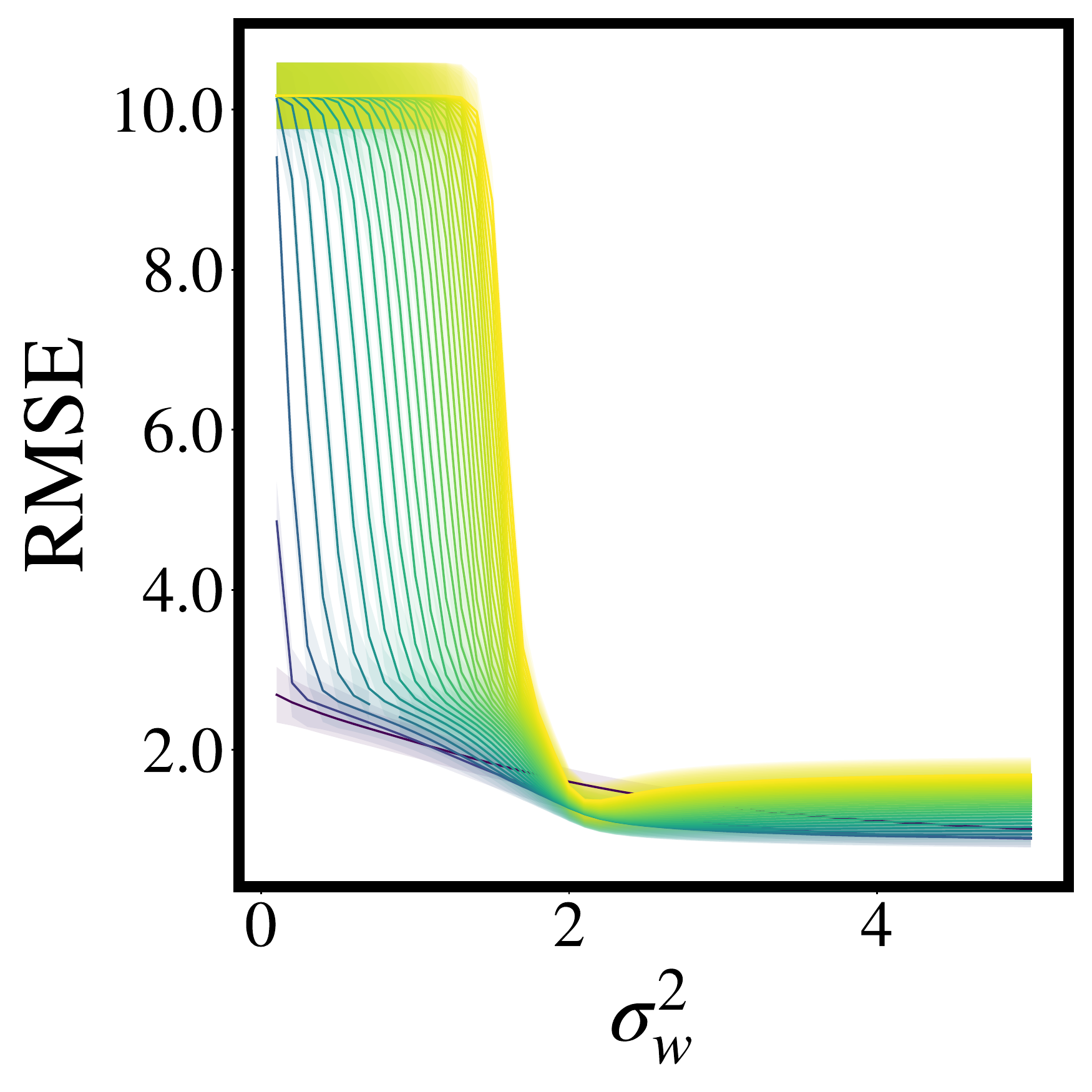}
\includegraphics[scale=0.22]{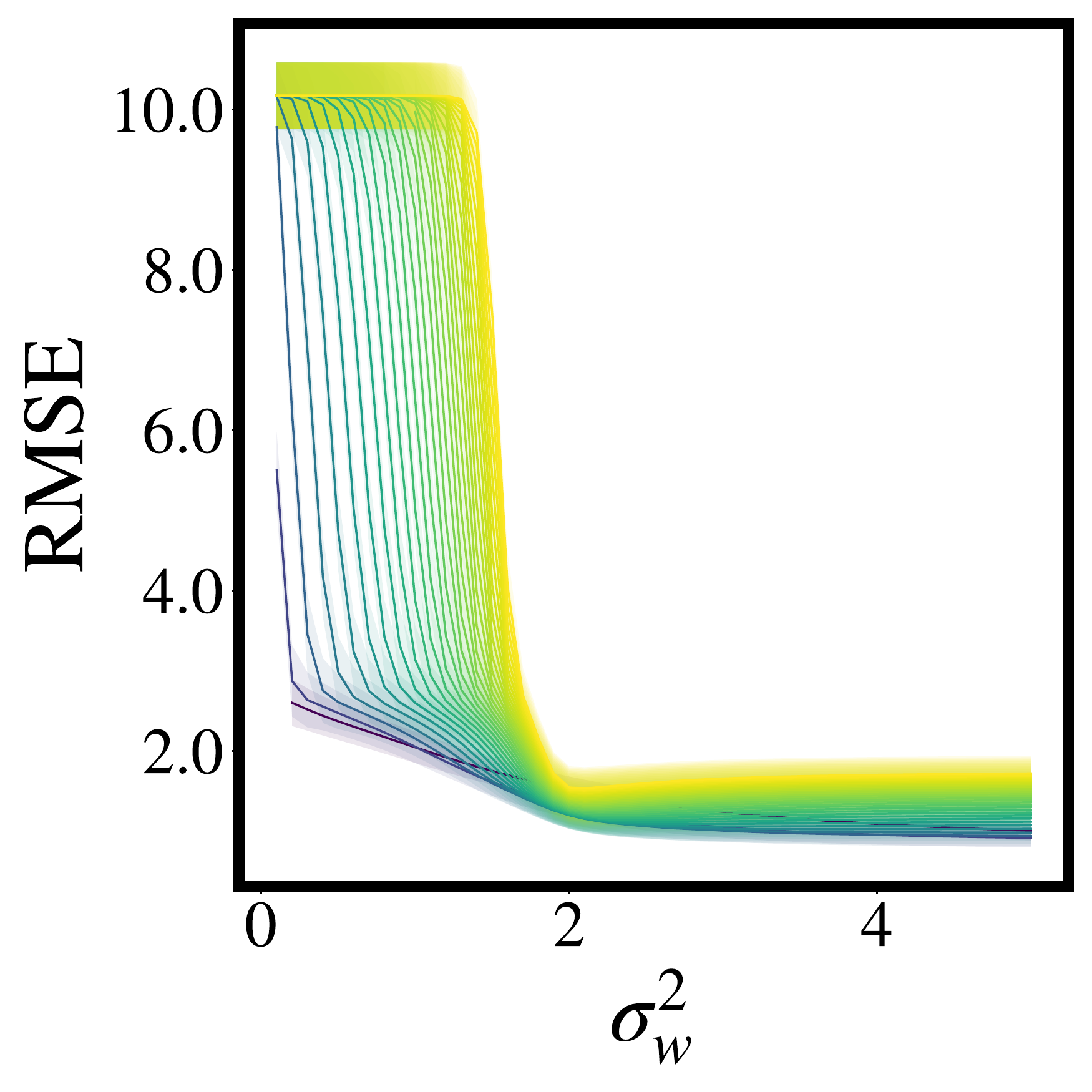}
\includegraphics[scale=0.22]{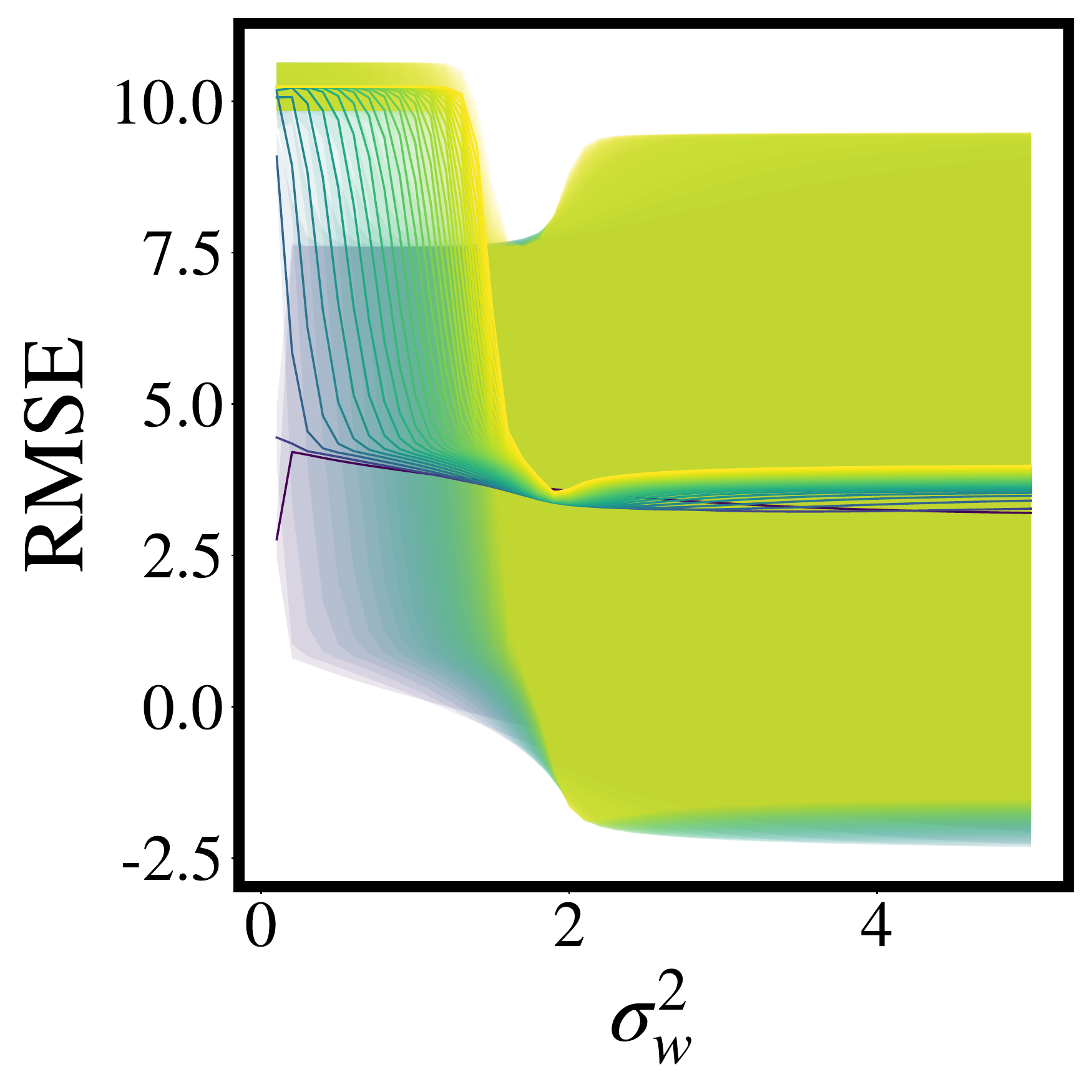}
\includegraphics[scale=0.22]{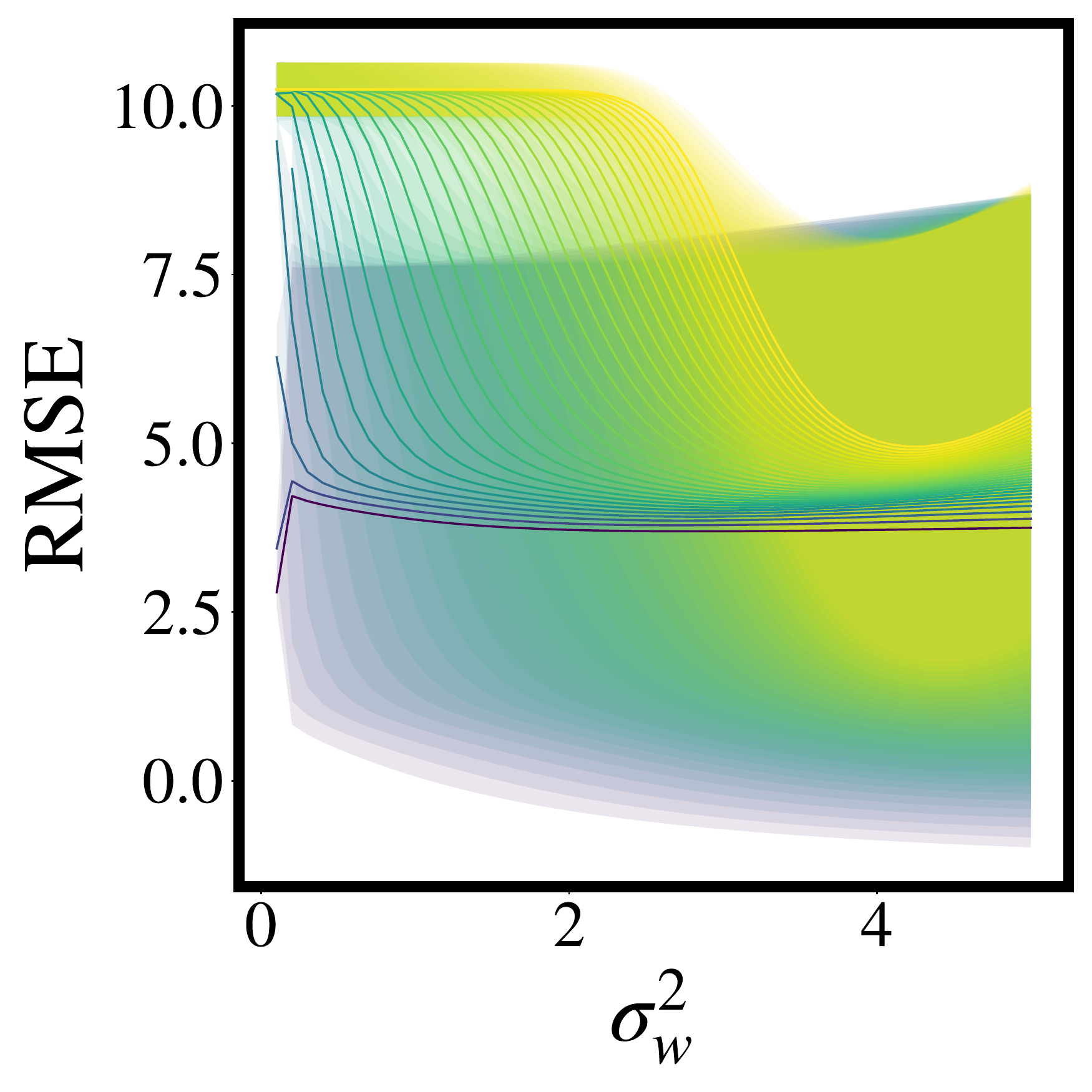}

\includegraphics[scale=0.22]{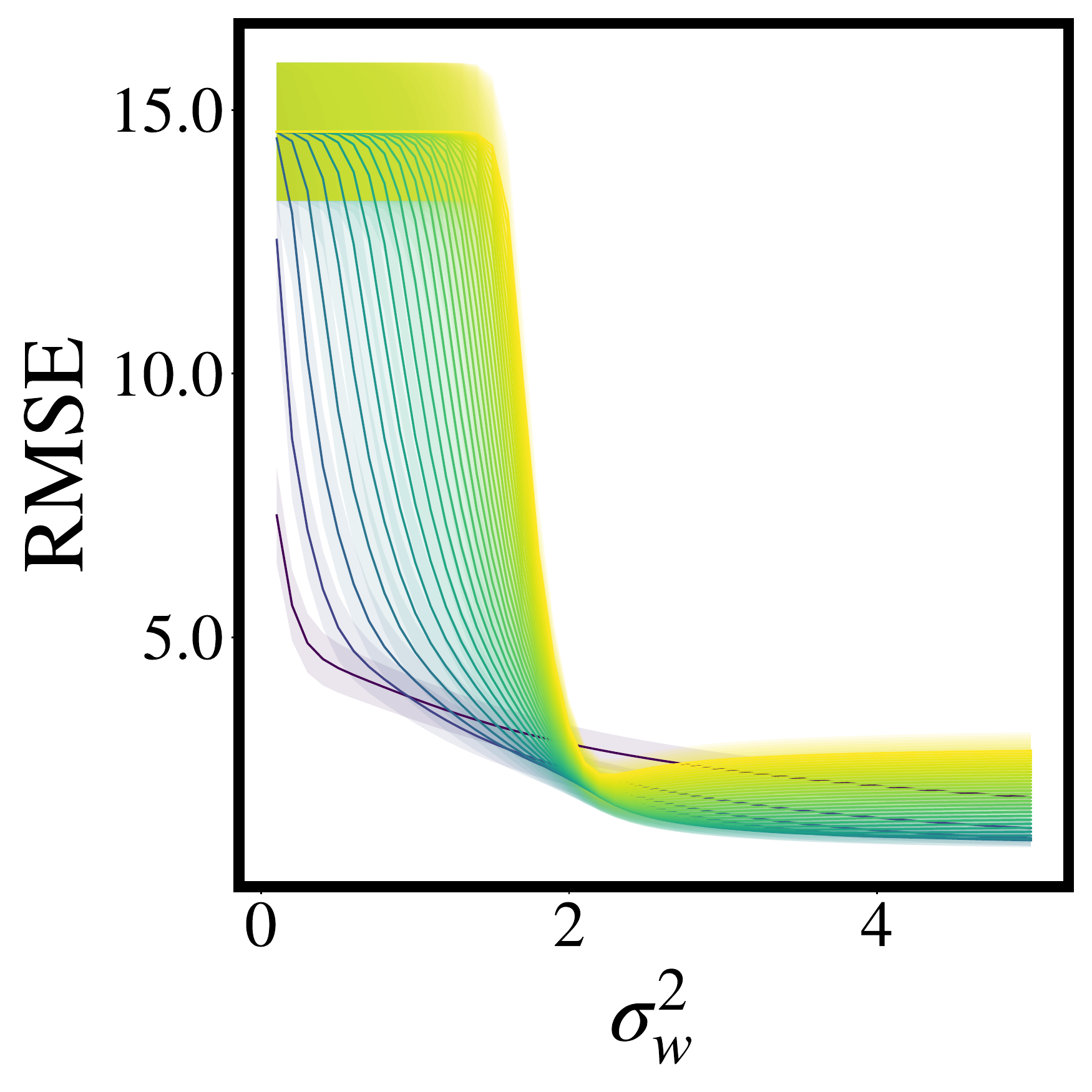}
\includegraphics[scale=0.22]{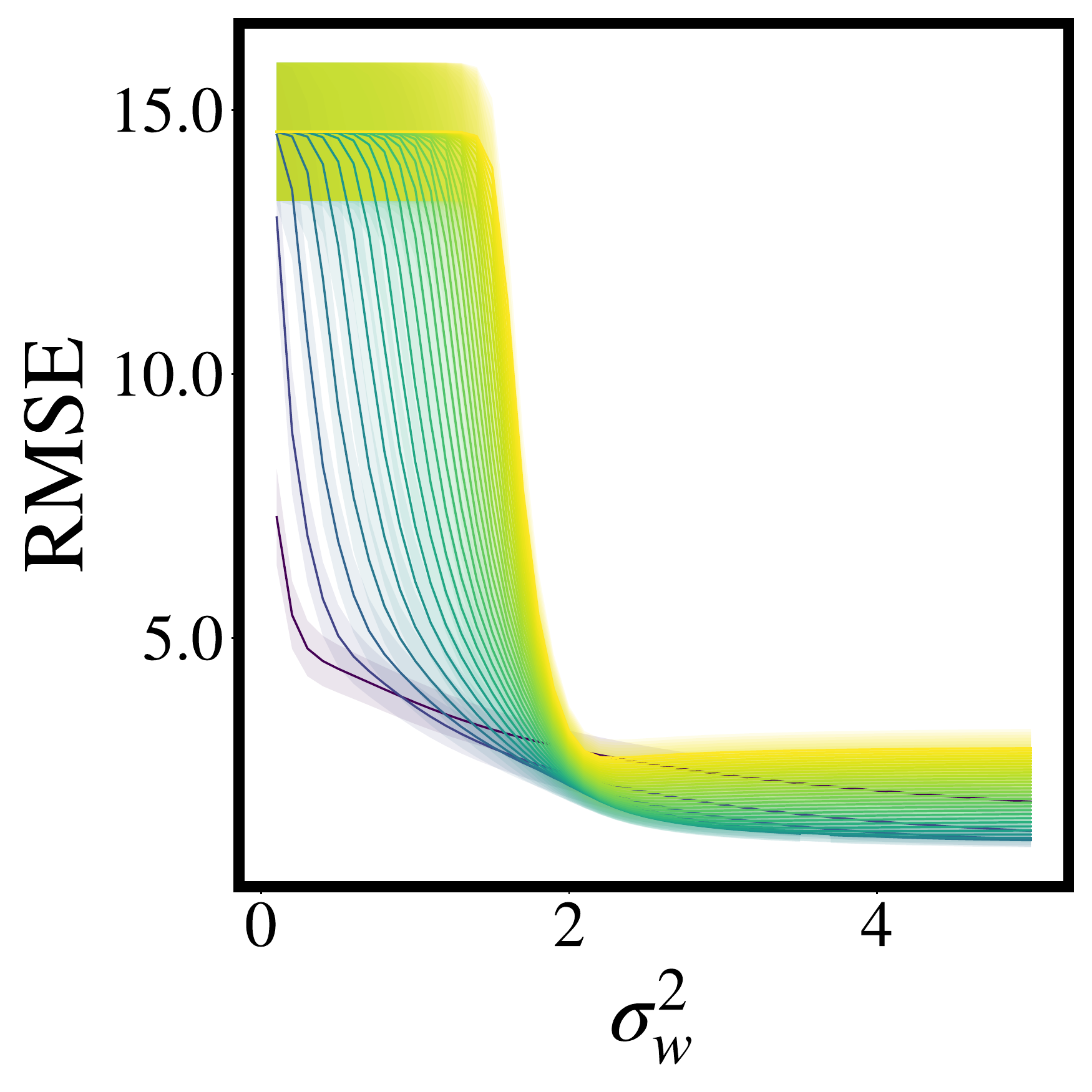}
\includegraphics[scale=0.22]{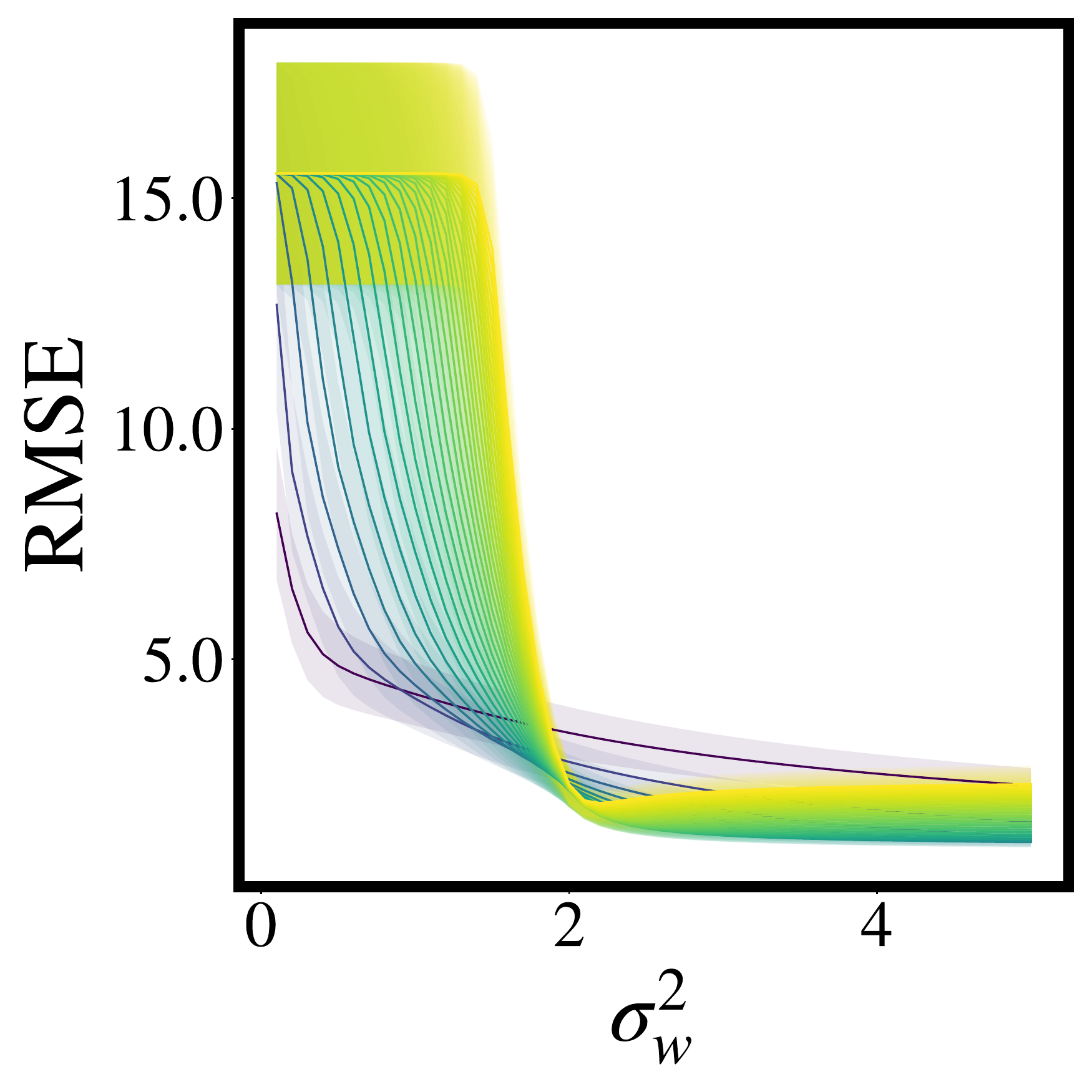}
\includegraphics[scale=0.22]{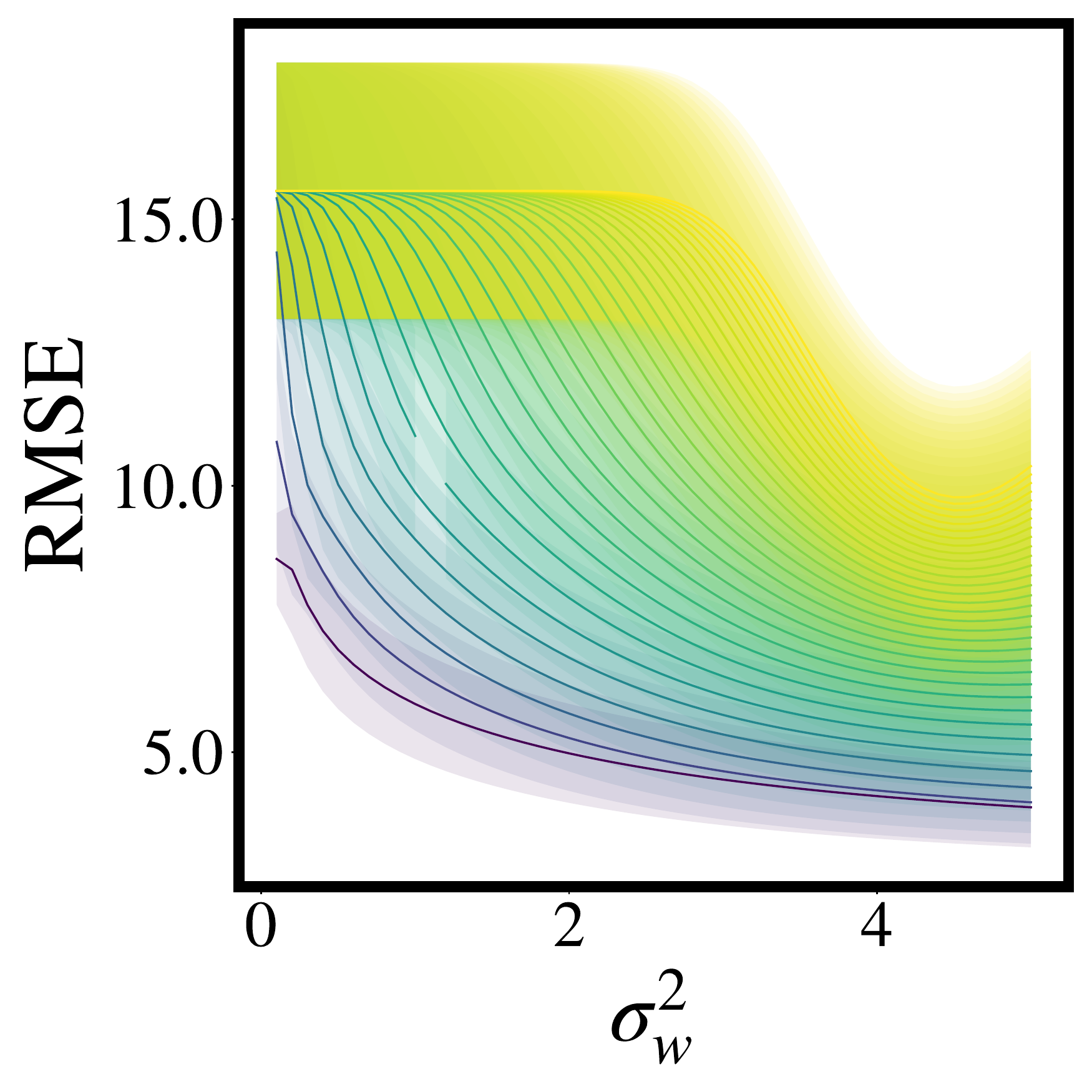}

\includegraphics[scale=0.22]{figures/deep_plots/wine/ReLUWine/rmse_plot}
\includegraphics[scale=0.22]{figures/deep_plots/wine/GELUWine/rmse_plot}
\includegraphics[scale=0.22]{figures/deep_plots/wine/LReLUWine/rmse_plot}
\includegraphics[scale=0.22]{figures/deep_plots/wine/ERFWine/rmse_plot}
\end{center}
\caption{RMSE as depth and $\sigma_w^2$ varies. Legend as in Figure~\ref{fig:bench_regression_deep}. Note that as in the main paper, we plot $\pm 1$ standard deviation using shaded regions, and thus the shaded region can be less than zero. The solid curve in the centre of the shaded region never falls below zero. (Top - Bottom) Boston, Concrete, Energy, Yacht, Wine. (Left - Right) ReLU, GELU, LReLU, ERF.}
\label{fig:deep_all}
\end{figure*}

\newpage
\section{Chain rule for infinitely wide networks}
\label{app:chainrule}
To propagate the gradients through each layer, we will find it more convenient to work with the states $\big((s_1^{(l)})^2, (s_2^{(l)})^2, k^{(l)} \big)$ instead of the states $\big((s_1^{(l)})^2, (s_2^{(l)})^2, \rho^{(l)} \big)$ used in the main text of paper. This just amounts to rescaling the third dimension of the state,
$$ k^{(l)} = \rho^{(l)} s_1^{(l)} s_2^{(l)}.$$

The Jacobian is given by
$$J^{(l+1)} = \begin{bmatrix}
    \frac{\partial (s_1^{(l+1)})^2 }{\partial (s_1^{(l)})^2 } & \frac{\partial (s_1^{(l+1)})^2 }{\partial (s_2^{(l)})^2 } & \frac{\partial (s_1^{(l+1)})^2}{\partial k^{(l)} } \\ 
    \frac{\partial (s_2^{(l+1)})^2 }{\partial (s_1^{(l)})^2 } & \frac{\partial (s_2^{(l+1)})^2 }{\partial (s_2^{(l)})^2 } & \frac{\partial (s_2^{(l+1)})^2}{\partial k^{(l)} } \\ 
    \frac{\partial k^{(l+1)} }{\partial (s_1^{(l)})^2 } & \frac{\partial k^{(l+1)}  }{\partial (s_2^{(l)})^2 } & \frac{\partial k^{(l+1)} }{\partial k^{(l)} } 
\end{bmatrix} = \begin{bmatrix}
    \lambda_1^{(l+1)}  & 0 & 0 \\ 
    0 & \lambda_2^{(l+1)} & 0 \\ 
    \frac{\partial k^{(l+1)} }{\partial (s_1^{(l)})^2 } & \frac{\partial k^{(l+1)}  }{\partial (s_2^{(l)})^2 } & \lambda_3^{(l+1)}\frac{s_1^{(l+1)}s_2^{(l+1)}}{s_1^{(l)}s_2^{(l)}}
\end{bmatrix},$$ 
where $\lambda_i^{(l+1)}$ for $i=1,2,3$ is the quantity obtained by applying Theorem~\ref{thm:jacobian} to layer $l+1$. Given the Jacobians in each layer, the chain rule says that in an $L$-hidden layer network, for $\sigma^{(l)}=\sigma^{(l)}_w$ or $\sigma^{(l)}=\sigma_b^{(l)}$ with $l<L+1$ (the case $l=L+1$ is trivial),
\begin{equation} \frac{\partial k^{(L+1)}}{\partial (\sigma^{(l)})^2} =  \Bigg( \big( J^{(L+1)}... J^{(l+1)}\big)  \frac{\partial \mathbf{s}^{(l)} }{\partial (\sigma^{(l)})^2  } \Bigg)_3
\end{equation}
where
\begin{align*}
\frac{\partial \mathbf{s}^{(l)} }{\partial (\sigma^{(l)})^2  } &= \begin{bmatrix}
   \partial\big( (s_1^{(l)})^2 \big)/\partial \big( (\sigma^{(l)})^2 \big) \\
   \partial\big( (s_2^{(l)})^2 \big)/\partial \big( (\sigma^{(l)})^2  \big) \\
   \partial\big( k^{(l)} \big)/\partial \big( (\sigma^{(l)})^2 \big) 
\end{bmatrix}
\end{align*}
is easily obtained from~\eqref{eq:g_iterated}. Since the Jacobian is lower triangular, Theorem~\ref{thm:jacobian} together with $\frac{\partial (s_i^{(l+1)})^2}{\partial k^{(l)}}=0$ for $i=1,2$ provides $7$ out of the $9$ elements of the Jacobian. The unknown elements are $\frac{\partial k^{(l+1)}}{\partial (s_i^{(l)})^2}$ for $i=1, 2$. For special cases, this is straight-forward to evaluate but a more general expression in the vein of Theorem~\ref{thm:jacobian} is currently elusive. As an example, when $\psi$ is ReLU, due to absolute homogeneity and noting that $k^{(l+1)} = k^{(l+1)} s_1^{(l+1)} s_2^{(l+1)}$, we have that
\begin{align*}
     \frac{\partial k^{(l+1)} }{\partial s_1^{(l)} } &= \frac{ k^{(l+1)} - (\sigma_b^{(l+1)})^2 }{s_1^{(l)} } \\
     \frac{\partial k^{(l+1)} }{\partial (s_1^{(l)})^2 } &= \frac{ k^{(l+1)} - (\sigma_b^{(l+1)})^2 }{2 (s_1^{(l)})^2 }.
\end{align*}

\newpage
\section{Overfitting and underfitting curves}
\label{app:train_test_curves}
\begin{figure}[h]
    \centering
    \includegraphics[scale=0.39]{figures/test_train/0gelutrain_test.pdf}
    \includegraphics[scale=0.39]{figures/test_train/0relutrain_test.pdf} \\ 
    \includegraphics[scale=0.39]{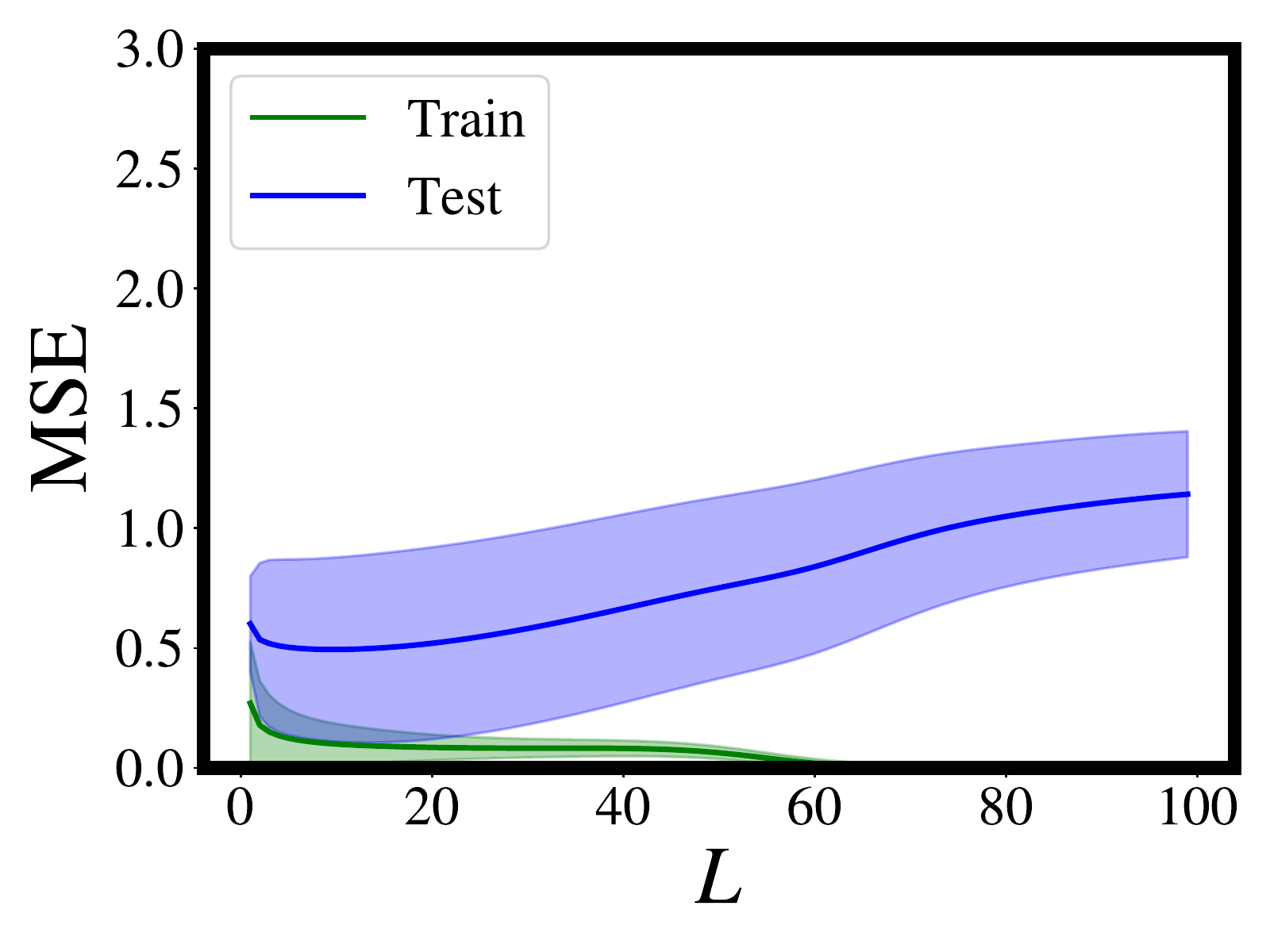}
    \includegraphics[scale=0.39]{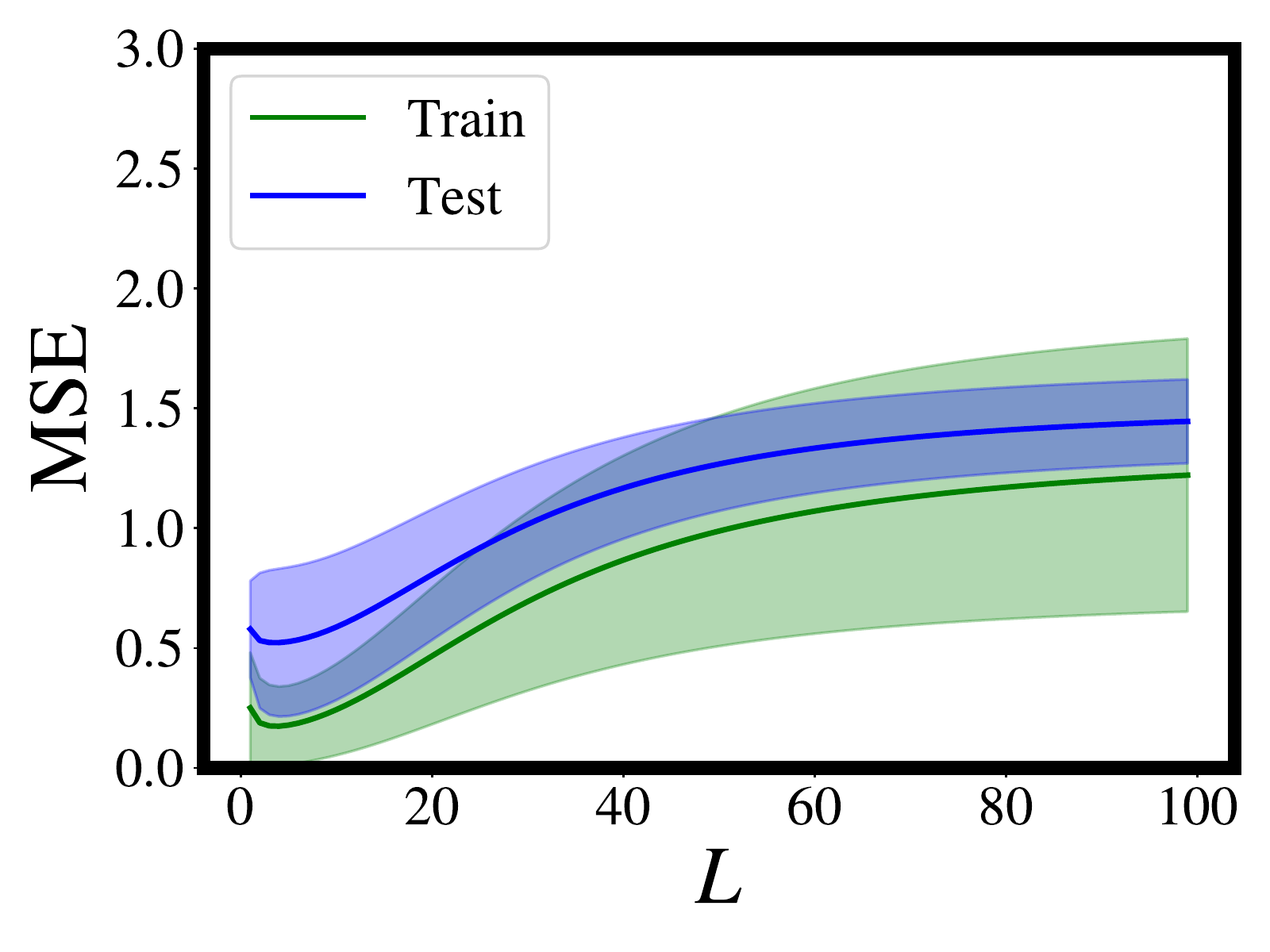} \\ 
    \includegraphics[scale=0.39]{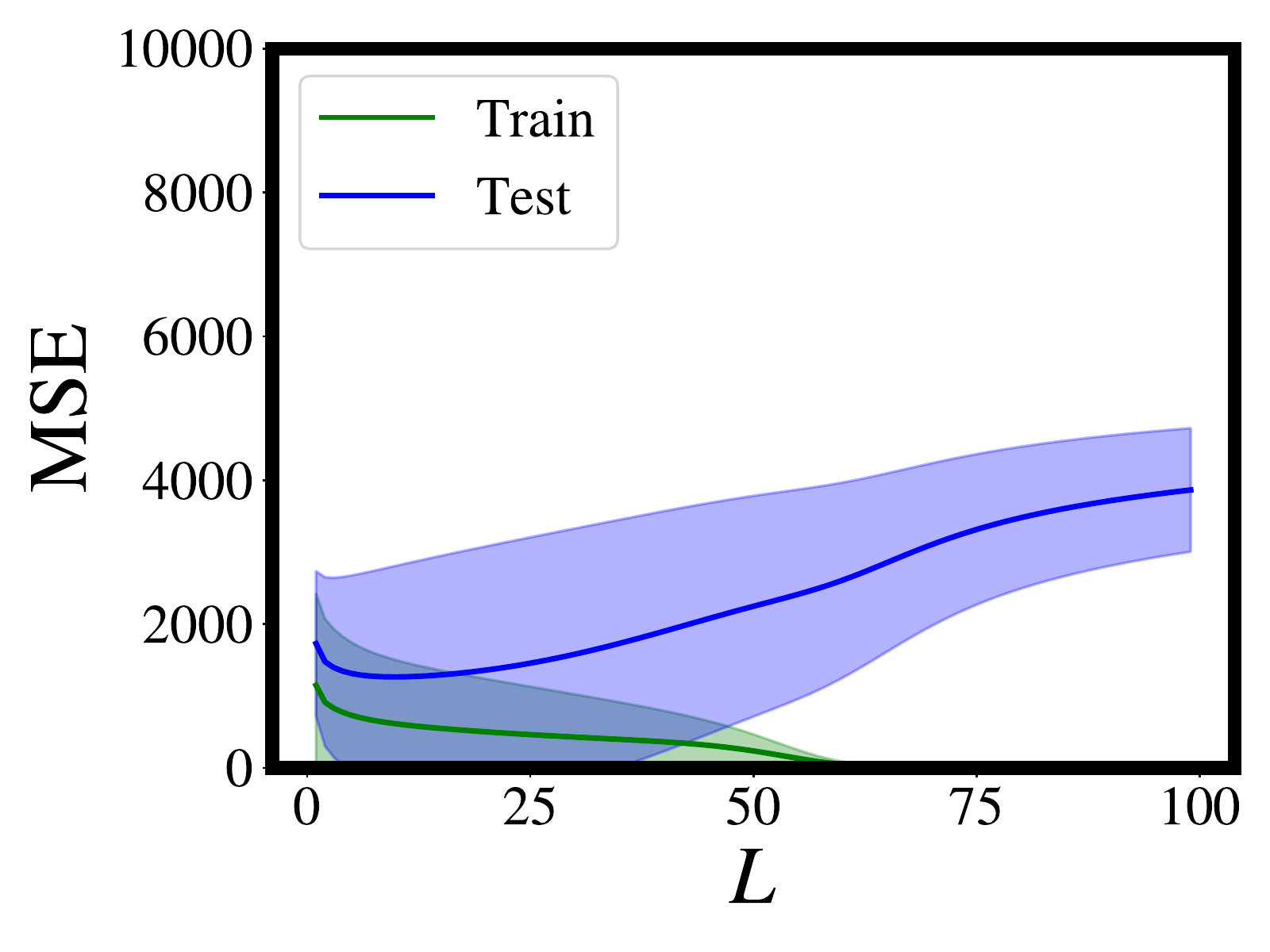}
    \includegraphics[scale=0.39]{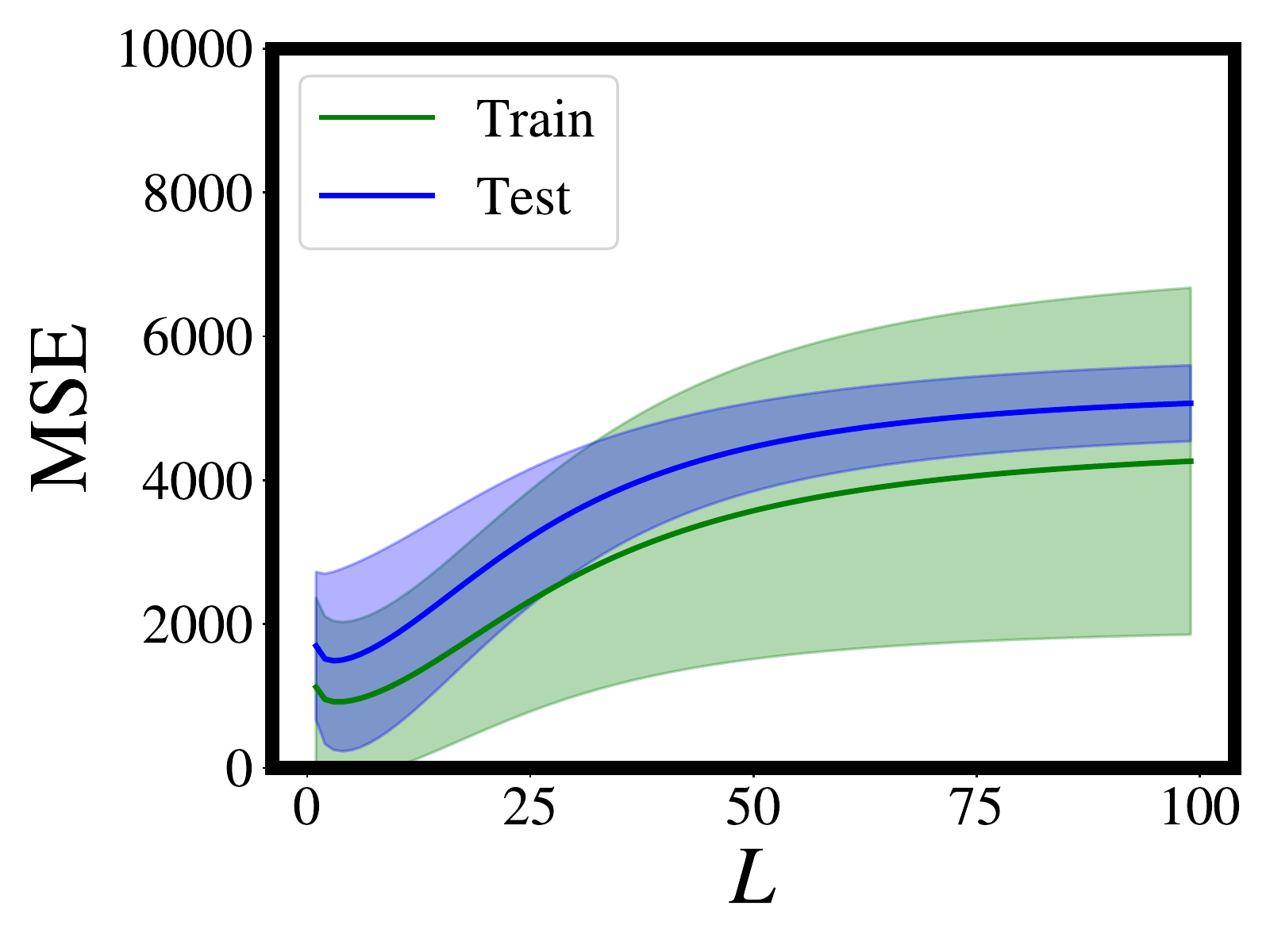}
    \caption{Training and testing errors for Gaussian processes with covariance functions corresponding to infinitely wide MLPs of increasing depth. Solid curve shows the mean over $10$ training data samples, and the shaded region shows $\pm$ two standard deviations. (Left) GELU (Right) ReLU. (Top - Bottom) $f(\gamma)= \sin(\gamma)$, $f(\gamma)=\text{saw}(\gamma)$, $f(\gamma)=\gamma^3-4$. Continues over page...}
    \label{fig:test_train_overfit_app}
\end{figure}

\renewcommand{\thefigure}{\arabic{figure} (Cont.)}
\addtocounter{figure}{-1}
\begin{figure}
    \centering
    \includegraphics[scale=0.4]{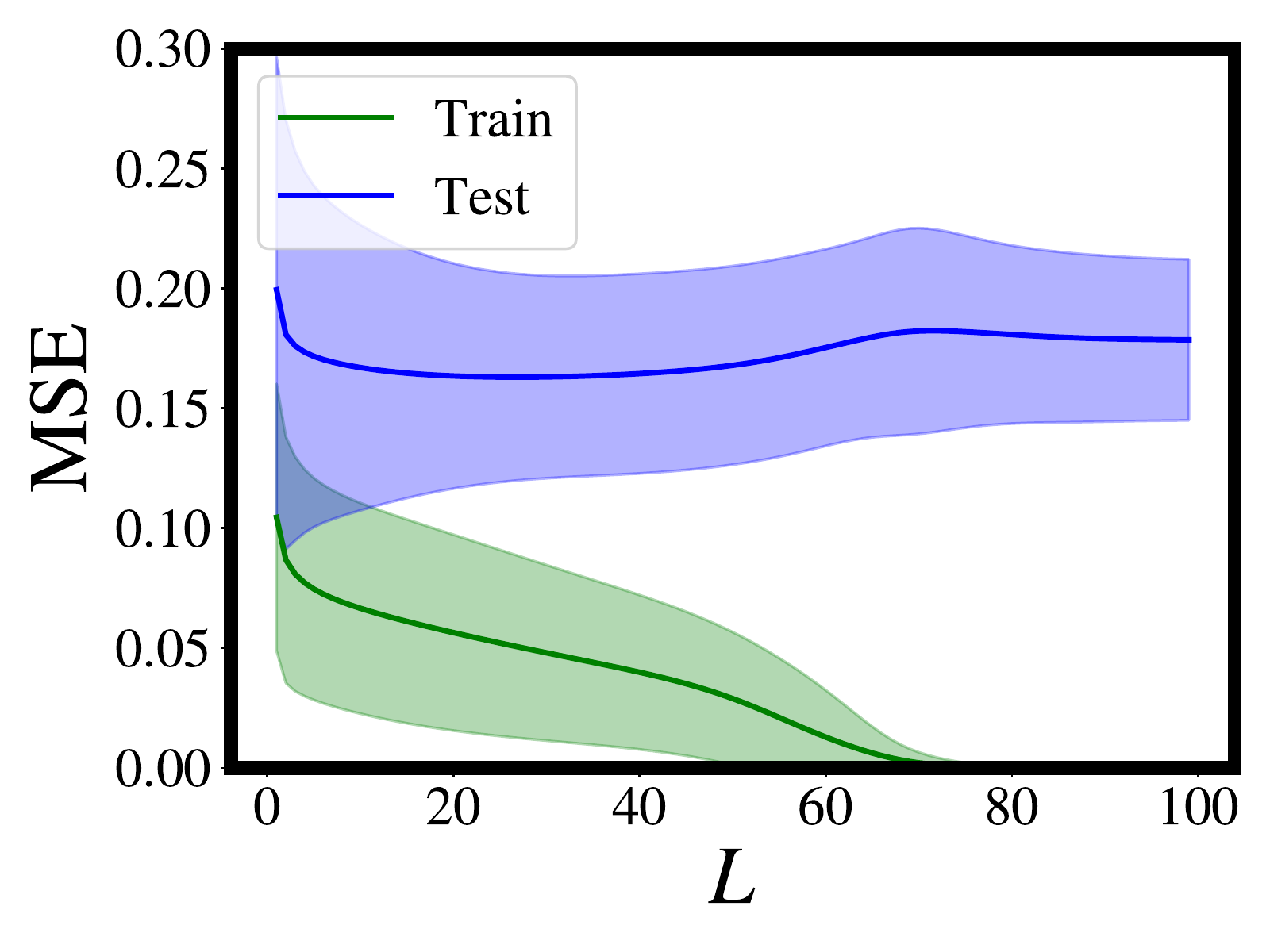}
    \includegraphics[scale=0.4]{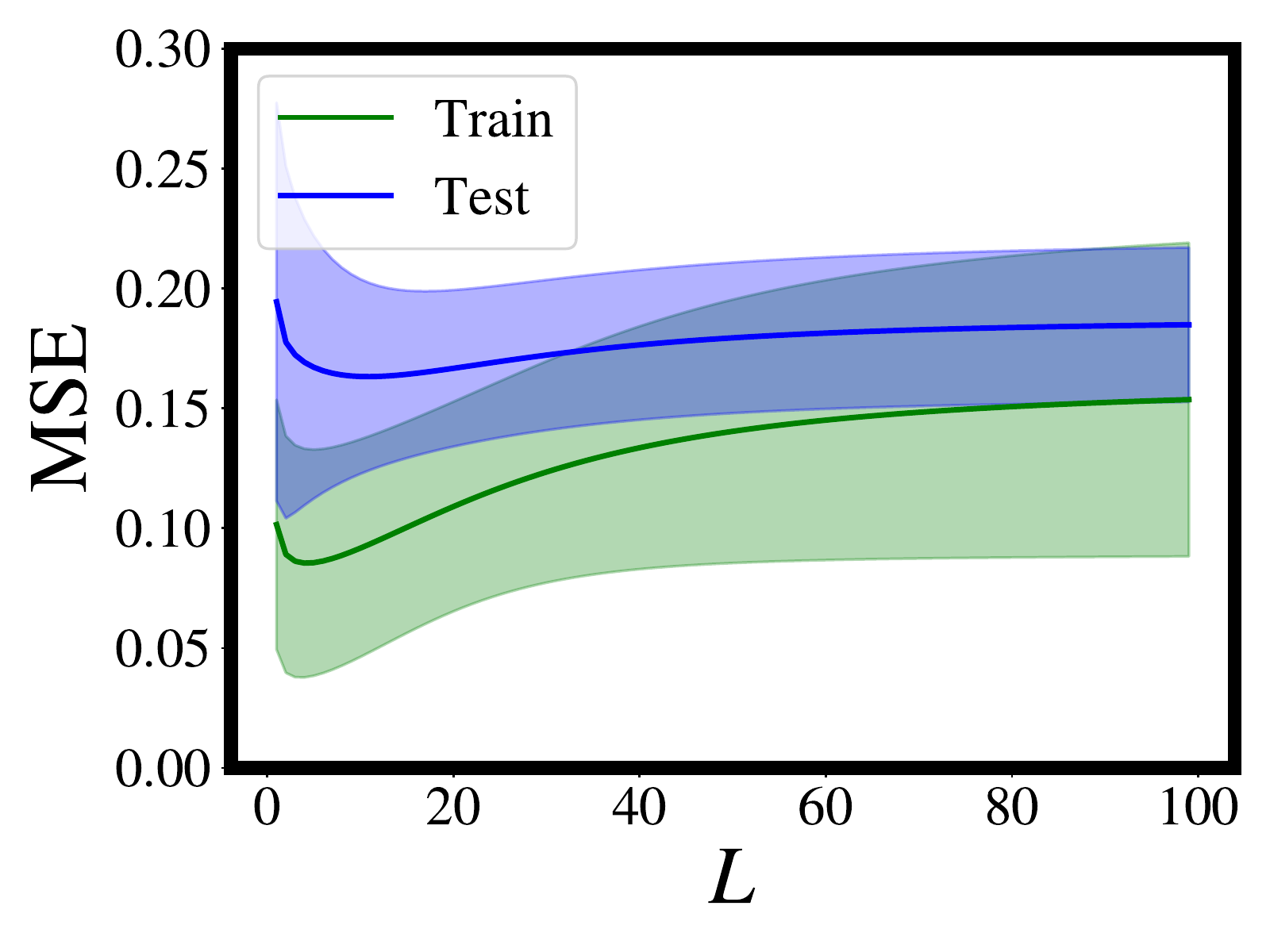} \\ 
    \includegraphics[scale=0.4]{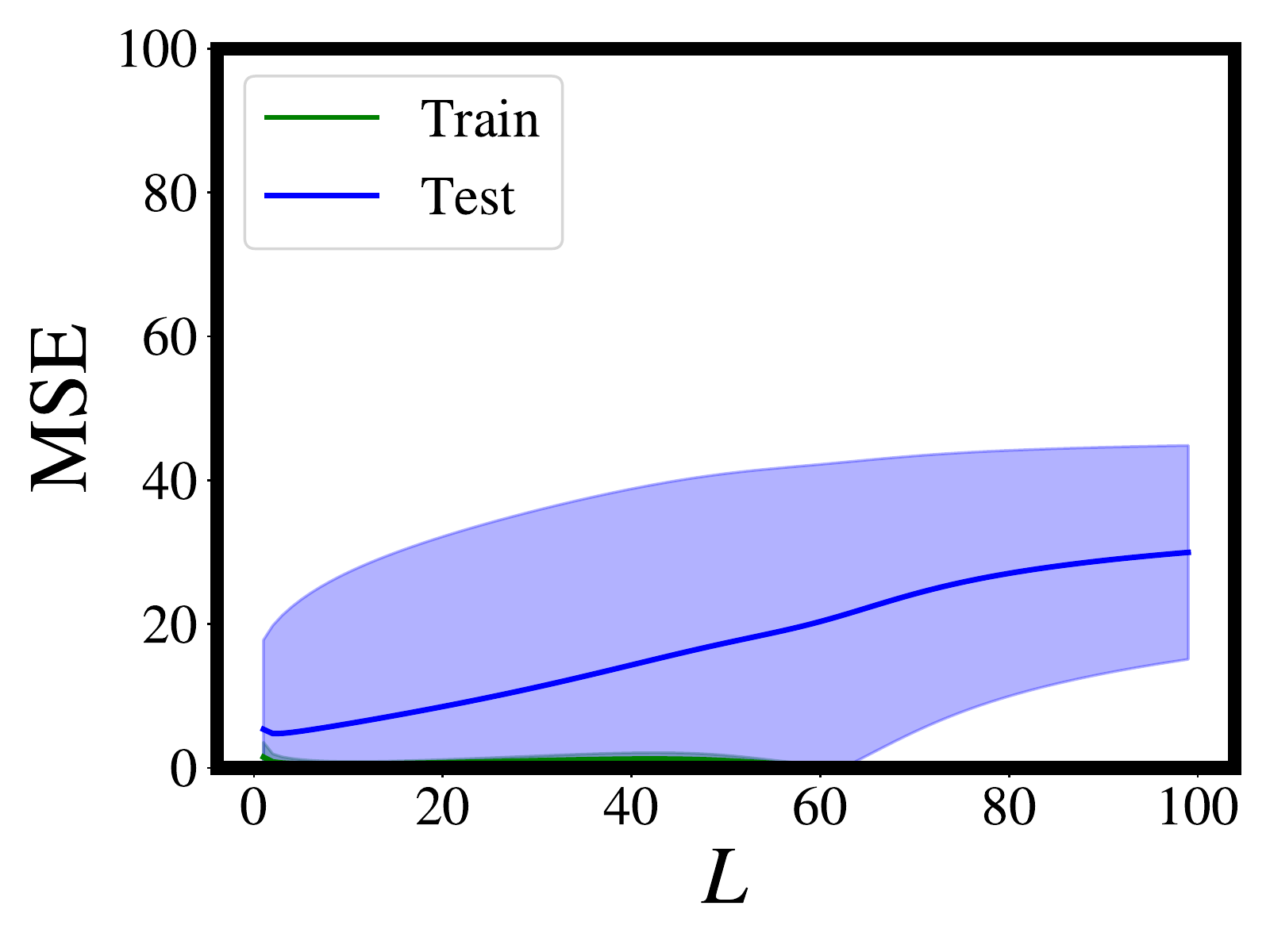}
    \includegraphics[scale=0.4]{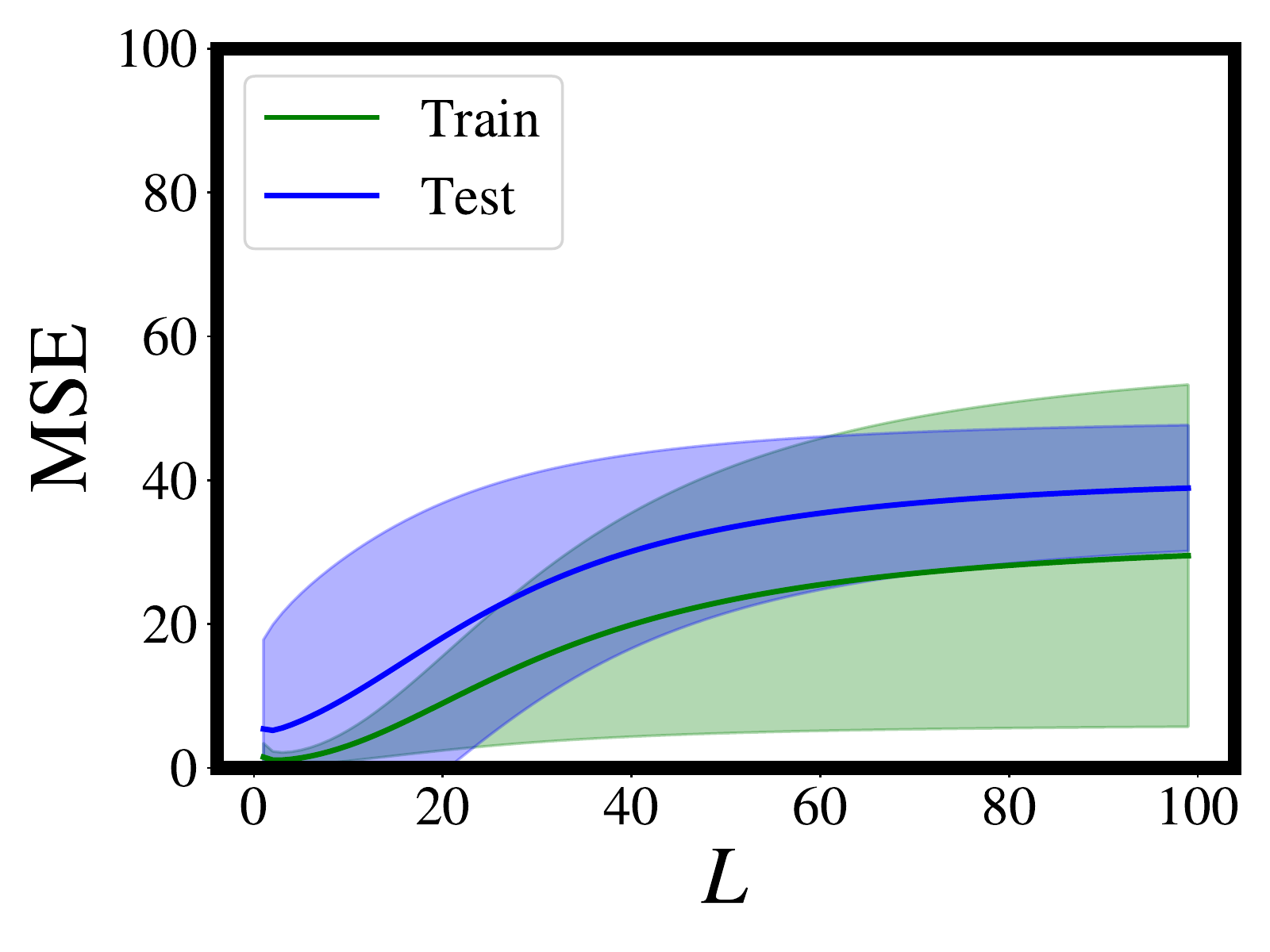} \\ 
    \includegraphics[scale=0.4]{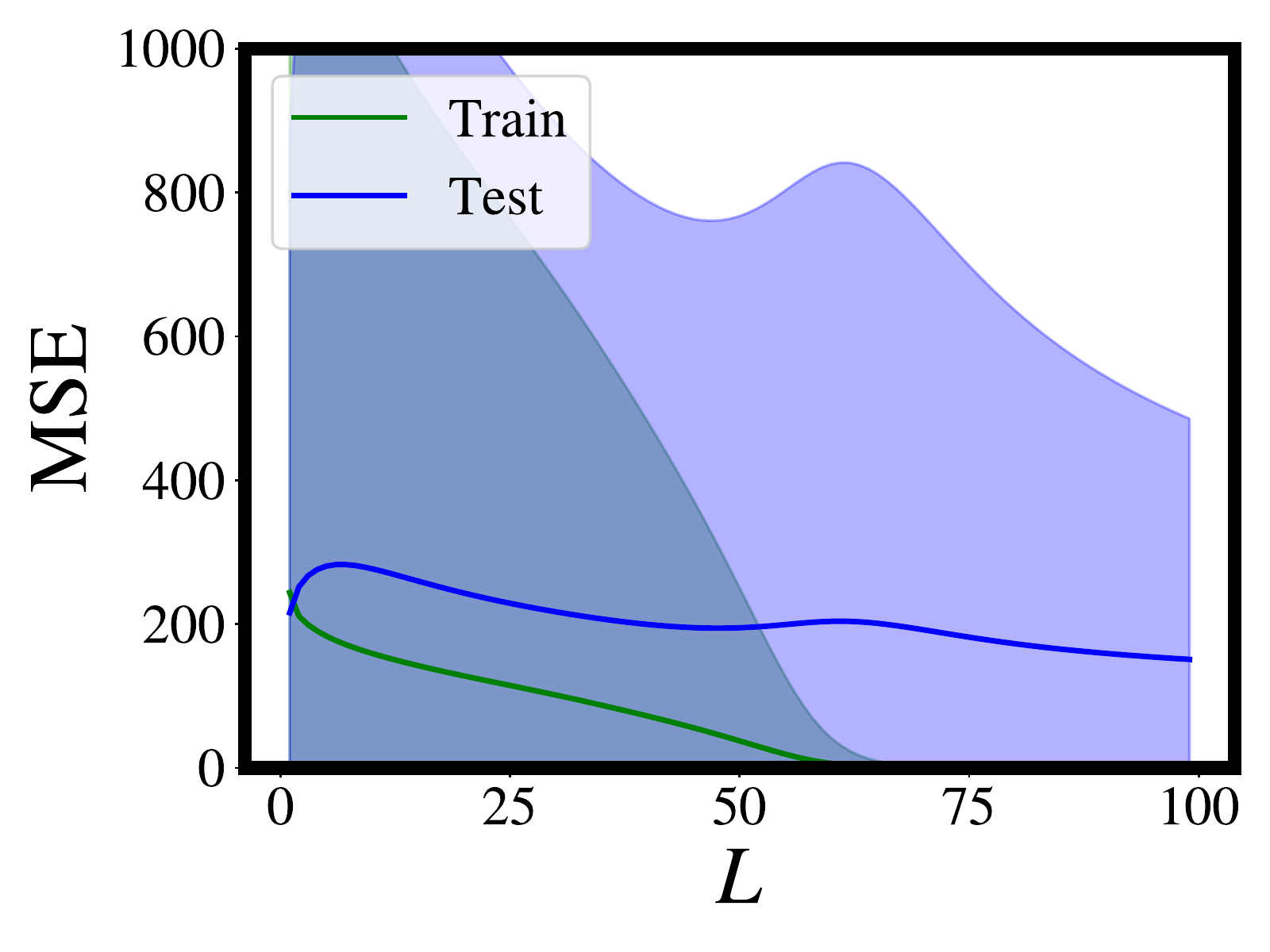}
    \includegraphics[scale=0.4]{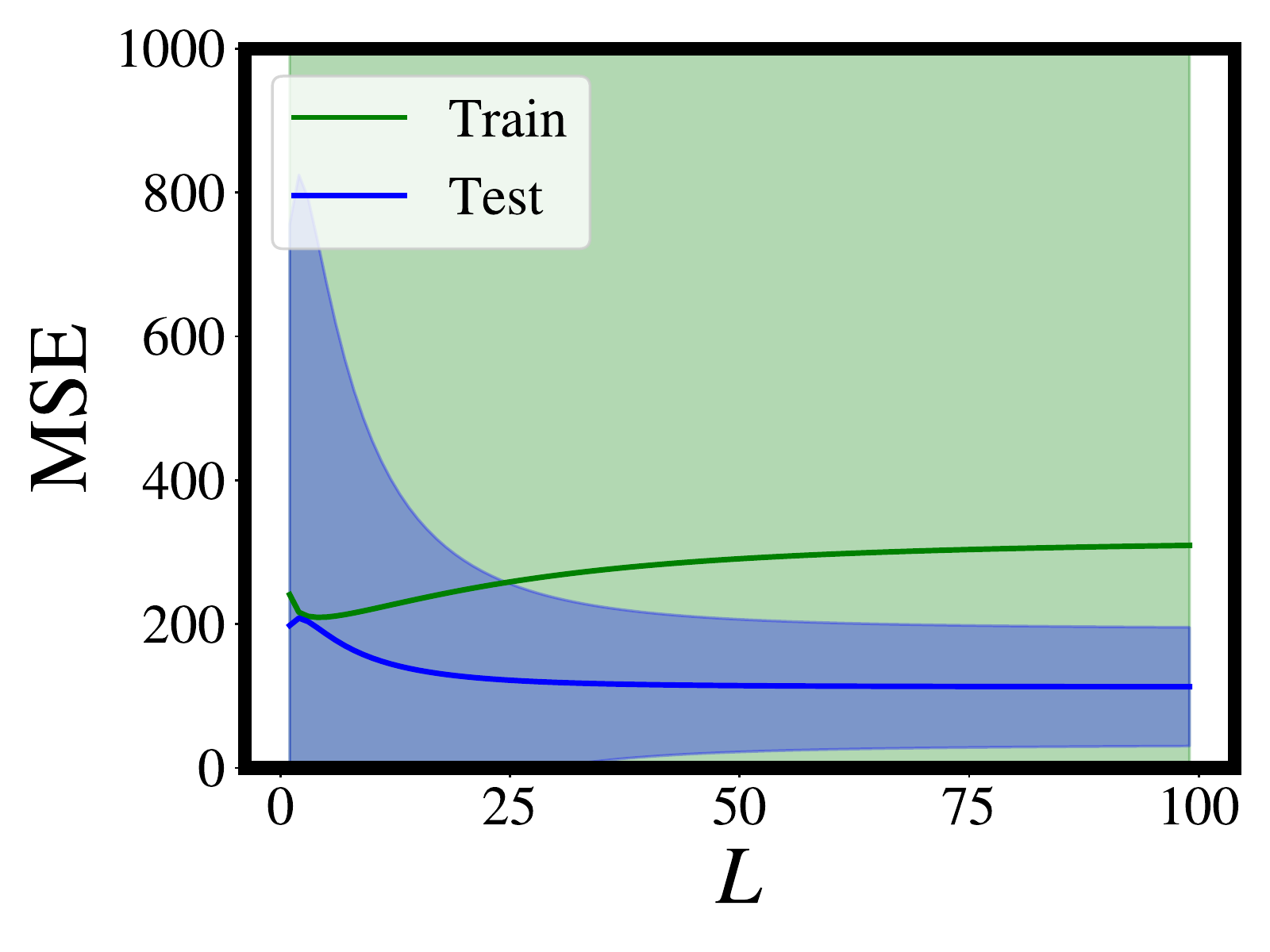} 
    \caption{Training and testing errors for Gaussian processes with covariance functions corresponding to infinitely wide MLPs of increasing depth. Solid curve shows the mean over $10$ training data samples, and the shaded region shows $\pm$ two standard deviations. (Left) GELU (Right) ReLU. (Top - Bottom) $f(\gamma)= \text{sinc}(\gamma)$, $f(\gamma)=\exp |\gamma - \pi |$, $f(\gamma)=\tan \gamma$}
\end{figure}
\renewcommand{\thefigure}{\arabic{figure}}

\newpage
\section{Simplicity bias illustrations}
\label{app:simplicity_bias}
\begin{figure}[!ht]
    \centering
    \begin{alignat*}{2}
    \begin{array}{l}
    \includegraphics[scale=0.3]{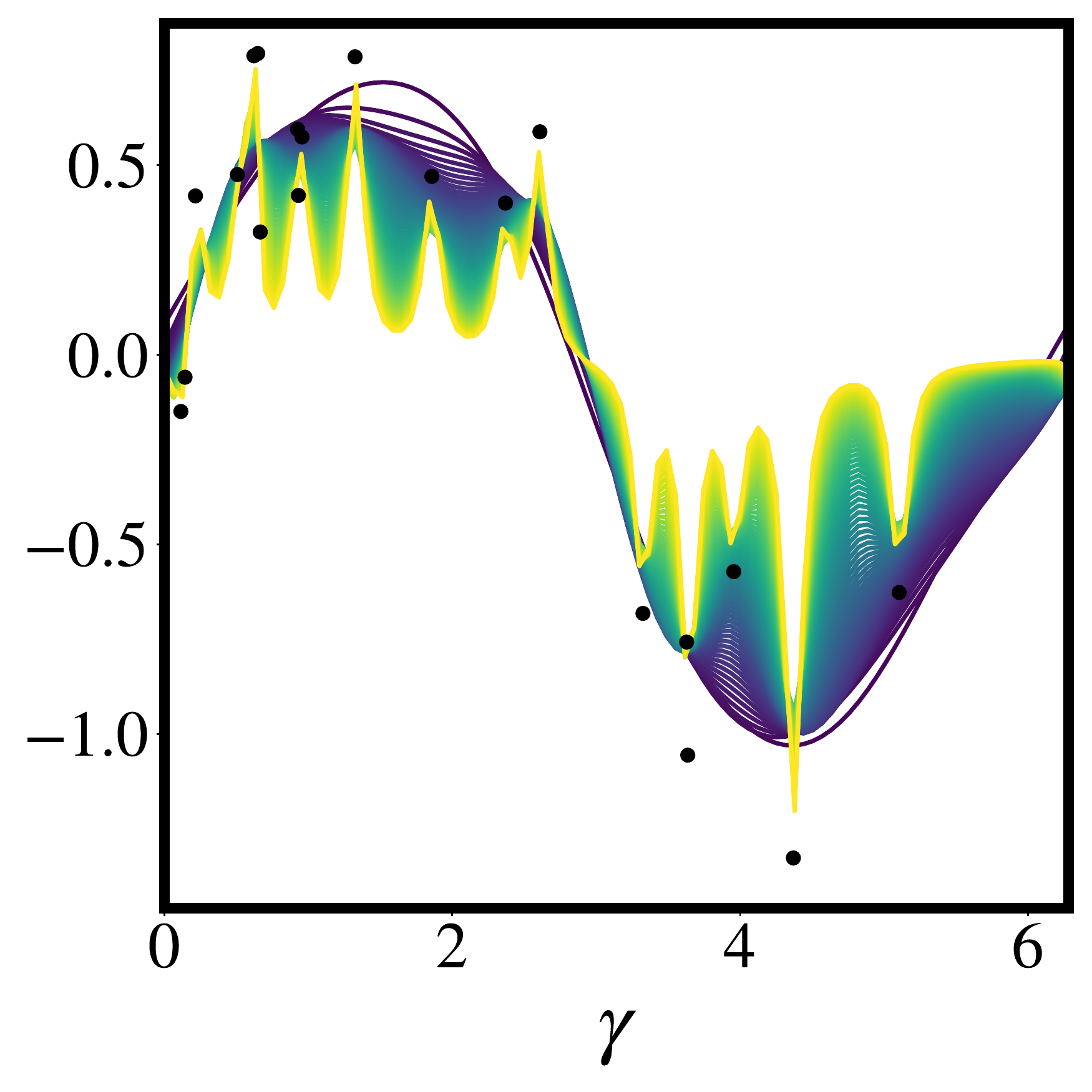}
    \includegraphics[scale=0.3]{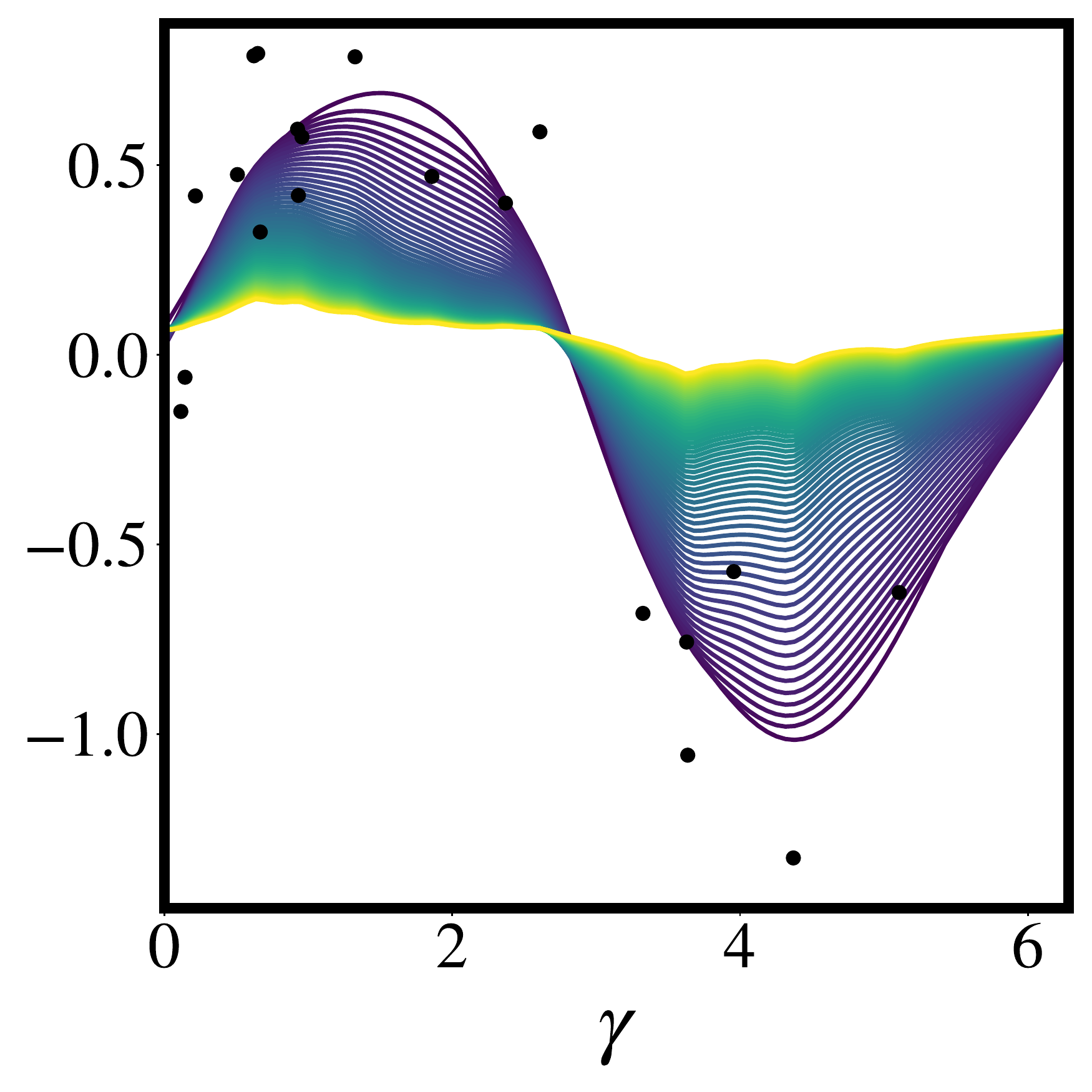} 
    \end{array} &f(\gamma) &&= \sin\gamma \\
    \begin{array}{l}
    \includegraphics[scale=0.3]{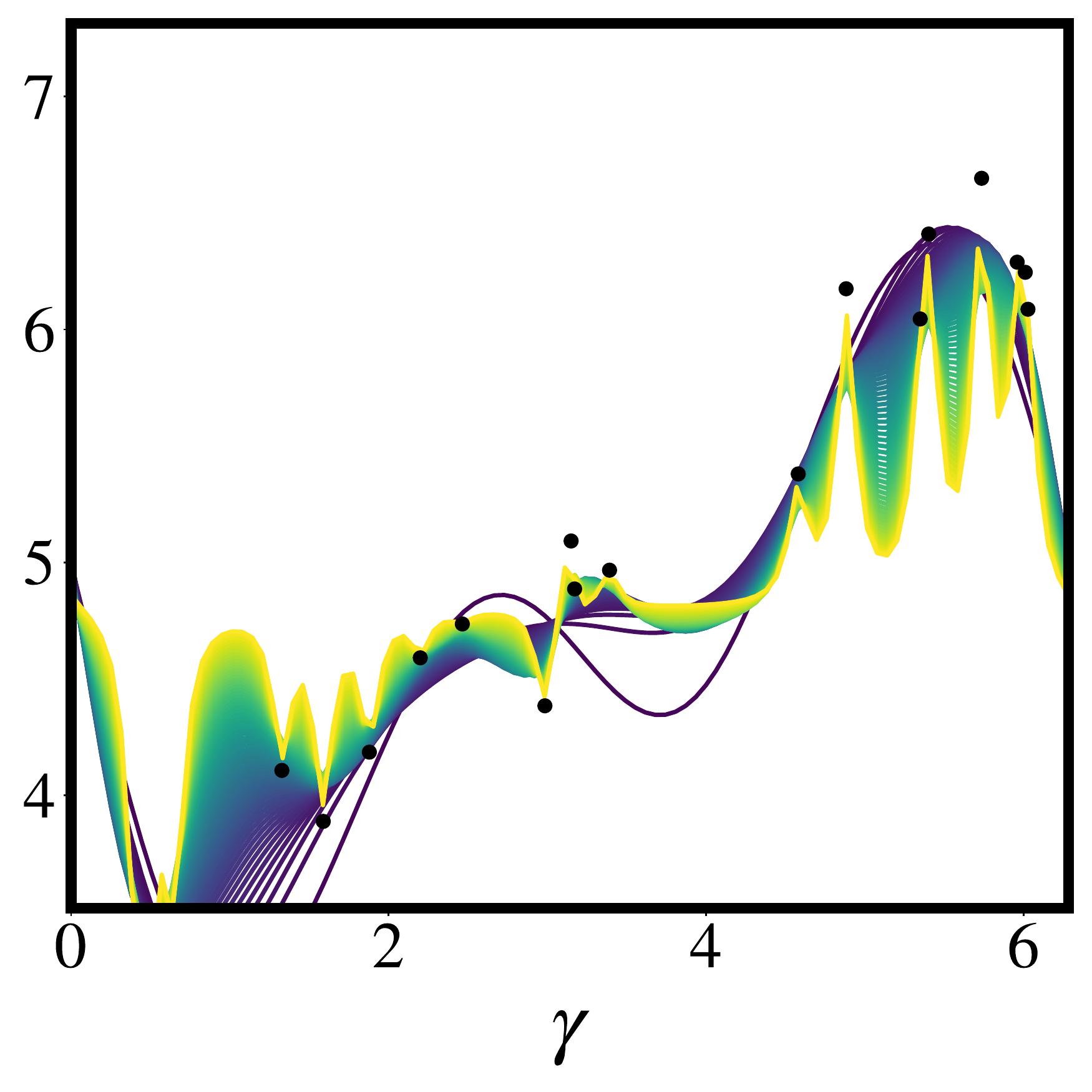}
    \includegraphics[scale=0.3]{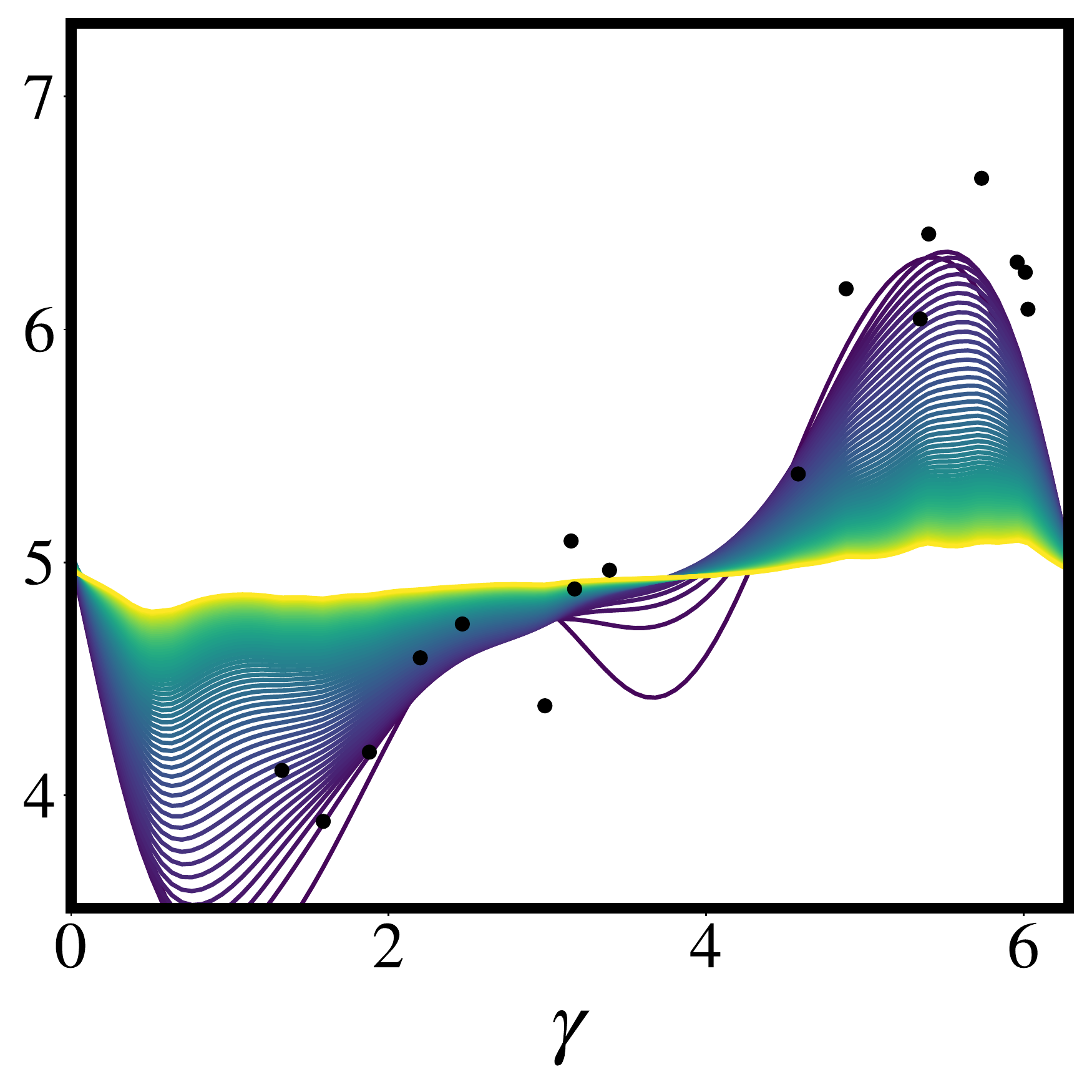} 
    \end{array} &f(\gamma) &&= 2\text{saw}(\gamma) + 5\\
    \begin{array}{l}
    \includegraphics[scale=0.3]{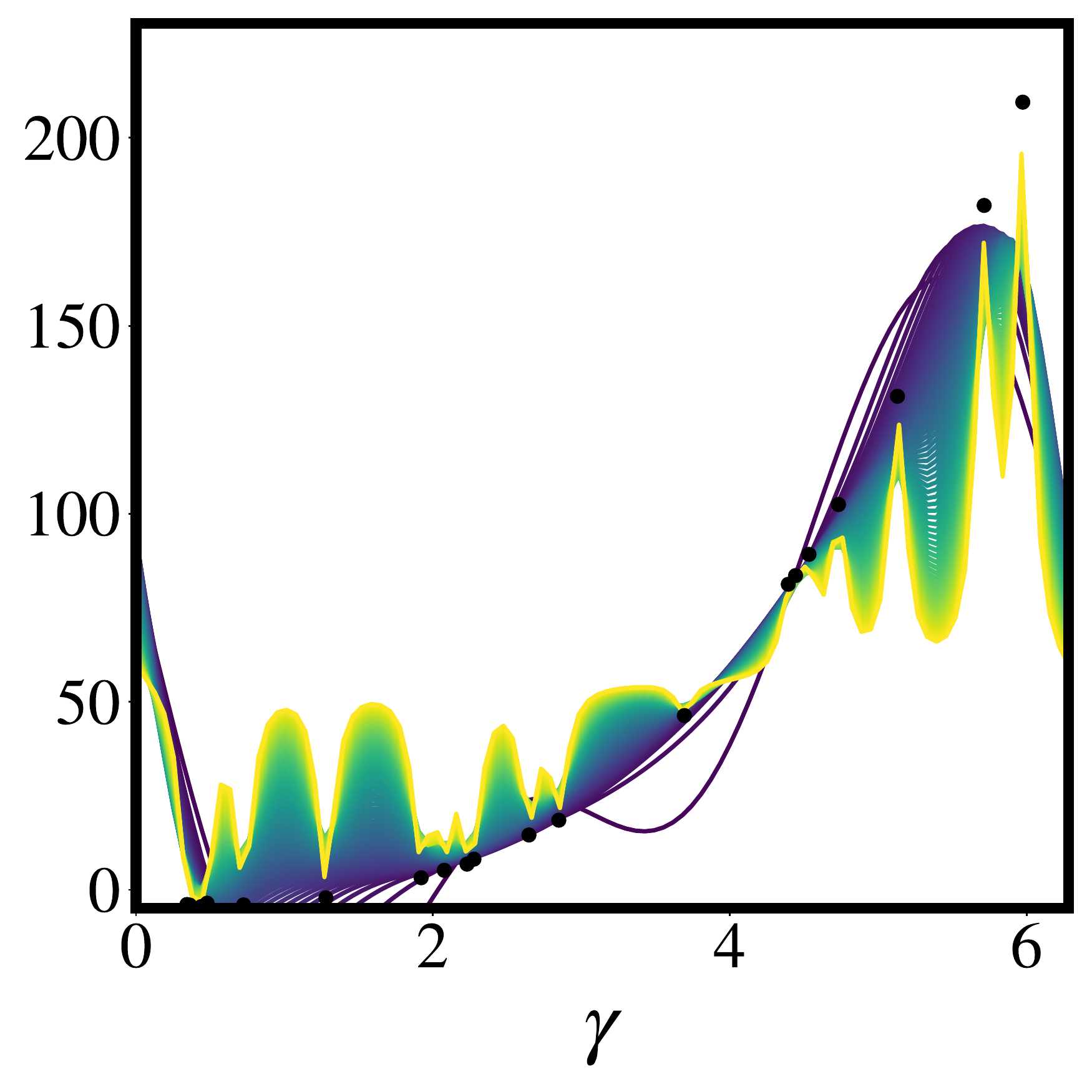}
    \includegraphics[scale=0.3]{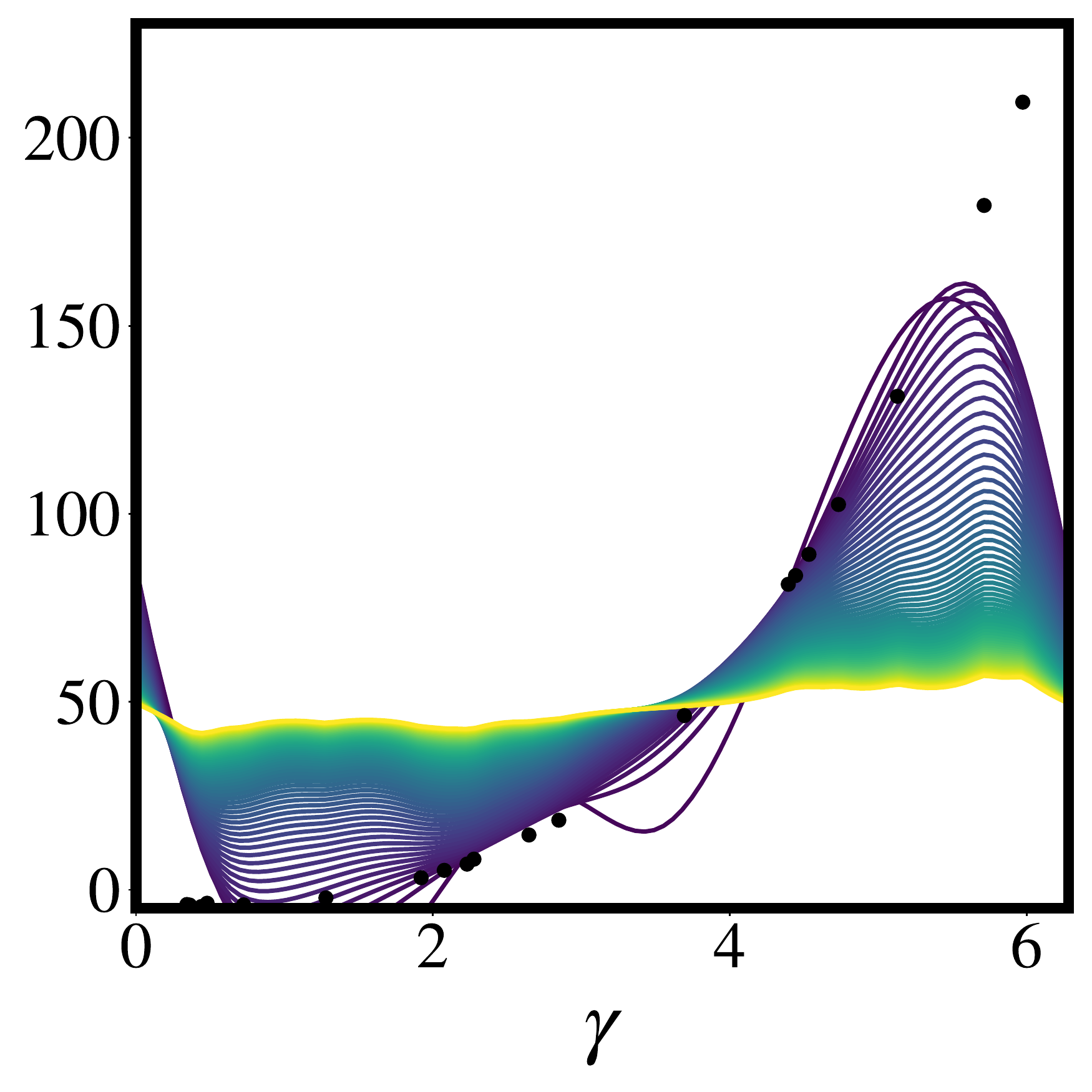} 
    \end{array} &f(\gamma) &&= \gamma^3-4
    \end{alignat*}
    \caption{Illustration of (lack of) simplicity bias due to kernel fixed points. Training data $\mathbf{x} \in \mathbb{R}^2$ is uniformly sampled on the unit disc at heading $\gamma$. Curves show the posterior mean of a GP regression model on $y=f(\gamma) + \epsilon$ with known additive noise variance fixed at $0.1$. $\sigma_w$ is chosen according to Figure 3. Colours move from purple to yellow as depth increases from $1$ to $64$. (Left) GELU without unique kernel fixed point leading to overfitting (Right) ReLU with unique kernel fixed point leading to underfitting. Continues over page...}
    \label{fig:simplicity_app}
\end{figure}

\renewcommand{\thefigure}{\arabic{figure} (Cont.)}
\addtocounter{figure}{-1}
\begin{figure}[!ht]
    \centering
    \begin{alignat*}{2}
    \begin{array}{l}
    \includegraphics[scale=0.3]{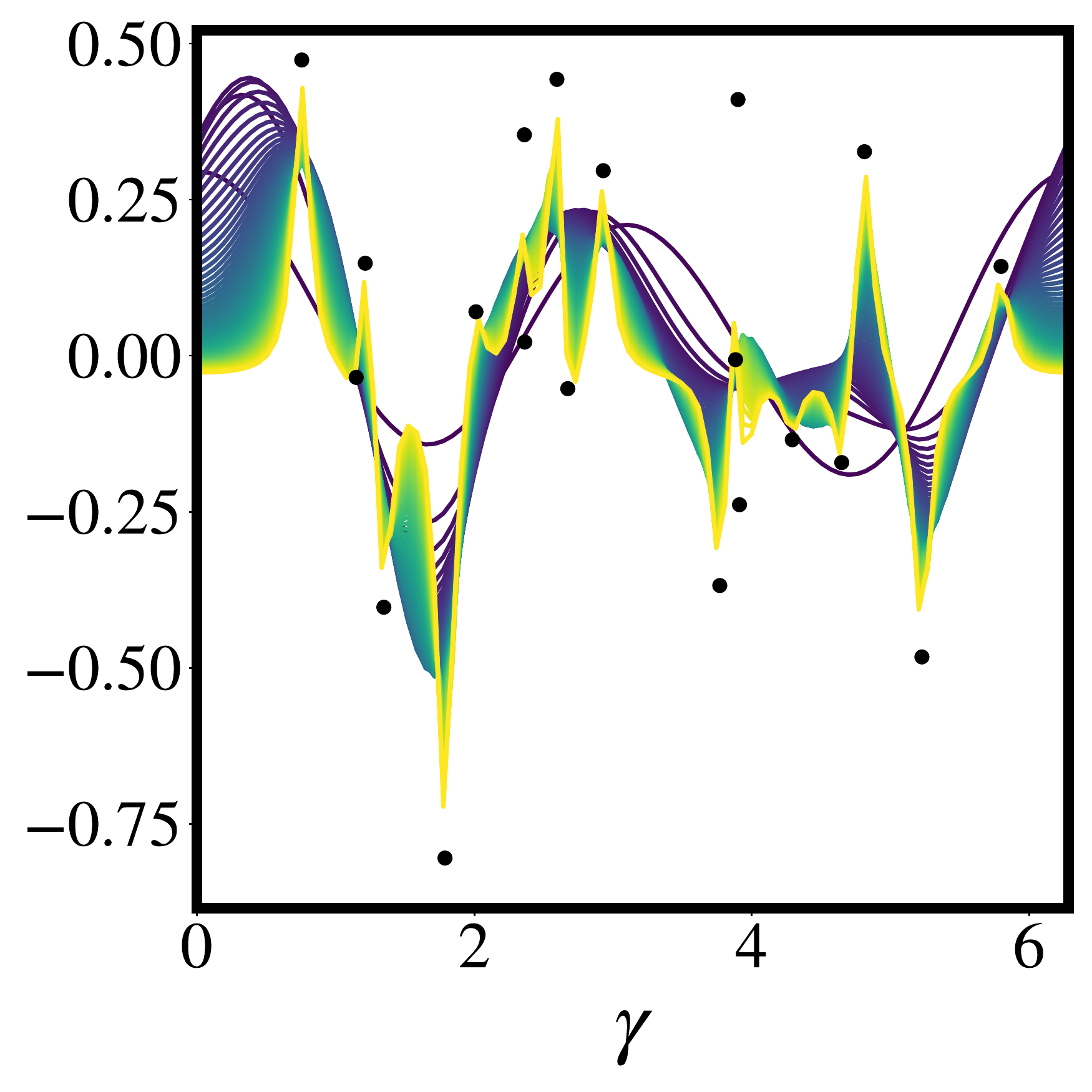}
    \includegraphics[scale=0.3]{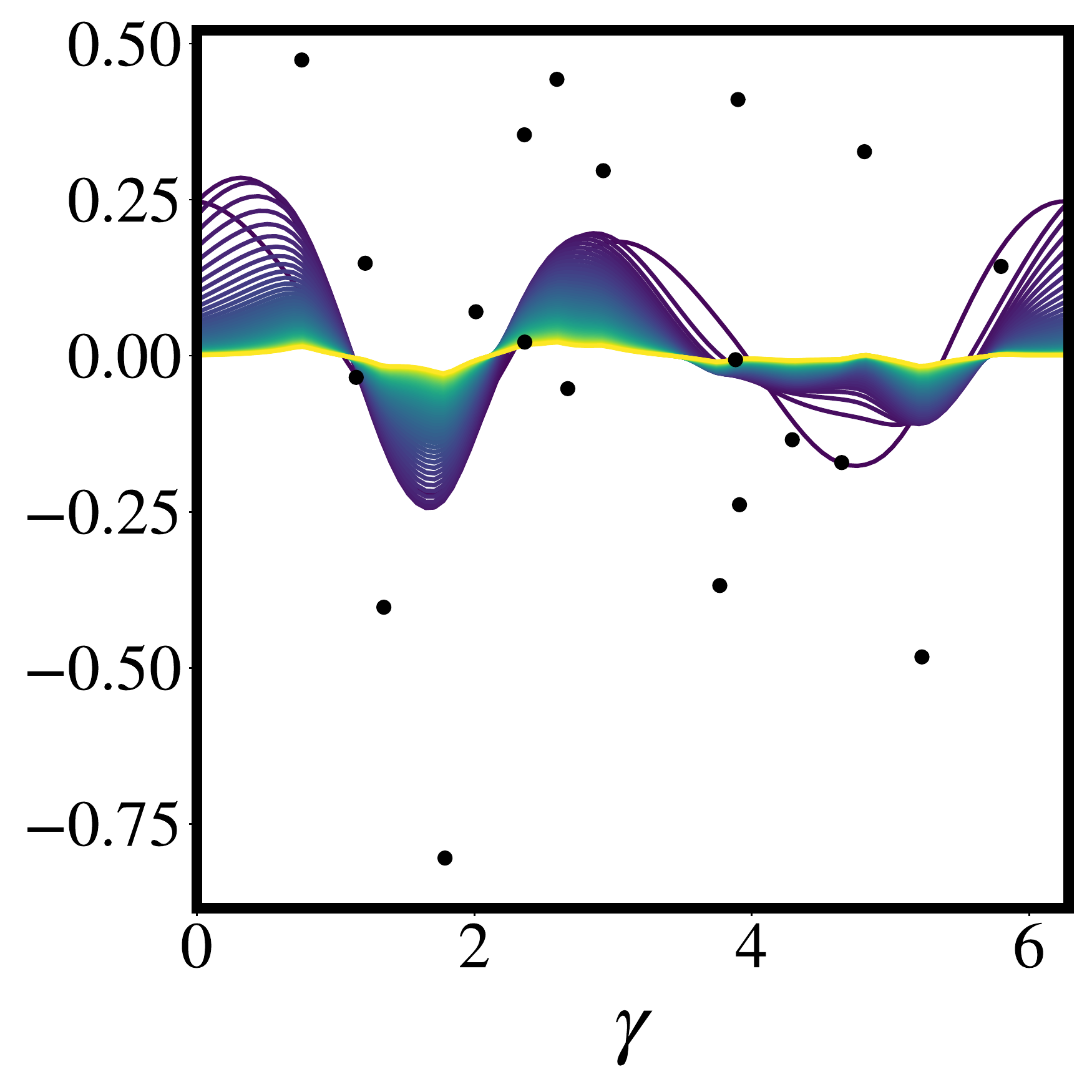} 
    \end{array} &f(\gamma) &&= \text{sinc} \gamma \\
    \begin{array}{l}
    \includegraphics[scale=0.3]{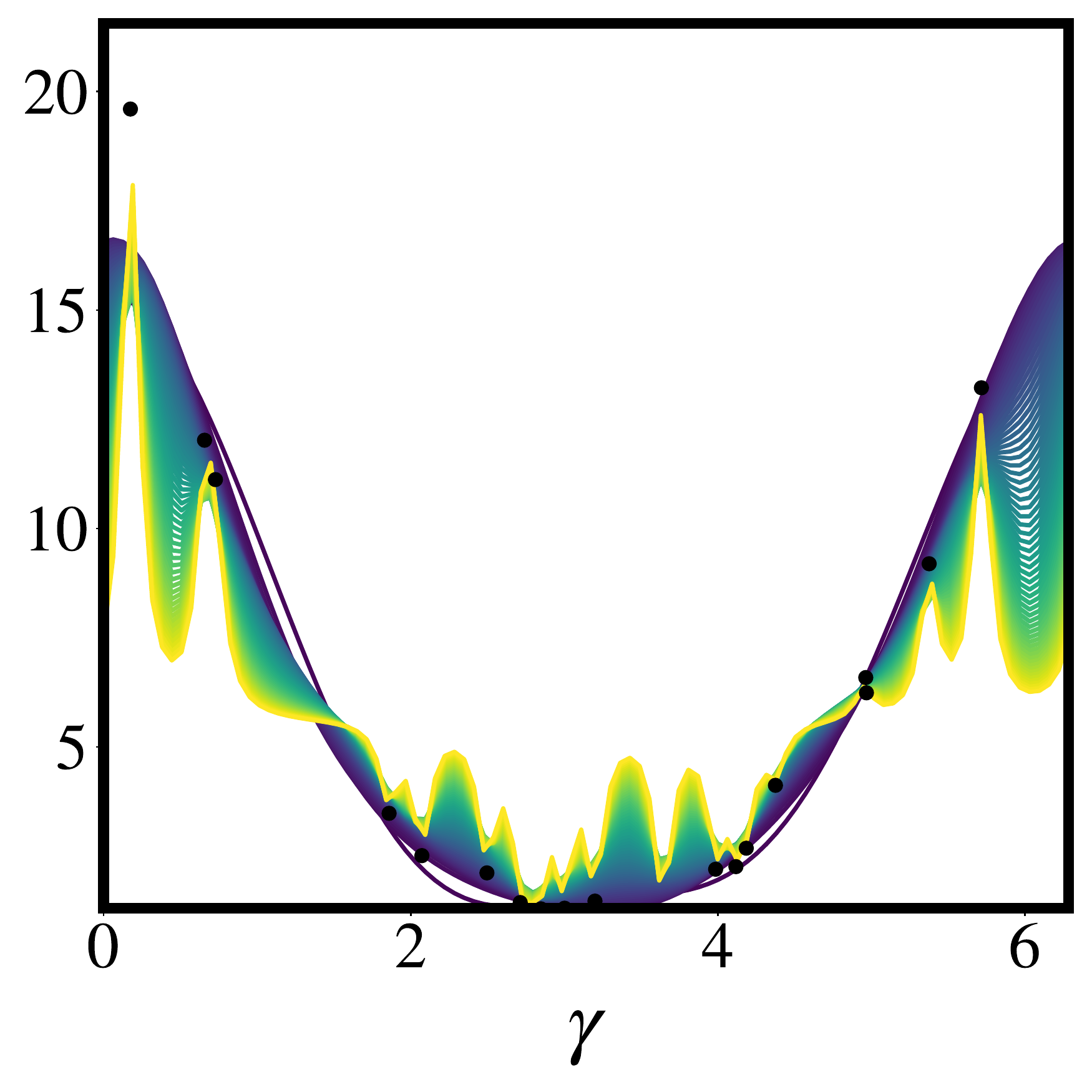}
    \includegraphics[scale=0.3]{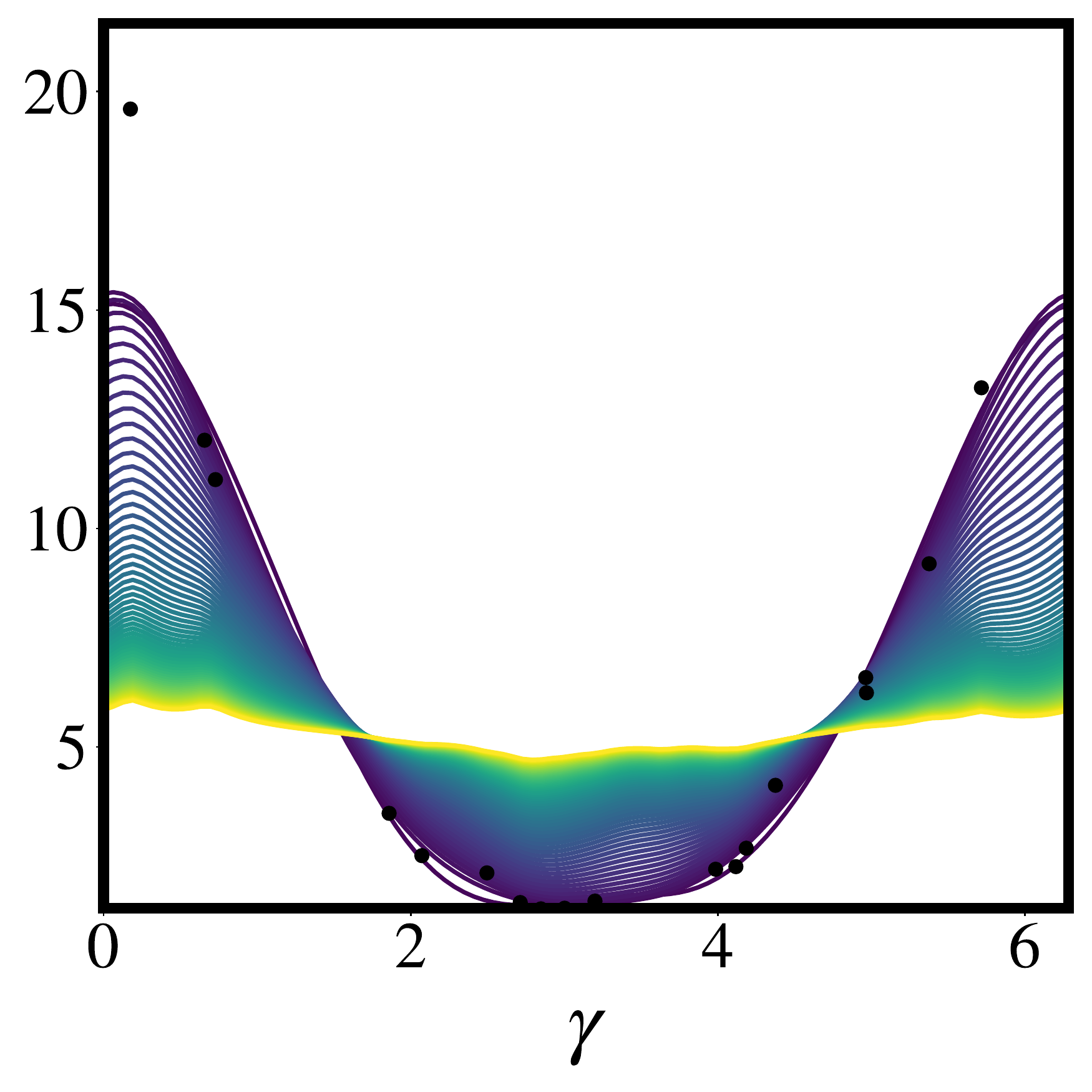} 
    \end{array} &f(\gamma) &&= e^{|\gamma-\pi|} \\
        \begin{array}{l}
    \includegraphics[scale=0.3]{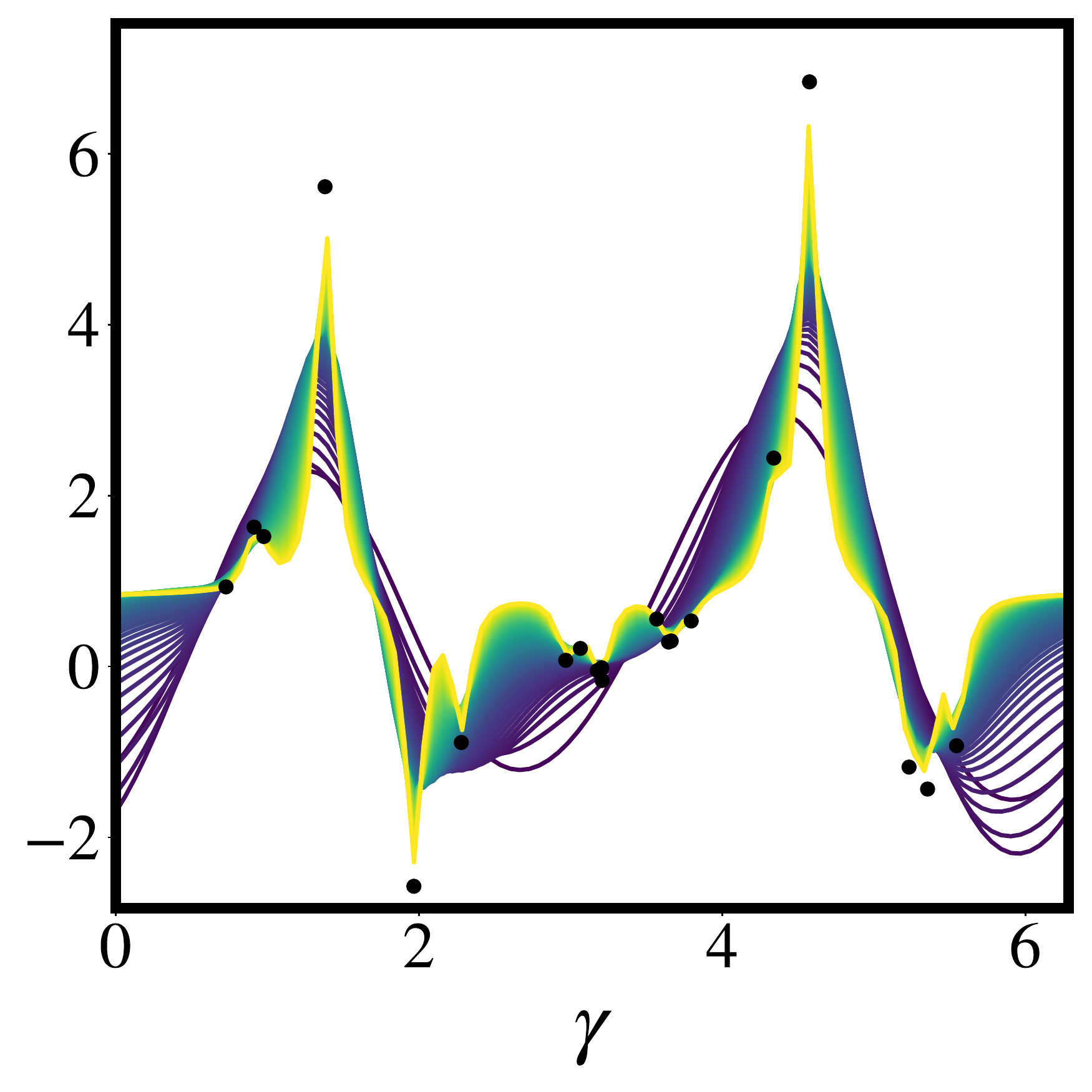}
    \includegraphics[scale=0.3]{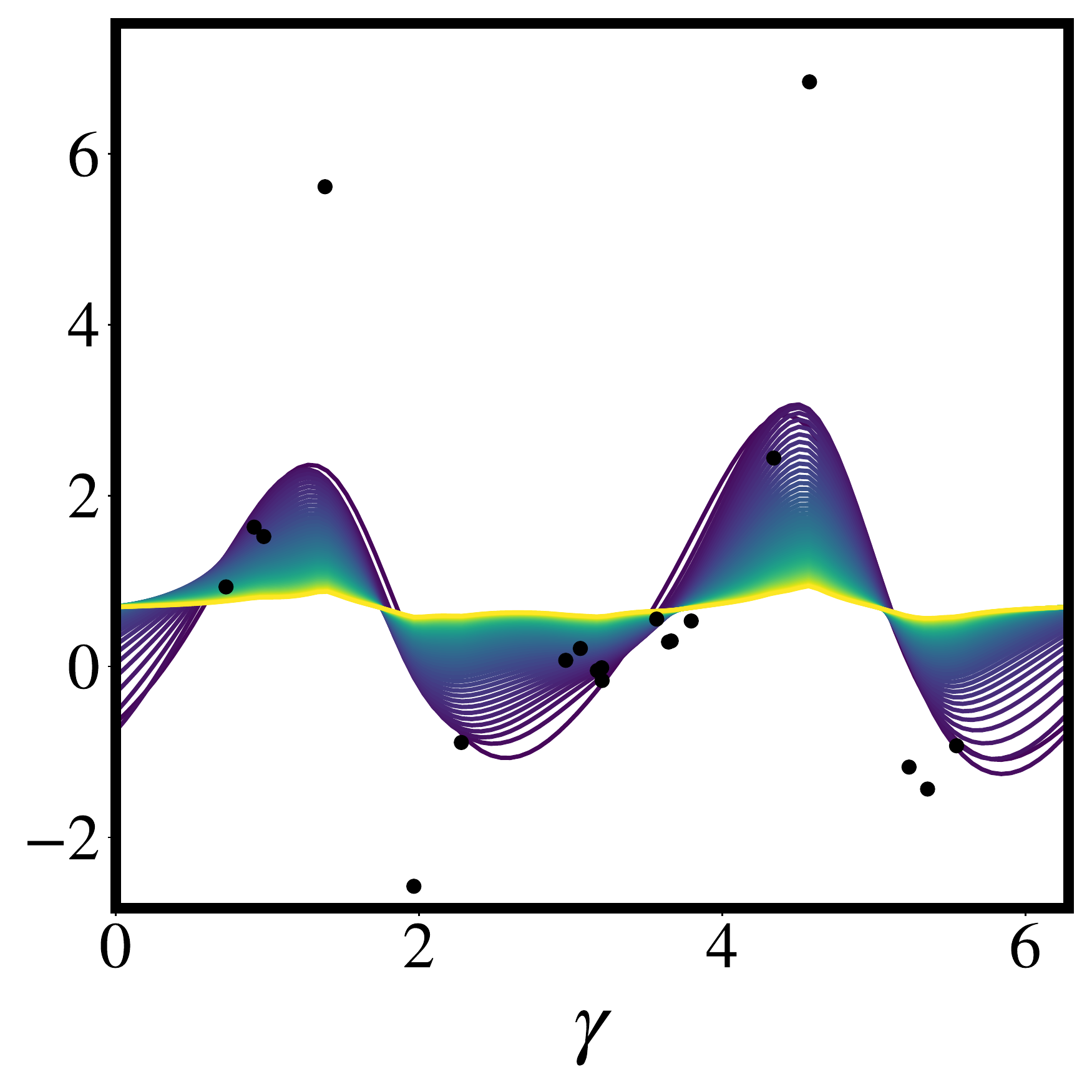} 
    \end{array} &f(\gamma) &&= \tan\gamma \\
    \end{alignat*}
    \caption{Illustration of (lack of) simplicity bias due to kernel fixed points. Training data $\mathbf{x} \in \mathbb{R}^2$ is uniformly sampled on the unit disc at heading $\gamma$. Curves show the posterior mean of a GP regression model on $y=f(\gamma) + \epsilon$ with known additive noise variance fixed at $0.1$. $\sigma_w$ is chosen according to Figure 3. Colours move from purple to yellow as depth increases from $1$ to $64$. (Left) GELU without unique kernel fixed point leading to overfitting (Right) ReLU with unique kernel fixed point leading to underfitting.}
\end{figure}
\renewcommand{\thefigure}{\arabic{figure}}
}
{}
\end{document}